%% file: main_arxiv.tex
\documentclass{article}

\usepackage{graphicx}
\graphicspath{ {./pics/} }
\input{header}
\input{def}
\usepackage{xr}
\icmltitlerunning{MetFlow: MCMC \& VI}

\begin{document}

\twocolumn[
%\icmltitle{Metropolized Flow for Variational Inference}
%\icmltitle{Simple and efficient MCMC Variational Inference using Metropolized Flows}
\icmltitle{MetFlow: A New Efficient Method for Bridging the Gap between Markov Chain Monte Carlo and Variational Inference}




\icmlsetsymbol{equal}{*}

\begin{icmlauthorlist}
  \icmlauthor{Achille Thin}{ecole}
  \icmlauthor{Nikita Kotelevskii}{sk}
  \icmlauthor{Jean-Stanislas Denain}{ecole}
  \icmlauthor{Leo Grinsztajn}{ecole}
  \icmlauthor{Alain Durmus}{ens}
  \icmlauthor{Maxim Panov}{sk}
  \icmlauthor{Eric Moulines}{ecole}
\end{icmlauthorlist}

\icmlaffiliation{ecole}{CMAP, Ecole Polytechnique, Universite Paris-Saclay, 91128 Palaiseau, France}

\icmlaffiliation{ens}{Ecole Normale Supérieure Paris-Saclay, Cachan}

\icmlaffiliation{sk}{CDISE, Skolkovo Institute of Science and Technology, Moscow, Russia}

\icmlcorrespondingauthor{Achille Thin}{achille.thin@polytechnique.edu}

\icmlkeywords{Machine Learning, ICML}

\vskip 0.3in
]



\printAffiliationsAndNotice{}
%\icmlEqualContribution{} % otherwise use the standard text.

\begin{abstract}
  % A wealth of contributions has been devoted to harnessing the computational efficiency of Variational Inference (VI) and the exactness of Markov Chain Monte Carlo (MCMC) methods. Nevertheless, the proposed methods were only partially satisfactory, either since their implementation remains computationally prohibitive, or the optimization objective does not take into account the MCMC steps.

  In this contribution, we propose a new computationally efficient method to combine Variational Inference (VI) with Markov Chain Monte Carlo (MCMC).
  %It is based on the marginal density of the $K$-th iterate of the MCMC chain started from the point from standard variational posterior.
  This approach can be used with generic MCMC kernels, but is especially well suited to \textit{MetFlow}, a novel family of MCMC algorithms we introduce, in which proposals are obtained using Normalizing Flows. The marginal distribution produced by such MCMC algorithms is a mixture of flow-based distributions, thus drastically increasing the expressivity of the variational family. Unlike previous methods following this direction, our approach is amenable to the reparametrization trick and does not rely on computationally expensive reverse kernels.
  Extensive numerical experiments show clear computational and performance improvements over state-of-the art methods.
\end{abstract}

\section{Introduction}
\label{sec:intro}
\input{introduction}

\section{A New Combination Between VI and MCMC}
\label{sec:main}
\input{MarkovNormalizingFlows_new}

\section{Related Work}
\label{sec:related_work}
\input{related_work}

\section{Experiments}
\label{sec:experiments}
\input{Experiments}

\section{Conclusions}
\label{sec:conclusions}
\input{Conclusion}

\clearpage
\newpage

%% Bibliography goes here
\bibliography{mcmc}
\bibliographystyle{icml2020}

\newpage

\onecolumn

\appendix
\section*{Supplementary Material}
\input{SupplementPaper}
%\appendix
%\label{sec:appendix}

%\section{Supplementary materials}
% \todo{Here we might put some of the derivations and later we will make a separate file with Supplementary materials out of this.}

% \textbf{\emph{Do not put content after the references.}}
%

\end{document}

%% file: header.tex
\let\zz\[\let\zzz\]

\makeatletter
\def\set@curr@file#1{%
  \begingroup
    \escapechar\m@ne
    \xdef\@curr@file{\expandafter\string\csname #1\endcsname}%
  \endgroup
}
\def\quote@name#1{"\quote@@name#1\@gobble""}
\def\quote@@name#1"{#1\quote@@name}
\def\unquote@name#1{\quote@@name#1\@gobble"}
\makeatother
\usepackage{graphics}
\usepackage{microtype}
\usepackage{graphicx}
\usepackage{subfigure}
\usepackage{enumitem}
\usepackage{booktabs} % for professional tables
\usepackage[disable]{todonotes}

\usepackage{mathtools}
\usepackage{amsmath,amssymb, amsthm}
\usepackage[nottoc]{tocbibind}
\usepackage{easyeqn}
\usepackage{xargs}
\usepackage{upgreek}
\usepackage{ifthen}
\usepackage{paracol}
\usepackage{url}

\usepackage[accepted]{icml2020}

\usepackage{aliascnt}
\usepackage{cleveref}

\let\[\zz\let\]\zzz

\usepackage{amsthm}

\makeatletter

\crefname{theorem}{theorem}{Theorems}
\Crefname{Theorem}{Theorem}{Theorems}

\newtheorem{proposition}{Proposition}
\crefname{proposition}{proposition}{propositions}
\Crefname{Proposition}{Proposition}{Propositions}

\newtheorem{lemma}{Lemma}
\crefname{lemma}{lemma}{lemmas}
\Crefname{Lemma}{Lemma}{Lemmas}

\newtheorem{corollary}{Corollary}
\crefname{corollary}{corollary}{corollaries}
\Crefname{Corollary}{Corollary}{Corollaries}

\crefname{definition}{definition}{definitions}
\Crefname{Definition}{Definition}{Definitions}

\newtheorem{remark}{Remark}
\crefname{remark}{remark}{remarks}
\Crefname{Remark}{Remark}{Remarks}

\usepackage{bm}

%% file: def.tex
\def\rmd{\operatorname{d}\hspace{-2pt}}
\def\Id{\operatorname{Id}}

\def\iid{i.i.d.}
\def\Ind{\mathbb{I}}

\def\rset{\mathbb{R}}
\def\brset{\overline{\mathbb{R}}}
\def\nset{\mathbb{N}}
\def\nsets{\mathbb{N}^*}
\def\Ber{\mathrm{Ber}}

\newcommandx{\marginal}[2][1=]{\xi^{#2}_{#1}}

\newcommandx{\margindensm}[2]{m_{#1}^{#2}}
\newcommandx{\margindensmu}[4]{m_{#1}^{#2}(#3|#4)}

\newcommandx{\margindens}[4][4=]{\ifthenelse{\equal{#3}{}}{q_{#1}^{#4}(#2)}{q_{#1,#3}^{#4}(#2)}}
\newcommandx{\margindensw}[3][3=]{\ifthenelse{\equal{#3}{}}{q_{#1}^{#2}}{q_{#1,#3}^{#2}}}

\newcommand{\Mtransw}[2]{Q_{#1,#2}}
\newcommand{\chunk}[3]{#1_{#2:#3}}
\newcommand{\KL}[2]{\operatorname{KL}\left(#1\Vert #2\right)}
\newcommand{\KLLigne}[2]{\operatorname{KL}(#1| #2)}
\def\lowerbound{\mathcal{L}}
\def\lowerboundaux{\mathcal{L}_{\mathrm{aux}}}

\def\rmd{\mathrm{d}}

\def\eqsp{\,}

\def\fwdtransfo{T}
\newcommandx{\fwdtransfoparam}[2]{\fwdtransfo_{#1,#2}}

\def\invtransfo{\mathring{\fwdtransfo}}
\def\trueflow{\mathsf{T}}

\def\dmhratio{\mathring{\alpha}_{\phi,\tu}}

\def\msa{\mathsf{A}}

\def\mcu{\mathcal{U}}
\def\wrt{w.r.t.}

\newcommandx{\Vnorm}[2][1=V]{\| #2 \|_{#1}}
\newcommandx{\VnormEq}[2][1=V]{\ensuremath{\left\| #2 \right\|_{#1}}}
\newcommand{\parentheseDeux}[1]{\left[ #1 \right]}

\newcommand{\defEns}[1]{\left\lbrace #1 \right\rbrace }

\newcommandx\probaMarkovTilde[2][2=]
{\ifthenelse{\equal{#2}{}}{{\widetilde{\mathbb{P}}_{#1}}}{\widetilde{\mathbb{P}}_{#1}\left[ #2\right]}}

\newcommand{\plusinfty}{+\infty}

\def\ie{\textit{i.e.}}

\def\eqsp{\;}

\newcommand{\ooint}[1]{\left(#1\right)}
\newcommand{\ccint}[1]{\left[#1\right]}

\newcommandx{\weight}[2][2=n]{\omega_{#1,#2}^N}

 \newcommand{\alaini}[1]{\todo[color=black!20,inline]{{\bf AD:} #1}}

\newcommandx\sequence[3][2=,3=]
{\ifthenelse{\equal{#3}{}}{\ensuremath{\{ #1_{#2}\}}}{\ensuremath{\{ #1_{#2}, \eqsp #2 \in #3 \}}}}
\newcommandx\sequenceD[3][2=,3=]
{\ifthenelse{\equal{#3}{}}{\ensuremath{\{ #1_{#2}\}}}{\ensuremath{( #1)_{ #2 \in #3} }}}

\newcommandx{\sequencen}[2][2=n\in\N]{\ensuremath{\{ #1_n, \eqsp #2 \}}}
\newcommandx\sequenceDouble[4][3=,4=]
{\ifthenelse{\equal{#3}{}}{\ensuremath{\{ (#1_{#3},#2_{#3}) \}}}{\ensuremath{\{  (#1_{#3},#2_{#3}), \eqsp #3 \in #4 \}}}}
\newcommandx{\sequencenDouble}[3][3=n\in\N]{\ensuremath{\{ (#1_{n},#2_{n}), \eqsp #3 \}}}

\def\iid{i.i.d.}

\newcommand{\opnorm}[1]{{\left\vert\kern-0.25ex\left\vert\kern-0.25ex\left\vert #1
    \right\vert\kern-0.25ex\right\vert\kern-0.25ex\right\vert}}

\def\Id{\operatorname{Id}}

\def\varphibf{\operatorname{g}}
\newcommandx{\CPE}[3][1=]{{\mathbb E}_{#1}\left[\left. #2 \middle \vert #3 \right. \right]} %%%% esperance conditionnelle
\newcommandx{\CPVar}[3][1=]{\mathrm{Var}^{#3}_{#1}\left\{ #2 \right\}}
\newcommand{\CPP}[3][]
{\ifthenelse{\equal{#1}{}}{{\mathbb P}\left(\left. #2 \, \right| #3 \right)}{{\mathbb P}_{#1}\left(\left. #2 \, \right | #3 \right)}}

\newcommandx{\osc}[2][1=]{\mathrm{osc}_{#1}(#2)}

\def\Id{\operatorname{Id}}

\def\v{v}

\def\z{z}
\newcommand\coupling[2]{\Gamma(\mu,\nu)}

\def\tpi{\tilde{\pi}}

\newcommand{\complementary}{\mathrm{c}}

\def\Rad{\operatorname{Rad}}
\def\msi{\mathsf{I}}

\def\msa{\mathsf{A}}

\def\msb{\mathsf{B}}

\def\msu{\mathsf{U}}
\def\tmsu{{\mathsf{U}}}
\def\msv{\mathsf{V}}

\def\msphi{\mathsf{\Phi}}

\def\mcbb{\mathcal{B}}  %%% \mcb est déjà pris
\newcommand{\mcb}[1]{\mathcal{B}(#1)}

\def\mcu{\mathcal{U}}

\def\mcq{\mathcal{Q}}

\def\rset{\mathbb{R}}

\def\nset{\mathbb{N}}
\def\nsets{\mathbb{N}^*}

\def\rmd{\mathrm{d}}

\def\rme{\mathrm{e}}

\def\rmC{\mathrm{C}}

\def\tu{u}
\def\FamilyVar{\mathcal{Q}}
\def\RWM{\scriptscriptstyle{\operatorname{RWM}}}
\def\MALA{{\scriptscriptstyle{\operatorname{MALA}}}}
\def\MH{\mathsf{MH}}
\def\stephmc{\upgamma}

\def\densgauss{\varphibf}
\def\LF{\mathsf{LF}}
\def\Refresh{\mathsf{ref}}

\newcommand{\card}[1]{\vert #1\vert}
\def\tz{\tilde{z}}

%% file: introduction.tex
% !TEX root = main.tex

One of the biggest computational challenge these days in machine learning and computational statistics is to sample from a complex distribution known up to a multiplicative constant. Indeed, this problem naturally appears in Bayesian inference~\cite{robert2007bayesian} or generative modeling~\cite{kingma2013auto}. Very popular methods to address this
problem are Markov Chain Monte Carlo (MCMC) algorithms~\cite{brooks2011handbook} and Variational Inference (VI)~\cite{wainwright:jordan:2008, Blei_2017}.
%The methods proposed so far to combine these two approaches have some drawbackFins.
The main contribution of this paper is to present a new methodology to successfully combine these two approaches mitigating their drawbacks and providing the state-of-the-art sampling quality from high dimensional unnormalized distributions.

% Sampling a high-dimensional distribution $\pi = \tilde{\pi}/\rmC_\pi$, known only up to a normalizing constant $C_\pi$, is an ubiquitous issue in many fields \cite{brooks2011handbook,levy2017generalizing,papamakarios2019normalizing}.  Sampling from such distribution is a very difficult problem in general, which is critical for learning and inference. This work focuses on Variational Inference (VI) \cite{wainwright:jordan:2008, Blei_2017}, and in particular Variational Auto Encoder \cite{kingma2013auto}, but the methodology which is introduced goes of course well beyond.

Starting from a parameterized family of distributions $\FamilyVar= \{\margindensw{\phi}{}{}\colon \, \phi \in \msphi \subset
\rset^q\}$, VI approximates the intractable distribution with density $\pi$ on $\rset^D$ by maximizing the evidence lower bound (ELBO) defined by
\begin{align}
\label{eq:lower-bound-marginal}
  \lowerbound(\phi) &= \int \log \bigl(\tilde{\pi}(z) / \margindens{\phi}{z}{}\bigr) \margindens{\phi}{z}{} \rmd z \eqsp,
  % \nonumber
  % \\
  %   \eqsp,
\end{align}
using an unnormalized version $\tilde{\pi}$ of $\pi$, \ie~$\pi=\tilde{\pi}/\rmC_{\pi}$ setting $\rmC_{\pi} = \int_{\rset^D} \tilde{\pi}(z)\rmd z$. Indeed, this
approach consists in minimizing $\phi \mapsto \KLLigne{\margindensw{\phi}{}{}}{\pi}$ since $\lowerbound(\phi) = \log(\rmC_\pi) - \KLLigne{\margindensw{\phi}{}{}}{\pi}$. The
design of the family $\FamilyVar$ of variational distributions has a huge influence on the overall performance -- more flexible families provide better approximations of the target.

Recently, it has been suggested to enrich the traditional mean field variational approximation by combining them with invertible mappings with additional trainable parameters.
A popular implementation of this principle is the Normalizing Flows (NFs) approach~\cite{dinh2016density,pmlr-v37-rezende15,kingma2016improved} in which a mean-field variational distribution is  deterministically transformed through a fixed-length sequence of parameterized invertible mappings. NFs have received a lot of attention recently and have proven to be very successful for VI and in particular for Variational Auto Encoder; see~\cite{kobyzev2019normalizing,papamakarios2019normalizing} and the references therein.

The drawback of variational methods is that they only allow the target distribution to be approximated by a parametric  family of distributions. On the contrary, MCMC are generic methods which have theoretical guarantees~\cite{robert2013monte}. The basic idea behind MCMC is to design a Markov chain $(z_k)_{k \in \nset}$ whose stationary distribution is $\pi$.
% Approximate samples from the target
% are obtained by drawing $z_0$ from an initial distribution $\margindensw{\phi}{0}$ and sampling iteratively from the Markov kernel $z_k \sim \Mtrans{\phi}{z_{k-1}}{\cdot}{}$ for $k \in  \{1, \dots, K\}$.
%By carefully designing $\Mtransw{\phi}{}$,
Under mild assumptions, the distribution of $z_K$ converges to the target $\pi$ as $K$ goes to infinity.
Yet, this convergence is in most cases very slow and therefore this class of methods can be prohibitively computationally expensive.
 % The main problem with MCMC algorithms is that the rate at which the marginal converges to the target (the mixing time) can be prohibitively slow, the problem becoming more acute when the  dimension of the state-space increases (note however that this curse of dimensionality can be at least partially alleviated using sophisticated MCMC samplers based on Langevin or Hamiltonian dynamics; see for example~\cite{ma2019sampling} and the references therein.

The idea to ``bridge the gap'' between MCMC and VI was first considered in~\cite{salimans2015markov} and has later been pursued in several works; see~\cite{wolf2016variational}, \cite{hoffman2019neutra} and~\cite{caterini2018hamiltonian} and the references therein. In these papers, based on a family of Markov kernels with target distribution $\pi$ and depending on trainable parameters $\phi$, the family $\mcq$ consists in the marginal distribution obtained after $K$ iterations of these Markov kernels.

In this paper, we develop a new approach to combine the VI and MCMC approaches.
Compared to~\cite{salimans2015markov} and~\cite{wolf2016variational}, we do not need to extend the variational approximation on the joint distribution of the $K$ samples of the Markov chain and therefore avoid to introduce and learn reverse Markov kernels.

Our main contributions can be summarized as follows:
\begin{enumerate}[wide, labelwidth=!, labelindent=0pt,label=\arabic*)]
  \item We propose a new computationally tractable ELBO which can be applied to most MCMC algorithms, including Metropolis-Adjusted Langevin Algorithm - MALA - and Hamiltonian Monte Carlo -HMC-. Compared to~\cite{pmlr-v70-hoffman17a}, the Markov kernels can be jointly optimized with the initial distribution $\margindensw{\phi}{0}{}$. 

  \item We propose an implementation of our approach \textit{MetFlow} using a new family of ergodic MCMC kernels in which the proposal distributions are constructed using Normalizing Flows. Then, combining these Markov kernels with classical mean-field variational initial distributions, we obtain variational distributions with more expressive power than NF models, at the reasonable increase of the computational cost. Moreover, unlike plain NFs, we guarantee that each Markov kernel leaves the target invariant and that each iteration improves the distribution.

  \item We present several numerical illustrations to show that our approach allows us to meaningfully trade-off between the approximation of the target distribution and computation, improving over state-of-the-art methods. The following link provides access to the implementation of the proposed method and all the experiments: \url{https://github.com/stat-ml/metflow}.
\end{enumerate}

Our paper is organized as follows. In \Cref{sec:main}, we start by describing our new methodology. Then, in \Cref{sec:deterministic_MCMC}, we introduce MetFlow, a class of ``deterministic'' MCMC algorithm, taking advantage of the flexibility of Normalizing Flows as proposals. \Cref{sec:related_work} discusses the related work in more detail. The benefits of our method are illustrated through several numerical experiments in \Cref{sec:experiments}. Finally, we discuss the outcomes of the study and some future research directions in Section~\ref{sec:conclusions}.

%% file: MarkovNormalizingFlows_new.tex
% !TEX root = main.tex
\label{sec:cond-metr-kern}

\subsection{Basics of Metropolis-Hastings}
  The Metropolis Hastings (MH) algorithm to sample a density $\pi$ \wrt~the Lebesgue measure on $\rset^D$ defines a Markov chain $(Z_k)_{k \in \nset}$ with stationary distribution $\pi$ as follows.
  Conditionally to the current state $Z_k \in \rset^D$, $k \in \nset$, %Denote by $Z_k \in \rset^D$ the current state of the Markov chain (MC). At each iteration,
  a proposal $Y_{k+1}=T_\phi(Z_k, U_{k+1})$ is sampled %from a transition density. We might always write the proposal
  %\begin{EQA}[c]
  %  Y_{k+1} = T_\phi(Z_k, U_{k+1}),
  %\end{EQA}
  where $(U_k)_{k \in \nsets}$ is a sequence of \iid~random variables valued in a measurable space $(\msu, \mcu)$, with density $h$ \wrt~to a $\sigma$-finite measure $\mu_{\msu}$, and $\fwdtransfo_\phi\colon \rset^D \times \msu \to \rset^D$ is a measurable function, referred to as the \textit{proposal mapping}, and parameterized by $\phi \in \msphi$. In this work, $\phi$ collectively denotes the parameters used in the proposal distribution, and $(U_k)$ is referred to as the \emph{innovation noise}.
 %Let $\mu_{\msu}$ be a $\sigma$-finite measure on $(\msu, \mcu)$ and $h$ the density function \wrt\ $\mu_{\msu}$ of the innovation noise.
  Then, $Y_{k+1}$ is accepted with probability $\alpha^{\MH}_\phi\bigl(Z_k, T_\phi(Z_k, U_{k+1})\bigr)$, where $\alpha^{\MH}_\phi \colon \rset^{2D}\to\ccint{0,1}$ is designed so that the resulting Markov kernel, denoted by $M_{\phi, h}$, is reversible \wrt~$\pi$, \ie~satisfies the detailed balance $\pi(\rmd z) M_{\phi, h}(z, \rmd z')= \pi(\rmd z')M_{\phi,h}(z', \rmd z)$. With this notation, $M_{\phi, h}$ can be written, for $z \in \rset^D$, $\msa \in \mcb{\rset^D}$, as
  \begin{equation}
    \label{eq:def_M}
    M_{\phi, h}(z,\msa) = \int_{\msu} h(u) Q_{\phi}\bigl((z, u), \msa\bigr) \mu_\msu(\rmd u) \eqsp,
  \end{equation}
  where $\{Q_{\phi}\colon \phi \in \msphi\}$ is the family of Markov kernels
  %on $ (\rset^D\times \msu) \times \mcbb(\rset^D)$,
   given for any $z \in \rset^D$, $u\in \msu$, $\msa \in \mcbb(\rset^D)$, by:
  \begin{multline}
  \label{eq:indexed_kernel}
    Q_{\phi}((z,u),\msa) = \alpha_{\phi}(z, u) \updelta_{ T_\phi(z, u)}(\msa)
    \\
    + \bigl\{1 - \alpha_{\phi}(z, u)\bigr\} \updelta_z(\msa) \eqsp.
  \end{multline}
  In this definition, $\updelta_z$ stands for the Dirac measure at $z$ and $\{\alpha_\phi\colon \rset^{D}\times \msu \to \ccint{0,1}\,,\, \phi \in \msphi\}$ is a family of acceptance functions related to the MH acceptance probabilities by $\alpha_{\phi}(z,u) = \alpha^{\MH}_\phi\bigl(z, T_\phi(z, u)\bigr)$.

  To illustrate this definition, consider first the symmetric Random Walk Metropolis Algorithm (RWM). In such case, $\msu =\rset^D$, $\mu_{\msu}$ is the Lebesgue measure, and $\densgauss$ is the $D$-dimensional standard normal density. The proposal mapping is
  \begin{EQA}[c]
    T^{\RWM}_{\phi}(z, u) = z + \Sigma_\phi^{1/2} u,
  \end{EQA}
  where $\{\Sigma_\phi, \phi \in \rset^q\}$ is a parametric family of positive definite matrices, and the acceptance function is given by $\alpha^{\RWM}_{\phi}(z, u) = 1 \wedge \bigl(\pi(T^{\RWM}_{\phi}(z,u))/\pi(z)\bigr)$.

  Consider now the Metropolis Adjusted Langevin Algorithm (MALA); see~\cite{besag1994comments}. Assume that $z \mapsto \log \pi(z)$ is differentiable and denote by $\nabla \log \pi(z)$ its gradient. The proposal mapping and the associated acceptance function for MALA algorithm are given by
  \begin{align}
  \label{eq:forward-transform-MALA}
    &T^{\MALA}_{\phi} (z,u) = z + \Sigma_\phi \nabla \log \pi(z)+ \sqrt{2} \Sigma_\phi^{1/2}u \eqsp,\\
    \nonumber
    &\alpha^{\MALA}_{\phi}(z,u) = 1 \wedge \frac{\pi\bigl(T^{\MALA}_{\phi} (z,u)\bigr) g_\phi\bigl(T^{\MALA}_{\phi} (z, u), z\bigr)} {\pi(z) g_\phi\bigl(z,T^{\MALA}_{\phi} (z, u)\bigr)}  \eqsp,
  \end{align}
  where $g_\phi(z_1,z_2)= |\Sigma_\phi|^{\scriptscriptstyle{-1/2}} \densgauss\bigl(\Sigma_\phi^{\scriptscriptstyle{-1/2}}(z_2-T^{\MALA}_{\phi} (z_1,0))\bigr)$ is the proposal kernel density.

  %Note that in these two examples, $\log(\tpi)$ should be tractable and it should be possible to evaluate this function. This can be a limitation in some cases as for example the number of observation is too large. A significant number of papers try to circumvent this problem, in particular using discretization of continuous dynamics. However, the Markov kernels they obtain does not target $\pi$ and are biased. In what follows we aim at using Markov kernels for which $\pi$ is really invariant which prevents us to use any biased procedure. We believe that
\subsection{Variational Inference Meets Metropolis-Hastings}
  Let $K \in \nsets$,  $\{\marginal[\phi]{0}\colon \phi \in \msphi\}$ on $\rset^D$ a parametric family of distributions and $\{h_i\}_{i=1}^K$ density functions w.r.t $\mu_{\msu}$. Consider now the following variational family
  \begin{equation}
  \label{eq:variational_family}
    \mathcal{Q}=\{\marginal[\phi]{K} = \marginal[\phi]{0} M_{\phi,h_1}\cdots M_{\phi,h_K}\colon \phi \in \msphi\} \eqsp,
  \end{equation}
  obtained by iteratively applying to the initial distribution $\marginal[\phi]{0}$ the Markov kernels $(M_{\phi,h_i})_{i=1}^K$.

  The objective of VI approach~\cite{wainwright:jordan:2008,Blei_2017} is to minimize the Kullback-Leibler (KL) divergence $\KLLigne{\marginal[\phi]{K}}{\pi}$ \wrt~the parameter $\phi \in \msphi$ to find the distribution which best fits the target $\pi$.
  %Recall that if the $\PP$ and $\QQ$ both have densities $p$, $q$ respectively \wrt~the measure $\mu$ on $\rset^D$, define $\KL{\PP}{\QQ} = \int_{\rset^D} \log(p(x)/q(x))p(x) \rmd \mu(x)$. We also abusively set $\KL{p}{q}= \KL{\PP}{\QQ}$. However, since $\KL{\PP}{\QQ}= \plusinfty$ if $\PP$ is singular \wrt~$\QQ$,
  %For the function $\phi \to \KLLigne{\marginal[\phi]{K}}{\pi}$ to be well-defined,  $\marginal[\phi]{K}$ must have a density  \wrt~the Lebesgue measure.
  For any $\phi \in \msphi$, $u \in \msu$ and $z \in \rset^D$, denote by $\fwdtransfoparam{\phi}{u}(z) = \fwdtransfo_\phi(z,u)$, $\alpha_{\phi,u}(z) = \alpha_{\phi}(z,u)$ and similarly for any $\msa \in \mcb{\rset^D}$, $\Mtransw{\phi}{u}(z, \msa) = Q_{\phi}\bigl((z,u), \msa\bigr)$.

  The key assumption in this section is that for any $\phi \in \msphi$ and $u \in \msu$, $\fwdtransfoparam{\phi}{u}$ is a $\rmC^1$ diffeomorphism. This property is satisfied under mild condition on the proposal. In particular, this holds for RWM and MALA associated with the proposal mappings $T^{\RWM}_{\phi,u}$ and $T^{\MALA}_{\phi,u}$. It holds also for Hamiltonian Monte Carlo at the expense of extending the state space to include a momentum variable, see the supplementary paper \Cref{SPsec:checking-RMW-MALA} and \Cref{SPsec:HMC}. However, it is in general not needed to specify a valid MCMC procedure. For a $\rmC^1(\rset^D,\rset^D)$ diffeomorphism $\psi$, define by $J_\psi(z)$ the absolute value of the Jacobian determinant at $z \in \rset^D$.
  \begin{lemma}
  \label{prop:density-one-iteration}
    Let $(u, \phi) \in \msu\times\msphi$. Assume that  $\marginal[\phi]{0}$ admits a density  $\margindensm{\phi}{0}$  \wrt~the Lebesgue
    measure. Assume in addition $\fwdtransfoparam{\phi}{u}$ is a $\rmC^1$
    diffeomorphism. Then, the distribution
    $\marginal[\phi]{1}(\cdot|u)= \int_{\rset^d}\margindensm{\phi}{0}(z_0) Q_{\phi,u}(z_0,\cdot) \rmd z_0$ has a density \wrt~the Lebesgue measure given by
    \begin{multline}
    \label{eq:margin_dens1}
      \margindensmu{\phi}{1}{z}{u} = \alpha_{\phi,u}\bigl(\fwdtransfoparam{\phi}{u}^{-1}(z)\bigr) \margindensm{\phi}{0}\bigl(\fwdtransfoparam{\phi}{u}^{-1}(z)\bigr) J_{\fwdtransfoparam{\phi}{u}^{-1}}(z) \\+ \bigl\{1-\alpha_{\phi, u}(z)\bigr\} \margindensm{\phi}{0}(z) \eqsp,
    \end{multline}
    and $\marginal[\phi]{1}$ has a density given by $m^1_\phi(z)= \int \margindensmu{\phi}{1}{z}{u} h(u) \mu_\msu(\rmd u)$.
    % and the distribution $\marginal{1}$ is defined for any $\msa \in \mcbb(\rset^D)$, $\marginal{1}(\msa)= \marginal{0} Q_u(\msa) = \int \marginal{0}(\rmd x)Q_u(x,\msa)$.
  \end{lemma}
  \begin{proof}
    Let $f$ be a nonnegative measurable function on $\rset^D$. \eqref{eq:margin_dens1} follows from the
    change of variable $z_1= T_{\phi,u}(z_0)$:
    \begin{align*}
      &\int_{\rset^D} f(z) \margindensm{\phi}{0}(z_0) Q_{\phi, u}(z_0, \rmd z)  \\
      & = \int_{\rset^D} \Bigl[\margindensm{\phi}{0}(z_0) \bigl\{\alpha_{\phi, u}(z_0)f\bigl(\fwdtransfoparam{\phi}{u}(z_0)\bigr)
      \\
      &\qquad + \bigl(1 - \alpha_{\phi, u}(z_0)\bigr) f(z_0)\bigr\}\Bigr]\rmd z_0 \eqsp,
      \\
      &= \int_{\rset^D} \Bigl[ \{\alpha_{\phi, u}\bigl(\fwdtransfoparam{\phi}{u}^{-1}(z_1)\bigr) \margindensm{\phi}{0}(\fwdtransfoparam{\phi}{u}^{-1}(z_1)) J_{T_u^{-1}}(z_1)
      \\
      & \qquad  + (1-\alpha_{\phi, u}(z_1)) \margindensm{\phi}{0}(z_1) \}f(z_1) \Bigr]\rmd z_1 \eqsp.
    \end{align*}
  \end{proof}
  
  An induction argument extends this property to the $K$-th marginal $\marginal[\phi]{K}$.
  Let us define $\fwdtransfo^0=\Id$. For a family $\{\fwdtransfo_i\}_{i=1}^{K}$ of mappings on $\rset^D$ and $1 \leq i \leq k < K$, define $\bigcirc_{j=i}^k \fwdtransfo_j = \fwdtransfo_i \circ \dots \circ \fwdtransfo_k$ and for a sequence $\{h_i\}_{i=1}^K$ of innovation noise densities \wrt~$\mu_\msu$, define $\chunk{h}{1}{K}(\chunk{u}{1}{K})= \prod_{i=1}^K h_i(u_i)$. Finally, set $\alpha^1_{\phi,u}(z) = \alpha_{\phi,u}(z)$ and $\alpha^0_{\phi,u}(z) = 1-\alpha_{\phi,u}(z)$.
  \begin{proposition}
  \label{coro:induction-argument}
    Assume that for any $(u, \phi) \in \msu\times\msphi$, $\fwdtransfoparam{\phi}{u}$ is a $\rmC^1$ diffeomorphism and  $\marginal[\phi]{0}$ admits a density  $\margindensm{\phi}{0}$  \wrt~the Lebesgue
    measure.
    For any $\{u_i\}_{i=1}^K \in \mathsf{U}^K$, the distribution $\marginal[\phi]{K}(\cdot\mid \chunk{u}{1}{K})= \marginal[\phi]{0}\Mtransw{\phi}{u_1}\cdots \Mtransw{\phi}{u_K}$ has a density $\margindensm{\phi}{K}$ given by
    \begin{align}
    \label{eq:margin_densK}
      &\margindensmu{\phi}{K}{z}{\chunk{u}{1}{K}} = \sum_{\chunk{a}{1}{K} \in \{0,1\}^K}\margindensmu{\phi}{K}{z, \chunk{a}{1}{K}}{\chunk{u}{1}{K}} \eqsp,
    \end{align}
    where
    \begin{multline}
    \label{eq:def_marginal_given_u}
      \margindensmu{\phi}{K}{z, \chunk{a}{1}{K}}{\chunk{u}{1}{K}} = \prod_{i=1}^K \alpha^{a_i}_{\phi,u_i} \bigl(\bigcirc_{j=i}^K\fwdtransfoparam{\phi}{u_j}^{-a_j}(z)\bigr)   \\
      \times \margindensm{\phi}{0} \bigl(\bigcirc_{j=1}^K\fwdtransfoparam{\phi}{u_j}^{-a_j}(z)\bigr) J_{\bigcirc_{j=1}^K\fwdtransfoparam{\phi}{u_j}^{-a_j}}(z) \eqsp.
    \end{multline}
    In particular, for a sequence $\{h_i\}_{i=1}^K$ of innovation noise densities, $\marginal[\phi]{K}$ \eqref{eq:variational_family} has a density \wrt~the Lebesgue measure, explicitly given, for any $z \in \rset^D$, by
    \begin{equation}
    \label{eq:def_marginal_family_distribution}
      \margindensm{\phi}{K}(z) = \int_{\msu^K} \defEns{\margindensm{\phi}{K}(z|\chunk{u}{1}{K}) \chunk{h}{1}{K}(\chunk{u}{1}{K})} \rmd \mu_{\msu}^{\otimes K} (\chunk{u}{1}{K}) \eqsp.
    \end{equation}
  \end{proposition}
  
  We can now apply the VI approach the family $\mathcal{Q}$ defined in~\eqref{eq:variational_family}.
  Consider a family of inference function
  \begin{equation*}
      \{\rho(\chunk{a}{1}{K}, \chunk{u}{1}{K}\mid z)\colon z \in \rset^D, \chunk{a}{1}{K}\in \{0,1\}^K, \chunk{u}{1}{K} \in \msu^K\}\eqsp.
  \end{equation*}
  This family may depend upon some parameters, implicit in this notation. As shown below, our objective is to take very simple expressions for those functions.
  We define our ELBO $\lowerboundaux(\phi)$, using now the extended space $(z_K, \chunk{a}{1}{K}, \chunk{u}{1}{K})$, by
  \begin{multline}
    \lowerboundaux(\phi) = \sum_{\chunk{a}{1}{K} \in \{0,1\}^K} \int_{} \chunk{h}{1}{K}(\chunk{u}{1}{K})\margindensmu{\phi}{K}{z_K, \chunk{a}{1}{K}}{\chunk{u}{1}{K}} \\
    \label{eq:Our_auxiliary_variational}
    \times \log \left(\frac{\tilde{\pi}(z_K) \rho(\chunk{a}{1}{K}, \chunk{u}{1}{K}| z_K) }{\margindensmu{\phi}{K}{z_K, \chunk{a}{1}{K}}{\chunk{u}{1}{K}}\chunk{h}{1}{K}(\chunk{u}{1}{K})} \right) \rmd z_K \rmd \mu^{\otimes K}_{\msu}(\chunk{u}{1}{K}) \eqsp.
  \end{multline}
  %which is tractable and can be estimated unbiasedly using %\eqref{eq:mk_z_a1K}.
  Note that $\lowerboundaux$ is a lower bound of $\lowerbound$ expressed in~\eqref{eq:lower-bound-marginal} since defining $\margindensm{\phi}{K}(z,\chunk{a}{1}{K}, \chunk{u}{1}{K})= \margindensm{\phi}{K}(z,\chunk{a}{1}{K}|\chunk{u}{1}{K}) \chunk{h}{1}{K}(\chunk{u}{1}{K})$ and $\margindensm{\phi}{K}(\chunk{a}{1}{K}, \chunk{u}{1}{K}|z) = \margindensm{\phi}{K}(z,\chunk{a}{1}{K}, \chunk{u}{1}{K})/\margindensm{\phi}{K}(z)$, we obtain
  \begin{multline*}
%    \label{eq:Our_auxiliary_variational}
    \lowerboundaux(\phi) = \lowerbound(\phi)\\
    -\int_{\rset^D} \margindensm{\phi}{K}(z_K)\KL{\margindensmu{\phi}{K}{\bullet}{z_K}}{\rho(\bullet|z_K)}\rmd z_K  \eqsp,
  \end{multline*}
  where $\KL{\margindensmu{\phi}{K}{\bullet}{z_K}}{\rho(\bullet|z_K)}$ denotes the KL divergence between $\margindensmu{\phi}{K}{\chunk{a}{1}{K}, \chunk{u}{1}{K}}{z_K}$ and $\rho(\chunk{a}{1}{K}, \chunk{u}{1}{K}| z_K)$.
  We specify the inference functions $\rho$. In particular, a simple choice is $\rho(\chunk{a}{1}{K}, \chunk{u}{1}{K}| z)= r(\chunk{a}{1}{K}| z, \chunk{u}{1}{K})\chunk{h}{1}{K}(\chunk{u}{1}{K})$, where $r$ is a similar family of inference function on $\{0,1\}$. This architecture is based on the representation of the Markov kernel~\eqref{eq:indexed_kernel} we built and simplifies our ELBO. In the following, we always assume the form of such inference function.
  %that case, the ELBO given above simplifies greatly to
%\begin{multline}
%    \lowerboundaux(\phi) = \sum_{\chunk{a}{1}{K} \in \{0,1\}^K} \int_{} \log \left(\frac{\tilde{\pi}(z) r(\chunk{a}{1}{K}| z, \chunk{u}{1}{K}) }{\margindensmu{\phi}{K}{z, \chunk{a}{1}{K}}{\chunk{u}{1}{K}}} \right) \\
%    \label{eq:Our_auxiliary_variational1}
%    \times \margindensmu{\phi}{K}{z, \chunk{a}{1}{K}}{\chunk{u}{1}{K}}\chunk{h}{1}{K}(\chunk{u}{1}{K}) \rmd z \rmd \mu^{\otimes K}(\chunk{u}{1}{K}) \eqsp.
% \end{multline}
  Note here that the key step of our approach for defining $\lowerboundaux$ is to rely on the representation~\eqref{eq:def_M}, allowing us to write explicitly our marginals $\margindensm{\phi}{K}$ compared to~\cite{salimans2015markov}.  % In general, note that it is not necessary to take $u_1, \ldots, u_K$ \iid, and they could for example all follow a different distribution $h_1, \ldots, h_K$ respectively.

  The ELBO $\lowerboundaux$ can be optimized \wrt~$\phi$, typically by stochastic gradient methods, which requires an unbiased estimator of the gradient $\nabla \lowerboundaux(\phi)$. Such estimator can be obtained using the reparameterization trick~\cite{rezende2014stochastic}.
  The implementation of this procedure is a bit more involved in our case, as the integration is done on a mixture of the components $\margindensmu{\phi}{K}{z,\chunk{a}{1}{K}}{\chunk{u}{1}{K}}$. To develop an understanding of the methodology, we consider first the case $K=1$. Denote by $\varphibf$ the density of the $D$-dimensional standard Gaussian distribution. Suppose for example that we can write $m_\phi^0(z) = \varphibf(V_\phi^{-1}(z)) J_{V_\phi^{-1}}(z)$ with $V_\phi(y) = \mu_\phi + \Sigma_\phi^{-1/2}y$. Other parameterization could be handled as well. With two changes of variables, we can integrate \wrt~$\varphibf$, which implies that
  \begin{align}
    & \lowerboundaux(\phi) =\hspace{-5pt} \sum_{a_1 \in \{0,1\}} \int  \rmd y \rmd \mu_\msu(u_1)\left[\alpha^{a_1}_{u_1}\bigl(V_\phi(y)\bigr) \varphibf(y) h(u_1)\right.
        \nonumber
    \\
    \label{eq:reparameterization_trick}
    &\left.\log \left(\frac{\tilde{\pi}\bigl(\fwdtransfoparam{\phi}{u_1}^{a_1}\bigl(V_\phi(y)\bigr)\bigr) r\bigl(a_1 | \fwdtransfoparam{\phi}{u_1}^{a_1}\bigl(V_\phi(y)\bigr), u_1\bigr)}{\margindensmu{\phi}{1}{\fwdtransfoparam{\phi}{u_1}^{a_1}\bigl(V_\phi(y)\bigr), a_1}{u_1}}\right)\right].
  \end{align}
  Justification and extension to the general case $K \in \nsets$ is given in the supplementary paper, see \Cref{SPsec:reparam_trick}.
  From~\eqref{eq:reparameterization_trick}, using the general identity $\nabla \alpha = \alpha \nabla \log(\alpha)$, the gradient of $\lowerboundaux$ is given by
  \begin{align*}
    &\nabla \lowerboundaux(\phi) = \hspace{-10pt}\sum_{a_{1}\in \{0,1\}}\hspace{-3pt} \int \rmd y \rmd \mu_\msu(u_1) \alpha^{a_1}_{u_1}\bigl(V_\phi(y)\bigr) \varphibf(y) h(u_1)
    \\
    &\times \left[ \nabla\log \left(\frac{\tilde{\pi}(\fwdtransfoparam{\phi}{u_1}^{a_1}\bigl(V_\phi(y)\bigr) r\bigl(a_1 | \fwdtransfoparam{\phi}{u_1}^{a_1}\bigl(V_\phi(y)\bigr), u_1\bigr)}{\margindensmu{\phi}{1}{\fwdtransfoparam{\phi}{u_1}^{a_1}\bigl(V_\phi(y)\bigr), a_1}{u_1}}\right)\right.
    \\
    &+ \nabla \log\bigl(\alpha^{a_1}_{u_1}\bigl(V_\phi(y)\bigr)\bigr)
    \\
    & \times\left. \log \left(\frac{\tilde{\pi}(\fwdtransfoparam{\phi}{u_1}^{a_1}\bigl(V_\phi(y)\bigr) r\bigl(a_1 | \fwdtransfoparam{\phi}{u_1}^{a_1}\bigl(V_\phi(y)\bigr), u_1\bigr)}{\margindensmu{\phi}{1}{\fwdtransfoparam{\phi}{u_1}^{a_1}\bigl(V_\phi(y)\bigr), a_1}{u_1}}\right)\right] \eqsp.
  \end{align*}
  Therefore, an unbiased estimator of $\nabla \lowerboundaux(\phi)$ can be obtained by sampling independently the proposal innovation $u_1 \sim h_1$ and the starting point $y \sim \mathcal{N}(0,\mathrm{I})$, and then the acceptation $a_1 \sim \Ber\bigl(\alpha^{a_1}_{u_1}\bigl(V_\phi(y)\bigr)\bigr)$.
  A complete derivation for the case $K \in \nsets$ is given in the supplementary paper, \Cref{SPsec:gradient_ELBO}.

\section{MetFlow: MCMC and Normalizing Flows}
\label{sec:deterministic_MCMC}
  In this section, we extend the construction above to a new class of MCMC methods, for which the proposal mappings are Normalizing Flows (NF). Our objective is to capitalize on the flexibility of NF to represent distributions, while keeping the exactness of MCMC. This new class of MCMC are referred to as \emph{MetFlow}, standing for Metropolized Flows.

  Consider a flow $\fwdtransfo_\phi\colon \rset^D \times \tmsu \to \rset^D$ parametrized by $\phi \in \msphi$. It is assumed that for any $\tu \in \tmsu$, $\fwdtransfoparam{\phi}{\tu}\colon z \mapsto \fwdtransfo_{\phi}(z,\tu)$ is a $\rmC^1$ diffeomorphism. Set $\msv=\{-1,1\}$. For any $\tu \in \tmsu$, consider the involution $\invtransfo_{\phi,u}$ on $\rset^D\times\msv$, \ie~$\invtransfo_{\phi,u}\circ\invtransfo_{\phi,u}= \Id$ , defined for $\z \in \rset^D$, $v\in\{-1,1\}$ by
  \begin{equation}
  \label{eq:involution}
    \invtransfo_{\phi,\tu}(z,v) = (\fwdtransfoparam{\phi}{\tu}^v(z),-v)\eqsp.
  \end{equation}
  The variable $v$ is called the direction. If $v=1$ (respectively $v=-1$), the ``forward''(resp. ``backward'') flow $\fwdtransfoparam{\phi}{\tu}$ (resp. $\fwdtransfoparam{\phi}{\tu}^{-1}$) is used.
  For any $\z \in \rset^D$, $v\in\{-1,1\}$, $\msa \in \mcb{\rset^D}$, $\msb \subset \msv$, we define the kernel
  \begin{multline}
  \label{eq:def_markov_Q}
    R_{\phi, \tu}\bigl((z,v), \msa \times \msb\bigr) = \dmhratio(z,v) \updelta_{\fwdtransfoparam{\phi}{\tu}^v(z)}(\msa)\otimes\updelta_{-v} (\msb)\\ + \{1-\dmhratio(z,v)\} \updelta_{z} (\msa)\otimes\updelta_{v} (\msb) \eqsp,
  \end{multline}
  where $\dmhratio\colon \rset^D\times\msv\to\ccint{0,1}$ is the acceptance function.

  \begin{proposition}
  \label{propo:2}
    Let $\nu$ be a distribution on $\msv$, and $(\tu,\phi) \in \tmsu\times\msphi$. Assume that $\dmhratio\colon \rset^D\times\msv \to \ccint{0,1}$ satisfies for any $(z,v) \in \rset^D\times\msv$,
    \begin{multline}
      \label{eq:condition_alpha}
      \dmhratio(z, v) \pi(z) \nu(v) \\
      = \dmhratio\bigl(\invtransfo_\phi(z,v)\bigr) \pi\bigl(\fwdtransfoparam{\phi}{\tu}^v(z)\bigr) \nu(-v) J_{\fwdtransfoparam{\phi}{\tu}^v}(z) \eqsp.
    \end{multline}
    Then for any $(\tu,\phi) \in \tmsu\times\msphi$, $R_{\phi,\tu}$ defined by~\eqref{eq:def_markov_Q} is reversible with respect to $\pi\otimes \nu$. In particular, if for any $(z,v) \in \rset^D\times\msv$,
    \[
      \dmhratio(z,v) = \varphi\left(\pi\bigl(\fwdtransfoparam{\phi}{\tu}^v(z)\bigr) \nu(-v) J_{\fwdtransfoparam{\phi}{\tu}^{v}}(z) / \pi(z) \nu(v)\right)\eqsp,
    \]
    for $\varphi\colon \brset_+ \to \brset_+$, then \eqref{eq:condition_alpha} is satisfied if $\varphi(\plusinfty)=1$ and for any $t \in \brset_+$,
    $t\varphi(1/t) = \varphi(t)$.
  \end{proposition}

  \begin{remark}
  \label{remark:ratios_discussion}
    The condition~\eqref{eq:condition_alpha} on the acceptance ratio $\dmhratio$ has been reported in~\citep[Section 2]{tierney:1998} (see also~\citep[Proposition 3.5]{andrieu2019peskun} for extensions to the non reversible case). Standard choices for the acceptance function $\dmhratio$ are the Metropolis-Hastings and Barker ratios which correspond to $\varphi\colon t \mapsto \min(1,t)$ and $t \mapsto t / (1 + t)$ respectively.
  \end{remark}

  If we define for $\tu \in \tmsu$, $v \in \msv$, $z \in \rset^D$, $\msa \in \mcb{\rset^D}$,
  \begin{align}
    \label{eq:link_Rphi_Qphi}
    &Q_{\phi,(\tu, v)}(z, \msa) = R_{\phi,\tu}((z,v), \msa\times \msv) \\
    \nonumber
    & = \dmhratio(z, v) \updelta_{\fwdtransfoparam{\phi}{\tu}^{v}(z)}(\msa)
    + \{1-\dmhratio(z, v)\} \updelta_{z}(\msa) \eqsp,
  \end{align}
  we retrieve the framework defined in~\Cref{sec:cond-metr-kern}.
  In turn, from a distribution $\nu$ for the direction, the family $\{Q_{\phi,(\tu, v)}\colon (\tu, v) \in \tmsu\times\msv\}$ defines a MH kernel, given for $\tu \in \tmsu$, $z \in \rset^D$, $\msa \in \mcb{\rset^D}$ by
  \begin{equation*}
    M_{\phi,u,\nu}(z,\msa) \hspace{-2pt}= \nu(1)Q_{\phi,(\tu, 1)}(z, \msa)+\nu(-1)Q_{\phi,(\tu, -1)}(z, \msa)\eqsp.
  \end{equation*}
  The key result of this section is
  \begin{corollary}
  \label{coro:Reversibility_Mphi_Rphi}
   For any $\tu \in \tmsu$ and any distribution $\nu$, the kernel $M_{\phi,u,\nu}$ is reversible \wrt~$\pi$.
  \end{corollary}
  Consider for example the MALA proposal mapping $T_\phi^\MALA$. If we set
  \begin{equation*}
    \dmhratio^\MALA(z,v) = 1\wedge\left\{ \pi\bigl(\fwdtransfoparam{\phi}{\tu}^v(z)\bigr) \nu(-v) J_{\fwdtransfoparam{\phi}{\tu}^{v}}(z) / \pi(z) \nu(v)\right\}
  \end{equation*} with $\fwdtransfo_\phi \leftarrow \fwdtransfo_\phi^\MALA$, then for any $\tu \in \rset^D$ and any distribution $\nu$, $M_{\phi,u,\nu}^\MALA$ is reversible \wrt~$\pi$, which is not the case for $Q_{\phi,u}^\MALA$ defined in~\eqref{eq:indexed_kernel} with acceptance function $\alpha_\phi^\MALA$ given by~\eqref{eq:forward-transform-MALA}. Recall indeed that the reversibility is only satisfied for the kernel $\int_{} Q_{\phi,u}^\MALA(z, \msa) \varphibf(u)\rmd u$.

  As the reversibility is satisfied for any $\chunk{\tu}{1}{K} \in \tmsu^K$, we typically get rid of the integration \wrt~the innovation noise $\chunk{h}{1}{K}$ and rather consider a fixed sequence $\chunk{\mathbf{u}}{1}{K}$ of proposal noise. For RWM or MALA, this sequence could be a completely uniformly distributed sequence as in Quasi Monte Carlo method for MCMC, see~\cite{schwedes:calderhead:2018, chen:dick:owen:2011}.
  Using definition~\eqref{eq:link_Rphi_Qphi} and \Cref{coro:induction-argument}, we can write the density $\margindensmu{\phi,\chunk{\mathbf{u}}{1}{K}}{K}{\cdot }{\chunk{v}{1}{K}}$ of the distribution $\marginal[\phi,\chunk{\mathbf{u}}{1}{K}]{K}(\cdot\mid \chunk{v}{1}{K})= \marginal[\phi]{0}\Mtransw{\phi}{(\mathbf{u}_1,v_1)}\cdots \Mtransw{\phi}{(\mathbf{u}_K,v_K)}$.
   Setting $\alpha_{\phi,\tu,v}(z)= \dmhratio(z, v)$ as in the previous section, we can write, for any $z \in \rset^D$, $\chunk{a}{1}{K} \in \{0,1\}^K$, $\chunk{v}{1}{K} \in \{-1,1\}^K$, $\chunk{\mathbf{u}}{1}{K} \in \tmsu^{K}$
  \begin{align}
    &\margindensmu{\phi,\chunk{\mathbf{u}}{1}{K}}{K}{z, \chunk{a}{1}{K} }{\chunk{v}{1}{K}} = \margindensm{\phi}{0}\bigl(\bigcirc_{j=1}^K\fwdtransfoparam{\phi}{\mathbf{u}_j}^{-v_j a_j}(z)\bigr)   \\
    \nonumber
    & \quad \times J_{\bigcirc_{j=1}^K\fwdtransfoparam{\phi}{\mathbf{u}_j}^{-v_j a_j}}(z) \prod_{i=1}^K \alpha^{ a_i}_{\mathbf{u}_i, v_i}\bigl(\bigcirc_{j=i}^K\fwdtransfoparam{\phi}{\mathbf{u}_j}^{-v_j a_j}(z)\bigr)\eqsp.
  \end{align}
  Moreover, as reversibility is satisfied for any distribution $\nu$, we could let it depend upon some parameters also denoted $\phi$ and write $\chunk{\nu}{\phi,1}{K}= \prod_{i=1}^{K}\nu_{\phi,i}$. Defining an inference function $r_{\chunk{\mathbf{u}}{1}{K}}(\chunk{a}{1}{K}|z, \chunk{v}{1}{K})$, we can thus obtain the lower bound parametrized by the fixed sequence $\chunk{\mathbf{u}}{1}{K}$ and $\phi$:
  {\small\begin{multline}
  \label{eq:Transfo_auxiliary_variational}
    \lowerboundaux(\phi;\chunk{\mathbf{u}}{1}{K}) =
      \int_{} \sum_{\chunk{v}{1}{K}} \sum_{\chunk{a}{1}{K}}  \margindensmu{\phi,\chunk{\mathbf{u}}{1}{K}}{K}{z, \chunk{a}{1}{K} }{\chunk{v}{1}{K}} \\
    \times \log \left(\frac{\tilde{\pi}(z) r_{\chunk{\mathbf{u}}{1}{K}}(\chunk{a}{1}{K}|z, \chunk{v}{1}{K}) }{\margindensmu{\phi,\chunk{\mathbf{u}}{1}{K}}{K}{z, \chunk{a}{1}{K} }{\chunk{v}{1}{K}}}\right) \chunk{\nu}{\phi,1}{K}(\chunk{v}{1}{K})\rmd z \eqsp,
    \end{multline}}
  for which stochastic optimization can be performed using the same reparametrization trick~\eqref{eq:reparameterization_trick}.

  The choice of the transformation $\fwdtransfo_\phi$ is really flexible. Let $\{\trueflow_{\phi,i}\}_{i=1}^K$ be a family of $K$ diffeomorphisms on $\rset^D$.  A flow model based on $\{\trueflow_{\phi,i}\}_{i=1}^K$ is defined as a composition $\trueflow_{\phi,K} \circ \cdots \circ \trueflow_{\phi,1}$ that pushes an initial distribution $\marginal[\phi]{0}$ with density $m^0_\phi$ to a more complex target distribution $\marginal[\phi]{K}$ with density $m^K_\phi$, given  for any $z \in \rset^D$ by $m^K_{\phi}(z) = m^0\bigl(\bigcirc_{i=1}^K \trueflow_{\phi,i}^{-1} (z)\bigr) J_{\bigcirc_{i=1}^K \trueflow_{\phi,i}^{-1}}(z)$, see~\cite{tabak2013family,pmlr-v37-rezende15,kobyzev2019normalizing,papamakarios2019normalizing}.
  We now proceed to the construction of \emph{MetFlow}, based on the same deterministic sequence of diffeomorphisms. A \emph{MetFlow} model is obtained by applying successively the Markov kernels $M_{\phi,1,\nu},\ldots, M_{\phi,K,\nu}$, written as, for $z \in \rset^D$, $\msa \in \mcb{\rset^D}$, $i \in \{1,\ldots,K\}$:
  \begin{align*}
    M_{\phi,i,\nu}(z,\msa)= \sum_{\v \in \msv}\nu(v)\mathring{\alpha}_{\phi,i}(z,v)  \updelta_{\trueflow_{\phi,i}^{v}(z)}(\msa)\\
    +(1-\sum_{\v \in \msv}\nu(v)\mathring{\alpha}_{\phi,i}(z,v))\updelta_z(\msa)\eqsp.
  \end{align*}
  Each of those is reversible \wrt~the stationary distribution $\pi$ and thus leaves $\pi$ invariant. In such a case, the resulting distribution $\marginal[\phi]{K}$ is a mixture of the pushforward of $\marginal[\phi]{0}$ by the flows $\{\trueflow_{\phi,K}^{v_K a_K} \circ \dots \circ\trueflow_{\phi,1}^{v_1 a_1}\,,\, \chunk{v}{1}{K}\in \msv^K, \chunk{a}{1}{K}\in \{0,1\}^K\}$. The parameters $\phi$ of the flows $\{\trueflow_{\phi,i}\}_{i=1}^K$ are optimized by maximizing an ELBO similar to~\eqref{eq:Transfo_auxiliary_variational} in which $\margindensm{\phi,\chunk{\mathbf{u}}{1}{K}}{K}$ is substituted by $\margindensm{\phi,1:K}{K}$ with $\fwdtransfo_{\phi, \mathbf{u}_i} \leftarrow \trueflow_{\phi,i}$.
  The kernel $M_{\phi,i,\nu}$ shares some similarity with Transformation-based MCMC (T-MCMC) introduced in~\cite{Dutta_2014}. However, the transformations considered in~\cite{Dutta_2014} and later by~\cite{dey2016geometric} are elementary additive or multiplicative transforms acting coordinate-wise. In our contribution, we considered much more sophisticated transformations, inspired by the recent advances on normalizing flows.

  Among the different flow models which have been considered recently in the literature~\cite{papamakarios2019normalizing}, we chose Real-Valued Non-Volume Preserving (RNVP) flows~\cite{dinh2016density} because they are easy to compute and invert. An extension of our work would be to consider other flows, such as Flow++~\cite{ho2019flow}, which can also be computed and inverted efficiently. We could also use autoregressive models, such as Inverse Autoregressive Flows~\cite{kingma2016improved}, which have a tractable - albeit non parallelizable - inverse. Even more expressive flows like NAF~\cite{huang2018neural}, BNAF~\cite{Cao2019BlockNA} or UMNN~\cite{wehenkel2019unconstrained} could be experimented with. Although they are not invertible analytically, this problem could be solved either by the Distribution Distillation method~\cite{oord2017parallel}, or simply by a classic root-finding algorithm: this is theoretically tractable because of the monotonous nature of these flows, and empirically satisfactory~\cite{wehenkel2019unconstrained}.

%% file: related_work.tex
% !TEX root = main.tex

In this section, we compare our method with the state-of-the-art for combining MCMC and VI.
The first attempt to bridge the gap between MCMC and VI is due to~\cite{salimans2015markov}. The method proposed in~\cite{salimans2015markov} uses a different ELBO, based on the joint distribution of the $K$ steps of the Markov chain $\chunk{z}{0}{K} = (z_0, \dots, z_K)$, whereas MetFlows are based on the marginal distribution of the $K$-th component. % . To cope with that augmented space, it is then necessary to have
\cite{salimans2015markov} introduce an auxiliary inference function $r$ and consider the ELBO:
\begin{align}
	&
  \int \log \left(\frac{\tilde{\pi}(z) r_{}(\chunk{z}{0}{K-1}| z_K)}{\margindens{\phi}{\chunk{z}{0}{K}}{}}\right)
  \margindens{\phi}{\chunk{z}{0}{K}}{} \rmd \chunk{z}{0}{K}
  \\
  \nonumber
  &=
  \lowerbound(\phi) - \int \margindens{\phi}{z_K}{}[K]
  \KL{\margindens{\phi}{\cdot \mid z_K}{}}{r_{}(\cdot \mid z_K)}\rmd z_K \eqsp,
\end{align}
where $  \margindens{\phi}{\chunk{z}{0}{K}}{}$ is the joint distribution of the path $\chunk{z}{0}{K}$ \wrt~the Lebesgue measure.
An optimal choice of the auxiliary inference distribution is $r_{}(\chunk{z}{0}{K-1} \mid z_K)= \margindens{\phi}{\chunk{z}{0}{K-1}|z_K}{}$, the conditional distribution of the Markov chain path $\chunk{z}{0}{K-1}$  given its terminal value $z_K$, but this distribution is in most cases intractable.
\cite{salimans2015markov,wolf2016variational} discuss several way to construct sensible approximations of $\margindens{\phi}{\chunk{z}{0}{K-1} \mid z_K}{}$ by introducing learnable time inhomogeneous backward Markov kernels. This introduces additional parameters to learn and degrees of freedom in the choice on the reverse kernels which are not easy to handle. On the top of that, this increases significantly the computational budget.
% Taking into account the acceptance / rejection which are essential in the Metropolis-Hastings scheme require to introduce additional auxiliary variables, again at the expense of  computational complexity; see~\citep[Section 4]{salimans2015markov}.

%Recent work has explored methods to combine MCMC algorithms with VI. For example, \cite{ruiz2019contrastive} and \cite{li2017approximate} use MCMC steps to guide the optimization of the variational parameter $\phi$. Other authors build more expressive variational approximations by applying $K$ MCMC steps to samples from the variational prior $q^0_{\phi}$. Due to the stochasticity of the MCMC steps, the marginal distribution of the final sample $z_K$ becomes the intractable mixture $\margindens{\phi}{z_K}{}[K]= \iint \margindens{\phi}{\chunk{z}{0}{K}}{} \rmd \chunk{z}{0}{K-1}$. ELBO-based approaches, which rely on computing $\margindens{\phi}{z_K}{}[K]$, therefore face a problem (problem \star).

% Hoffman paper's is only tackling the problem in the VAE context, which means that it's also using the MCMC steps to optimize theta (with another metric than the ELBO)

% In a Variational Auto Encoder setting,
\cite{pmlr-v70-hoffman17a} also suggests to build a Markov Chain to enrich the approximation of $\pi$. % the posterior distribution of his generative model (which plays exactly the role, in Bayesian Inference, of the target which he can only access up to a normalizing constant, the likelihood).
More precisely, \cite{pmlr-v70-hoffman17a} optimizes the ELBO with respect to the initial distribution $\margindensw{\phi}{0}{}$, and only uses the MCMC steps to produce ``better'' samples to the target distribution.
However, there is no feedback from the resulting marginal distribution  there to optimize the parameters of the variational distribution $\phi$. This method does not thus directly and completely unifies VI and MCMC, even though it simplifies the process by avoiding the use of the extended space and the reverse kernels.
%Training the variational prior is tantamount to devising better MCMC initialisations. Moreover, applying MCMC steps can only reduce the total variation distance and KL divergence between our approximation and the target ~\citep[p.~81]{cover2012elementsInformationTheory}. \cite{hoffman2019neutra} also separate MCMC and VI, by using VI to build a target that is more amenable to MCMC. However, these approaches don't automatically tune the parameters of the MCMC kernels, and don't use feedback from the final samples to optimize $\phi$.
\cite{ruiz:titsias:2019} refines~\cite{pmlr-v70-hoffman17a} by using a contrastive divergence approach; compared to the methodology presented in this paper, \cite{ruiz:titsias:2019} do not capitalize on the expression of the marginal density.  

Another solution to avoid reverse kernels is considered in~\cite{caterini2018hamiltonian} which amounts to remove randomness from an Hamiltonian MC algorithm. % The procedure might be seen as a HMC-inspired normalizing flow: the variational distribution is obtained by applying successive leapfrog and tempering steps to the mean-field variational prior: the marginal distribution $\margindensw{\phi}{K}{}$ is tractable via a change of variables formula.
However, by getting rid of the accept-reject step and the resampling of the momenta in the HMC algorithm, this approach forgoes the guarantees that come with exact MCMC algorithms.

%% file: Experiments.tex
% !TEX root = main.tex

  In this section, we illustrate our findings. We present examples of sampling from complex synthetic distributions which are often used to benchmark generative models, such as a mixture of highly separated Gaussians and other non-Gaussian 2D distributions. We also present posterior inference approximations and inpainting experiments on MNIST dataset, in the setting outlined by~\cite{levy2017generalizing}. Many more examples are given in the supplementary paper.

  We implement the \textit{MetFlow} algorithm described in \Cref{sec:deterministic_MCMC} to highlight the efficiency of our method. For our learnable transitions $T_\phi$, we use RNVP flows.

  We consider two settings. In the \emph{deterministic} setting, we use $K$ different RNVP transforms $\{\trueflow_{\phi,i}\}_{i=1}^K$, and the parameters for each individual transform $\trueflow_{\phi,i}$ are different. In the \emph{pseudo-randomized} setting, we define global transformation $\fwdtransfo_\phi$ on $\rset^D\times\msu$ and set $\trueflow_{\phi,i}= \fwdtransfo_\phi(\cdot, \mathbf{u}_i)$, where $\chunk{\mathbf{u}}{1}{K}$ are $K$ independent draws from a standard normal distribution. In such case, the parameters are the same for the flows $\trueflow_{\phi,i}$, only the innovation noise $\mathbf{u}_i$ differs. Typically, RNVP are encoded by neural networks. In the second setting, the network will thus take as input $z$ and $\mathbf{u}$ stacked together.

  In the second setting, once training has been completed and a fit \(\hat{\phi}\) of the parameters has been obtained, we can sample additional noise innovations \(\mathbf{u}_{(K+1):mK}\). We then consider the distribution given by $\marginal[\hat{\phi}]{mK} =\marginal[\hat{\phi}]{K}M_{\hat{\phi},\mathbf{u}_{K+1},\nu},\ldots, M_{\hat{\phi},\mathbf{u}_{mK},\nu}$ where $\nu$ is typically the uniform on $\{-1,1\}$, as defined as in Section~\ref{sec:deterministic_MCMC}. \(mK\) corresponds to the length of the final Markov chain we consider. In practice, we have found that sampling additional noise innovations this way yields a more accurate approximation of the target, thanks to the asymptotic guarantees of MCMC.

% is this amount of detail necessary?
  Barker ratios (see \Cref{remark:ratios_discussion}) have the advantage of being differentiable everywhere. Metropolis-Hastings (MH) ratios are known to more efficient than Barker ratios in the Peskun sense, see~\cite{peskun:1973}. Moreover, although they are not differentiable at every point $z \in \rset^D$, differentiating MH ratios is no harder than differentiating a ReLu function. We thus use MH ratios in the following.

  All code is written with the Pytorch~\cite{pytorch} and Pyro~\cite{bingham2018pyro} libraries, and all our experiments are run on a GeForce GTX 1080 Ti.

\subsection{Synthetic data. Examples of sampling.}
\label{sec:mixture-gaussians_0}

\subsubsection{Mixture of Gaussians}
\label{sec:mixture-gaussians}
  The objective is to sample from a mixture of 8 Gaussians in dimension $2$, starting from a standard normal prior distribution $q^0$, and compare MetFlow to RNVP. %Note here that if other flows bring more flexibility, the properties brought by our method will conceptually not be recovered by any others flows, to our knowledge.
  We are using an architecture of five RNVP flows ($K = 5$), each of which is parametrized by two three-layer fully-connected neural networks with LeakyRelu (0.01) activations. In this example, we consider the pseudo-randomized setting.
  The results for MetFlow and for RNVPs alone are shown on \Cref{fig:8gmm_all}.
%  \begin{figure}[h]
%    \begin{minipage}{0.4\linewidth}
%      \includegraphics[width=0.98\textwidth]{pics/gmm8_all_dpi400-target.png}
%    \end{minipage}
%    \begin{minipage}{0.6\linewidth}
%      \includegraphics[width=0.98\textwidth]{pics/gmm8_all_dpi400_result.png}
%    \end{minipage}
%   \caption{Density matching for mixture of 8 Gaussians. Left image: target distribution. Right images: prior distribution and the resulting distributions of different methods.} %MetFlow finds all the modes and improves with more iterations, while RNVP fails to find all the modes.}
%  \label{fig:8gmm_all}
%  \end{figure}
%
  \begin{figure}[h!]
    \begin{tabular}{ccc}
      \includegraphics[width=0.3\linewidth]{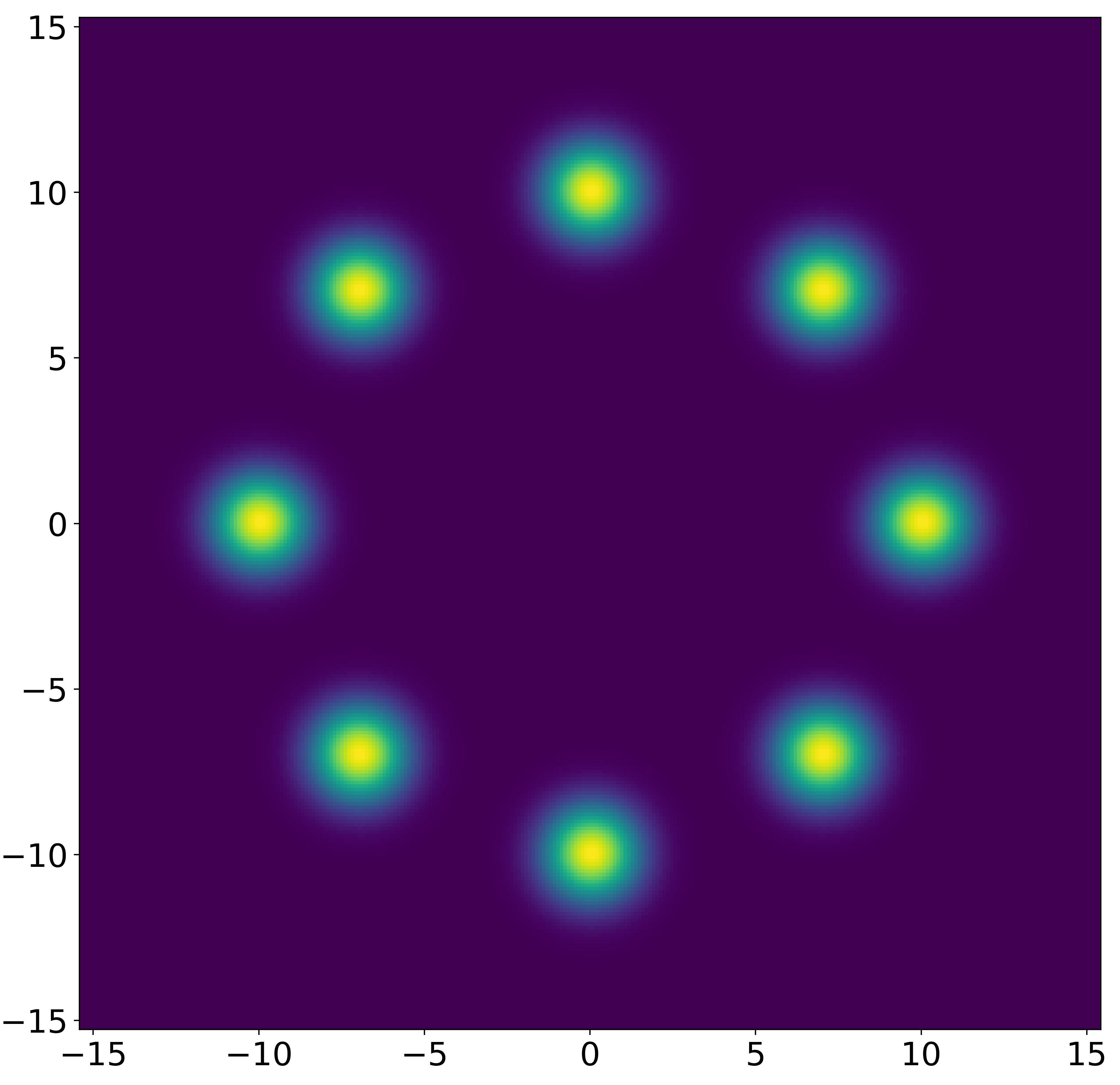} & \includegraphics[width=0.3\linewidth]{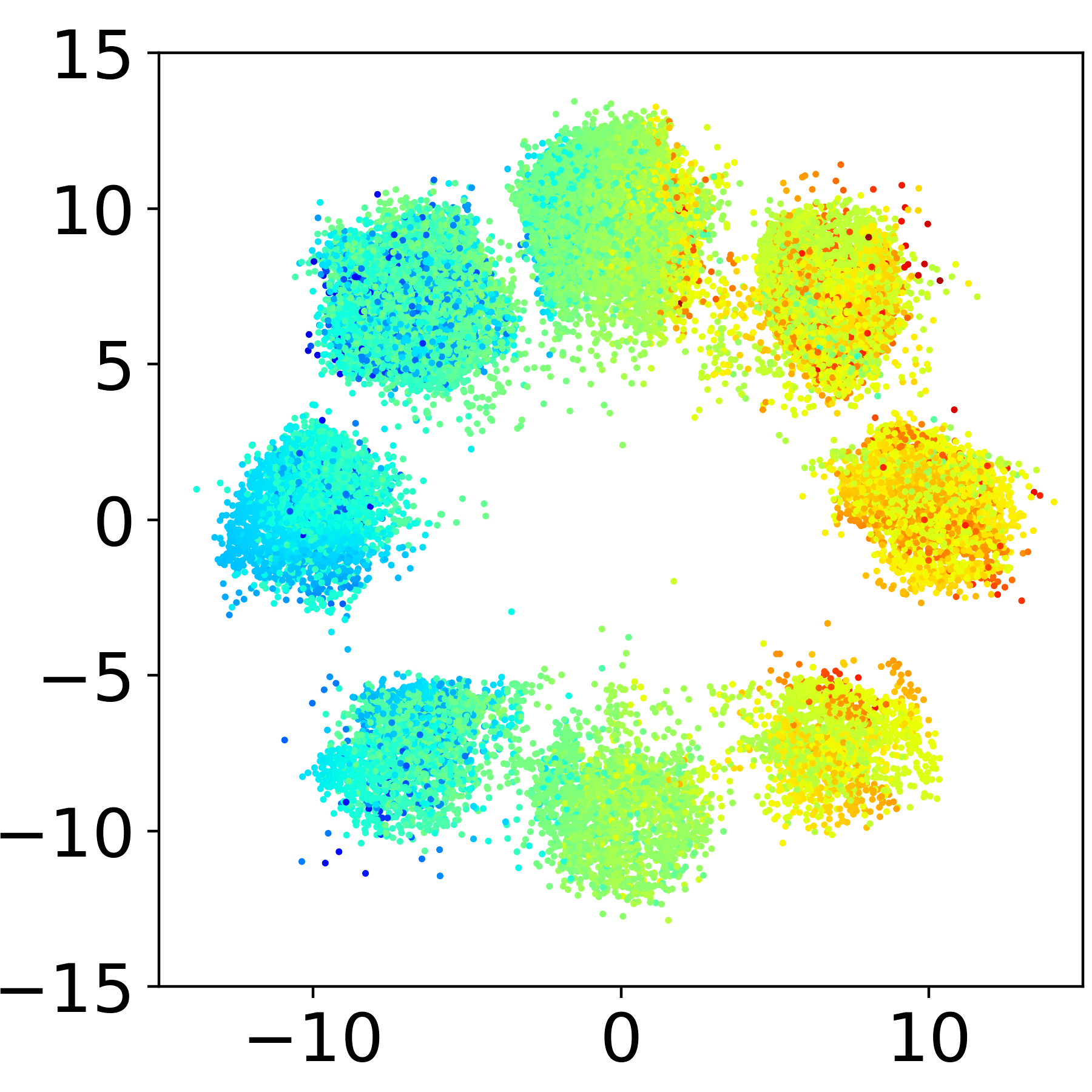} & \includegraphics[width=0.3\linewidth]{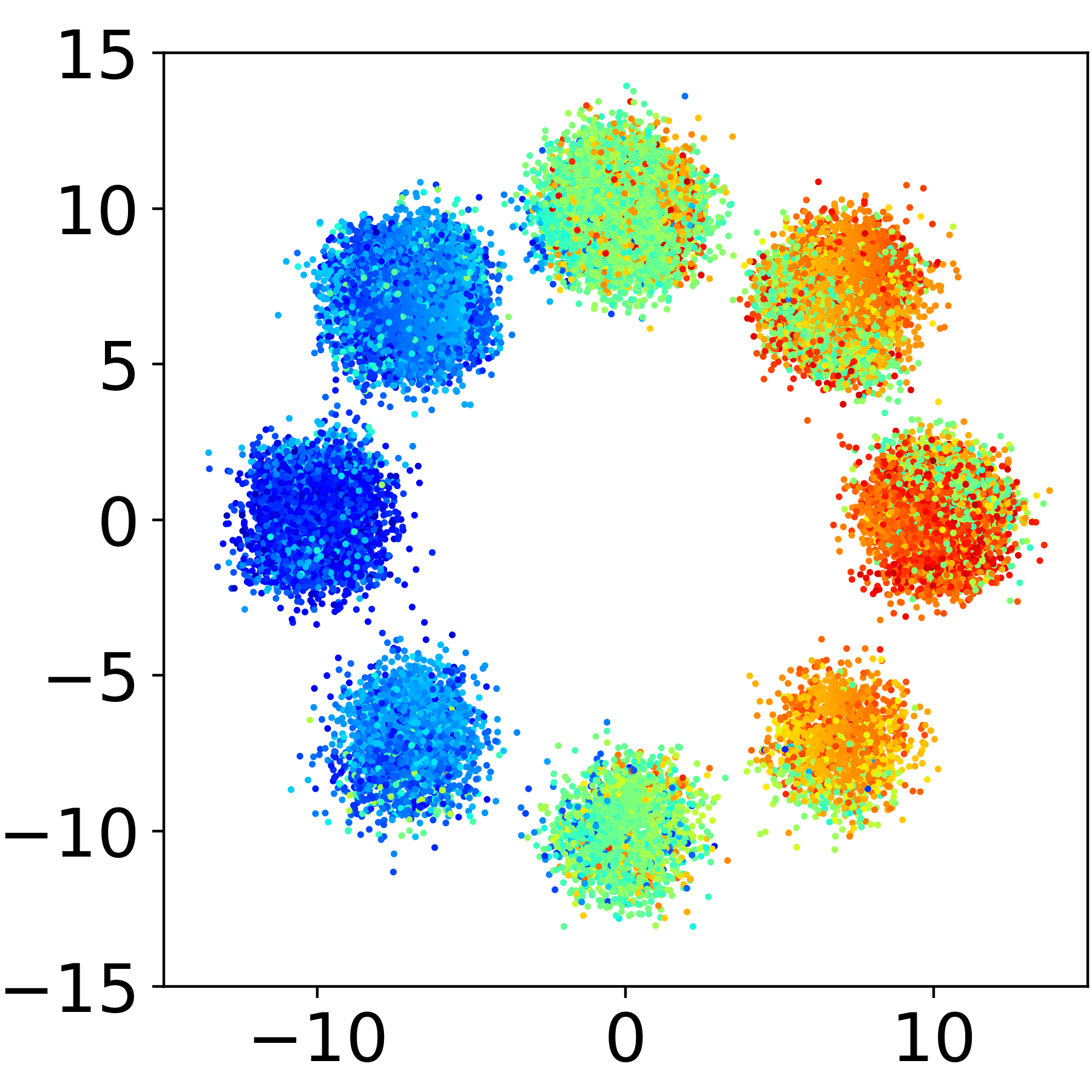} \\
      \includegraphics[width=0.3\linewidth]{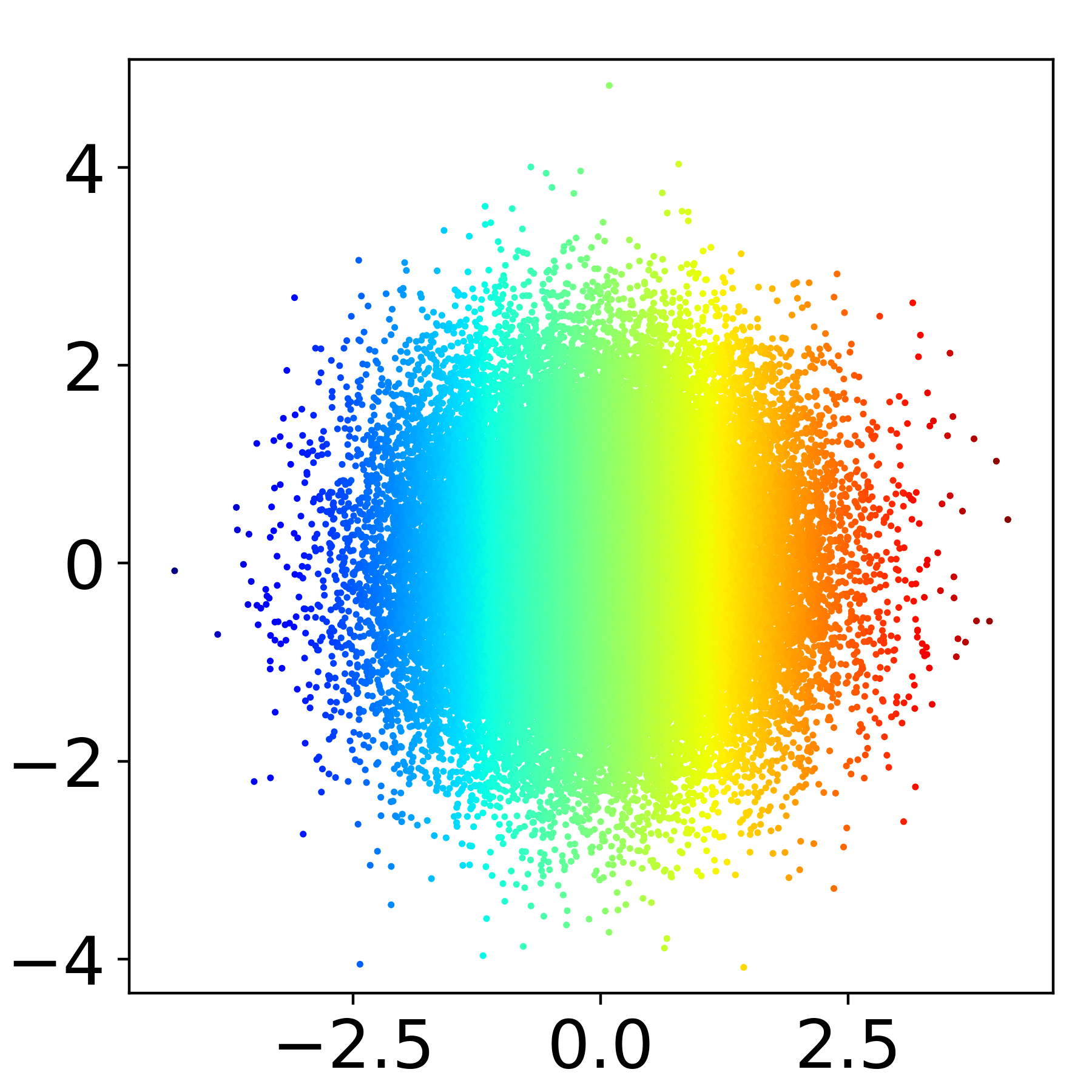} & \includegraphics[width=0.3\linewidth]{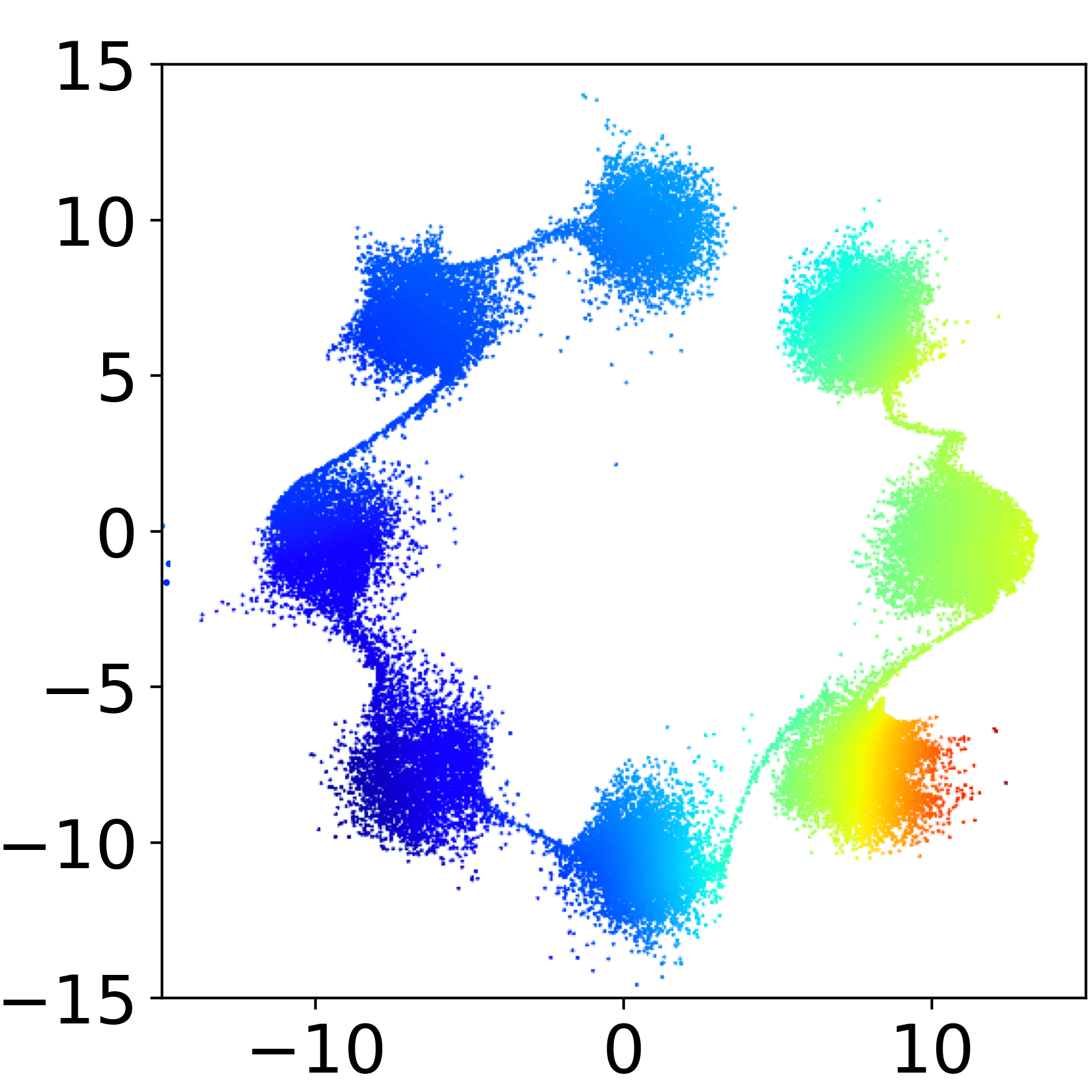} & \includegraphics[width=0.3\linewidth]{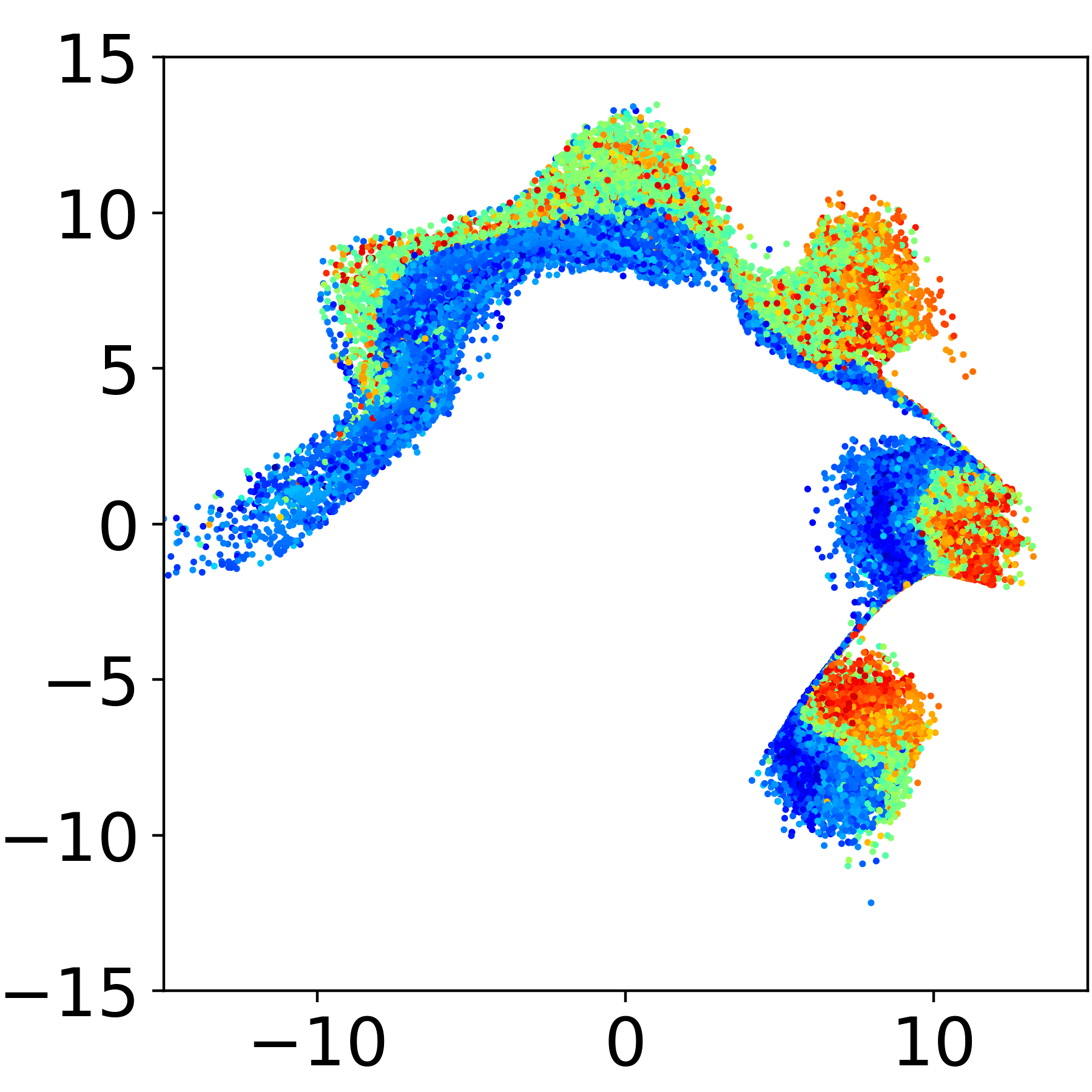}
    \end{tabular}
    \caption{Sampling a mixture of 8 Gaussian distributions. Top row from left to right: Target distribution, MetFlow, MetFlow with 145 resampled innovation noise. Bottom row from left to right: Prior distribution, First run of RNVP, Second run of RNVP. MetFlow finds all the modes and improves with more iterations, while RNVP depend on a good initialization to find all the modes and fails to separate them correctly.}
  \label{fig:8gmm_all}
  \end{figure}  
%  \begin{figure}
%    \begin{center}
%      \includegraphics[width=\linewidth]{pics/gmm8_all_dpi400.png}
%      \caption{Density matching for mixture of 8 Gaussians.}
%    \end{center}
%  \label{fig:8gmm_all}
%  \end{figure}
  First, we observe that while our method successfully finds all modes of the target distribution, RNVP alone struggles to do the same. Our method is therefore able to approximate multimodal distributions with well separated modes. Here, the mixture structure of the distribution (with potentially $3^5 = 243$ modes) produced by MetFlow is very appropriate to such a problem. On the contrary, classical flows are unable to approximate well separated modes starting from a simple unimodal prior, without much surprise. In particular, mode dropping is a serious issue even in small dimension. Moreover, an other advantage of MetFlow in the pseudo randomized setting is to be able to iterate the learnt kernels which still preserve the target distribution.
  Iterating MetFlow kernels widens the gap between both approaches, significantly improving the accuracy of our approximation.
  \begin{figure}[h!]
    \centering
    \includegraphics[width= 0.85\linewidth]{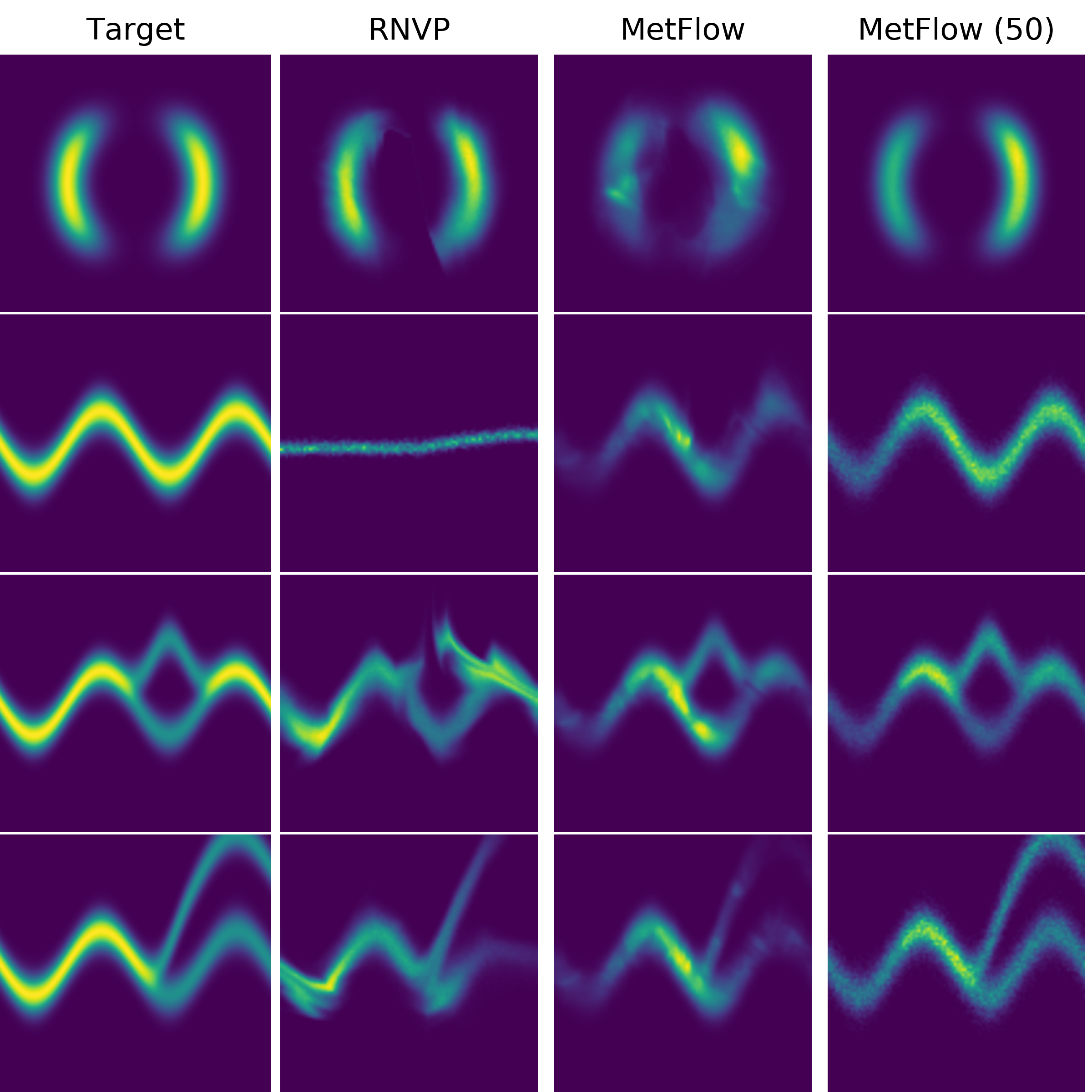}
    \caption{Density matching example~\cite{pmlr-v37-rezende15} and comparison between RNVP and MetFlow.}
    \label{fig:rezende}
  \end{figure}

\subsubsection{Non-Gaussian 2D Distributions}  
  \begin{figure}[h!]
    \centering
    \includegraphics[width=0.4\linewidth]{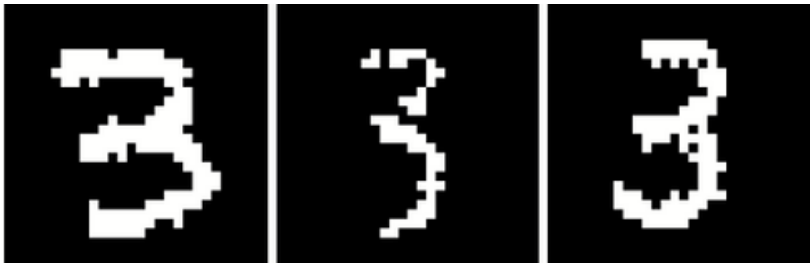}

    \bigskip

    \includegraphics[width=0.4\textwidth]{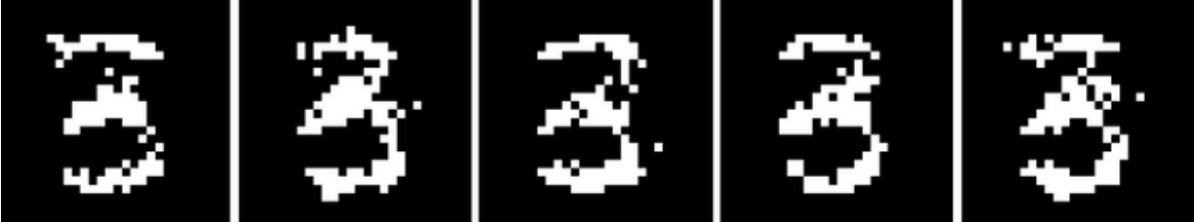}
    \includegraphics[width=0.4\textwidth]{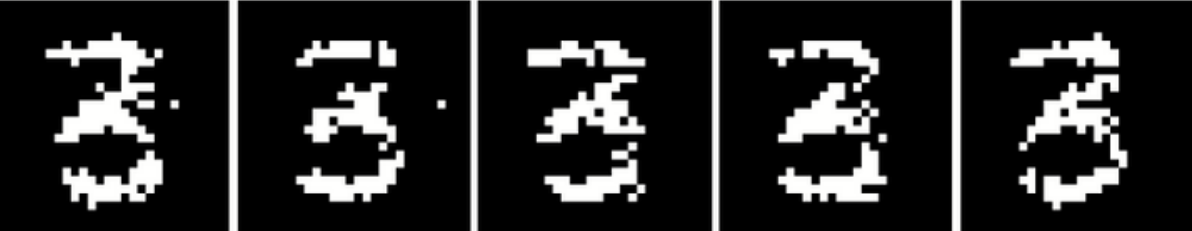}

    \bigskip

    \includegraphics[width=0.4\textwidth]{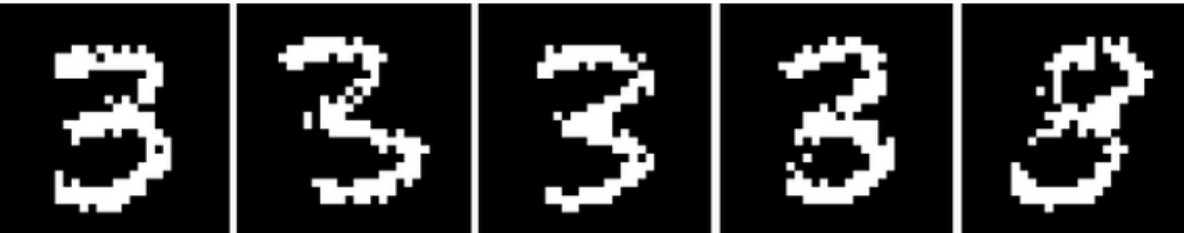}
    \includegraphics[width=0.4\textwidth]{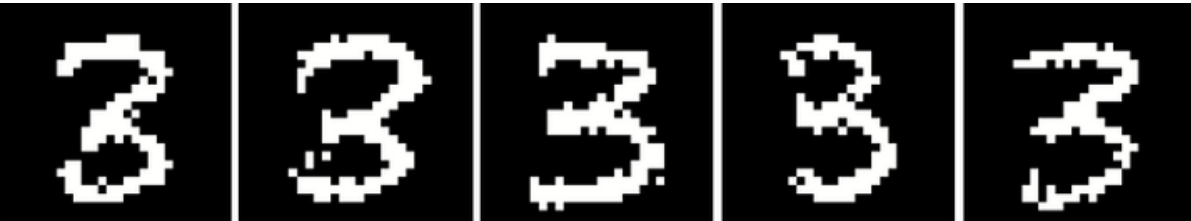}
    \caption{Mixture of '3' digits. Top: Fixed digits, Middle: NAF samples, Bottom: MetFlow samples. Compared to NAF, MetFlow is capable to mix better between these modes, while NAF seems to collapse.}
    \label{fig:mixture_of_3}
  \end{figure}

  In a second experiment, we sample the non-Gaussian 2D distributions proposed in~\cite{pmlr-v37-rezende15}. \Cref{fig:rezende} illustrates the performance of MetFlow compared to RNVP. We are again using 5 RNVPs ($K = 5$) with the architecture described above, and use the pseudo-randomized setting for MetFlow. After only five steps, MetFlow already finds the correct form of the target distribution, while the simple RNVP fails on the more complex distributions. Moreover, iterating again MetFlow kernels allows us to approximate the target distribution with striking precision, after only 50 MCMC steps.
  \begin{figure*}[h!]
    \includegraphics[width=\linewidth]{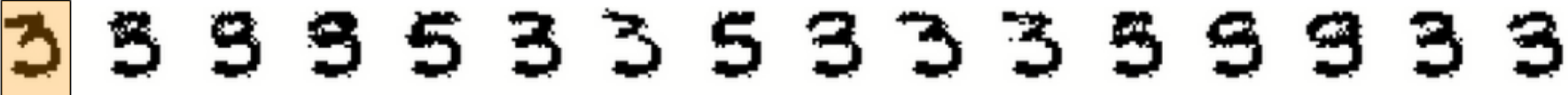}
    \includegraphics[width=\linewidth]{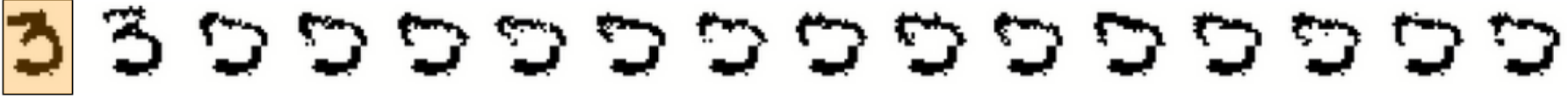}
    \includegraphics[width=\linewidth]{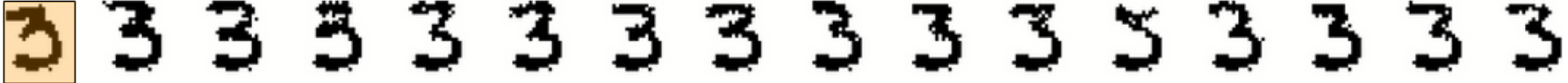}
    \caption{Top line: Mean-Field approximation and MetFlow, Middle line: Mean-Field approximation, Bottom line: Mean-Field Approximation and NAF. Orange samples on the left represent the initialization image. We observe that MetFlow easily mixes between the modes while other methods are stuck in one mode.}
    \label{fig:gibbs_sampling}
  \end{figure*}

\subsection{Deep Generative Models}
  Deep Generative Models (DGM), such as Deep Latent Gaussian Models (see~\citet{kingma2013auto,rezende2014stochastic}) have recently become very popular.
  The basic assumption in a DGM is that the observed data \(x\) is generated by sampling a latent vector \(z\) which is used as the input of a deep neural network.
  This network then outputs the parameters of a family of distributions (e.g., the canonical parameters of exponential family like Bernoulli or Gaussian distributions) from which the data are sampled. Given data generated by a DGM, a classical problem is to approach the conditional distribution $p(z \mid x)$ of the latent variables $z$ given the observation $x$, using variational inference to construct an amortized approximation.

 % \begin{figure*}[h!]
 %   \includegraphics[width=\linewidth]{pics/pdf/Gibbs_sampling.pdf}
 %   \caption{Gibbs sampling example showing capability of MetFlow to mix between different modes of underlying posterior distribution.}
 % \label{fig:gibbs_sampling}
 % \end{figure*}

  We consider the binarized MNIST handwritten digit dataset. The generative model is as follows. The latent variable $z$ is a $l=64$ dimensional standard normal Gaussian. The observation $x=(x^j)_{j=1}^D$ is a vector of $D=784$ bits. The bits $(x^j)_{j=1}^D$ are, given the latent variable $z$, conditionally independent Bernoulli distributed random variables with success probability $p_\theta(z)^j$ where $(p_\theta^j)_{j=1}^D$ is the output of a convolutional neural network. In this framework, $p_\theta$ is called the decoder.
  In the following, we show that our method provides a flexible and accurate variational approximation of the conditional distribution of the latent variable given the observation $p_\theta(z\mid x)$, outperforming mean-field and Normalizing Flows based approaches.

  %For the experiment we fix a State of the Art VAE.

  As we are focusing in this paper on the comparison of VI methods to approximate complex distributions and not on learning the Variational Auto Encoder itself, we have chosen to use a fixed decoder for both Normalizing Flows (here, Neural Autoregessive Flows) and MetFlow (with RNVP transforms).
  The decoder is obtained using state-of-the-art method described in the supplementary paper.
  We can illustrate the expressivity of MetFlow in two different ways. We first fix $L$ different samples. In this example, we take $L=3$ images representing the digit ``3''. We are willing to approximate, for a given decoder $p_\theta$, the posterior distribution $p_\theta(z|(x_i)_{i=1}^L)$. We show in \Cref{fig:mixture_of_3} the decoded samples corresponding to the following variational approximations of $p_\theta(\cdot|(x_i)_{i=1}^L)$: (i)
  %the simple mean-field approximation provided by the encoder of the fixed VAE, (ii) the encoder and the NAF of the fixed VAE, (iii)
  a NAF trained from the decoder to approximate $p_\theta(\cdot|(x_i)_{i=1}^L)$ and (ii) \textit{MetFlow} in the deterministic setting with $K=5$ RNVP flows.

  \Cref{fig:mixture_of_3} shows that the samples generated from (i) collapse essentially to one mode corresponding to the first digit. On the contrary, MetFlow is able to capture the three different modes of the posterior and generates much more variability in the decoded samples. The same phenomenon is observed in different settings by varying $L$ and the digits chosen, as illustrated in the supplementary paper.

  We now consider the in-painting set-up introduced in~\citep[Section 5.2.2]{levy2017generalizing}. Formally, we in-paint the top of an image using Block Gibbs sampling.
  Given an image $x$, we denote $x^t$, $x^b$ the top and the bottom half pixels. Starting from an image $\mathbf{x_0}$, we sample at each step $z_t \sim p_\theta(z\mid x_t)$ and then $\tilde{x}\sim p_\theta(x\mid z_t)$. We the set $x_{t+1} = (\tilde{x}^{t}, x_0^b)$. We give the output of this process when sampling from the mean-field approximation of the posterior only, the mean-field pushed by a NAF, or using our method.
  The result for the experiment can be seen on \Cref{fig:gibbs_sampling}.

  We can see that MetFlow mixes easily between different modes, and produces sharp images. We recognize furthermore different digits (3,5,9). It is clear from the middle plot that the mean-field approximation is not able to capture the complexity of the distribution $p_\theta(z\mid x)$. Finally, the NAF improves the quality of the samples but does not compare to MetFlow in terms of mixing.

%  \begin{figure}[ht!]
%    \includegraphics[width=\linewidth]{pics/pdf/mixture_of_3_another_plot.pdf}
%    \caption{Mixture of ``3'' digits. Compare to NAF, MetFlow is capable to mix better between these modes, while NAF seems to collapse to something in the middle. }
%  \label{fig:mixture_of_3}
%  \end{figure}

%%% Local Variables:
%%% mode: latex
%%% TeX-master: "main"
%%% End:

%% file: Conclusion.tex
In this paper, we propose a novel approach to combine MCMC and VI which alleviates the computational bottleneck of previously reported methods. In addition, we design \emph{MetFlow}, a particular class of MH algorithms which fully takes advantage of our new methodology while capitalizing on the great expressivity of NF. Finally, numerical experiments highlight the benefits of our method compared to state-of-the-arts VI.

%Further theoretical and numerical comparisons with existing VI schemes are under consideration and more precisely, we aim at comparing the different ELBO related to each method.
This work leads to several natural extensions. All NF applications can be adapted with \emph{MetFlow}, which can be seen as a natural extension of a NF framework. \emph{MetFlow} are very appropriate for VAE by amortizing. Due to lack of space, we did not present applications with Forward KL divergence. The mixture structure of the distribution obtained by \emph{MetFlow} suggests the Variational Expectation Maximization (VEM) is a sensible strategy and in particular a chunky version of VEM in the case where the number of steps $K$ is large~\cite{verbeek:vlassis:2003}.
%First, several MCMC methods do not fall within the framework we set, especially the ones based on data augmentation, such as the full refresh HMC, or some Gibbs samplers \cite{albert:chib:1993}.We would like to generalize our methodology to include such MCMC algorithms.
%In a second line of work, we plan to bypass the evaluation of the MH acceptance ratio using a Variational Expectation Maximization (VEM) strategy and in particular a chunky version of VEM in the case where the number of steps $K$ is large \cite{verbeek:vlassis:2003}. Finally, VAE is a natural application of our variational family and its corresponding ELBO and is the subject of an ongoing work. 

%% file: SupplementPaper.tex
% !TEX root = main_supplement.tex

\section{Proofs}
\subsection{Proof of \Cref{coro:induction-argument}}
\label{SPsec:proof-coro-induction}
  % \Cref{prop:density-one-iteration} allow us to write the density of the first iterate of the Markov chain defined by the iteration of kernel~\eqref{eq:def_M}, given the ``innovation'' noise $u$, using the initial density $m^0$.
  % In general, \Cref{prop:density-one-iteration} allow us to write the density of the $K$-th iterate of the Markov chain given the ``innovation'' noise at time $K$ using the density of the $K-1$-th iterate of the Markov Chain, supposing it exists. We set the notations for a mapping $\fwdtransfo\colon \rset^D \to \rset^D$, define  $\fwdtransfo^0=\Id$. For a family
  % $\{\fwdtransfo_i\}_{i=1}^D$ of mappings on $\rset^D$, define $\bigcirc_{j=i}^k \fwdtransfo_j = \fwdtransfo_i \circ \dots \circ \fwdtransfo_k$. Finally, set $\alpha^1_{\phi,u}(z) = \alpha_{\phi,u}(z)$ and $\alpha^0_{\phi,u}(z) = 1-\alpha_{\phi,u}(z)$.
  The proof is by induction on $K \in \nsets$. The base case $K=1$ is given by \Cref{prop:density-one-iteration}. Assume now that the statement holds for $K-1 \in \nsets$. Then noticing that $\marginal[\phi]{K}(\cdot|\chunk{u}{1}{K}) =\marginal[\phi]{K-1}(\cdot|\chunk{u}{1}{K-1})\Mtransw{\phi}{u_K}$,   using again \Cref{prop:density-one-iteration} and the induction hypothesis, we get that $\marginal[\phi]{K}(\cdot|\chunk{u}{1}{K})$ admits a density $\margindensmu{\phi}{K}{\cdot}{\chunk{u}{1}{K}}$ \wrt~the Lebesgue measure given for any $z \in \rset^D$ by 
  \begin{align}
  \label{eq:margin_densk-1}
    \margindensmu{\phi}{K}{z}{\chunk{u}{1}{K}} &= \alpha_{\phi,u_K}(\fwdtransfoparam{\phi}{u_K}^{-1}(z)) \margindensmu{\phi}{K-1}{\fwdtransfoparam{\phi}{u_K}^{-1}(z)}{\chunk{u}{1}{K-1}}J_{\fwdtransfoparam{\phi}{u_K}^{-1}}(z)+ \{1-\alpha_{\phi,u_K}(z)\}\margindensmu{\phi}{K-1}{z}{\chunk{u}{1}{K-1}} \eqsp,
    \nonumber
    \\
    \nonumber
    & = \sum_{a_K \in \{0,1\}} \alpha^{a_K}_{\phi, u_K}(\fwdtransfoparam{\phi}{u_K}^{-a_K}(z))J_{\fwdtransfoparam{\phi}{u_K}^{-a_K}}(z)\margindensmu{\phi}{K-1}{\fwdtransfoparam{\phi}{u_K}^{-a_K}(z)}{\chunk{u}{1}{K-1}} \eqsp.
  \end{align}

  Using the induction hypothesis, the density $\margindensmu{\phi}{K-1}{\cdot}{\chunk{u}{1}{K-1 }}$ \wrt~the Lebesgue measure of $\marginal[\phi]{K-1}(\cdot|\chunk{u}{1}{K-1})$ of the form~\eqref{eq:margin_densK}. Therefore, we obtain that
  \begin{align}
    &    \margindensmu{\phi}{K}{z}{\chunk{u}{1}{K}} % = \sum_{a_K \in \{0,1\}} \alpha^{a_K}_{\phi, u_K}(\fwdtransfoparam{\phi}{u_K}^{-a_K}(z))J_{\fwdtransfoparam{\phi}{u_K}^{-a_K}}(z)\margindensmu{\phi}{K-1}{\fwdtransfoparam{\phi}{u_K}^{-a_K}(z)}{\chunk{u}{1}{K-1}}
    % \nonumber
    % \\
    = \sum_{a_K \in \{0,1\}}  \alpha^{a_K}_{\phi, u_K}(\fwdtransfoparam{\phi}{u_K}^{-a_K}(z))J_{\fwdtransfoparam{\phi}{u_K}^{-a_K}}(z) \left[\sum_{a_{1:K-1} \in \{0,1\}^{K-1}} \margindensm{\phi}{0}(\bigcirc_{j=1}^{K-1}\fwdtransfoparam{\phi}{u_j}^{-a_j}(\fwdtransfoparam{\phi}{u_K}^{-a_K}(z))) \right.
    \nonumber
    \\
    &\qquad \qquad \qquad \times \left. J_{\bigcirc_{j=1}^{K-1}\fwdtransfoparam{\phi}{u_j}^{-a_j}}(\fwdtransfoparam{\phi}{u_K}^{-a_K}(z))
    \prod_{i=1}^{K-1} \alpha^{a_i}_{\phi,u_i}(\bigcirc_{j=i}^{K-1}\fwdtransfoparam{\phi}{u_j}^{-a_j}(\fwdtransfoparam{\phi}{u_K}^{-a_K}(z)))\right]
    \nonumber
    \\
    &= \sum_{\chunk{a}{1}{K} \in \{0,1\}^{K}}\margindensm{\phi}{0}(\bigcirc_{j=1}^{K}\fwdtransfoparam{\phi}{u_j}^{-a_j}(z))  J_{\fwdtransfoparam{\phi}{u_K}^{-a_K}}(z)J_{\bigcirc_{j=1}^{K-1}\fwdtransfoparam{\phi}{u_j}^{-a_j}}(\fwdtransfoparam{\phi}{u_K}^{-a_K}(z)) \alpha^{a_K}_{\phi, u_K}(\fwdtransfoparam{\phi}{u_K}^{-a_K}(z))\prod_{i=1}^{K-1} \alpha^{a_i}_{\phi,u_i}(\bigcirc_{j=i}^{K}\fwdtransfoparam{\phi}{u_j}^{-a_j}(z))
    \nonumber
    \\
    \nonumber
    &=\sum_{\chunk{a}{1}{K} \in \{0,1\}^{K}}\margindensm{\phi}{0}(\bigcirc_{j=1}^{K}\fwdtransfoparam{\phi}{u_j}^{-a_j}(z))  J_{\bigcirc_{j=1}^{K}\fwdtransfoparam{\phi}{u_j}^{-a_j}}(\fwdtransfoparam{\phi}{u_K}^{-a_K}(z)) \prod_{i=1}^{K} \alpha^{a_i}_{\phi,u_i}(\bigcirc_{j=i}^{K}\fwdtransfoparam{\phi}{u_j}^{-a_j}(z))\eqsp,
  \end{align}
  where in the last step, we have used that for any differentiable functions $\psi_1,\psi_2\colon \rset^D \to \rset^D$,  $J_{\psi_1\circ \psi_2}(z) = J_{\psi_1}(\psi_2(z)) J_{\psi_2}(z)$ for any $z \in \rset^D$.
   
  % Let us now prove by induction that the distribution $\marginal[\phi]{K}(\cdot|\chunk{u}{1}{K})= \marginal[\phi]{0}\Mtransw{u_1}{\phi}\cdots \Mtransw{u_K}{\phi}$ has a density $\margindensm{\phi}{K}$ given by
  % \begin{align}
  % \label{eq:margin_densK}
  %   &\margindensmu{\phi}{K}{z}{\chunk{u}{1}{K}} =  \sum_{\chunk{a}{1}{K} \in \{0,1\}^K}\margindensmu{\phi}{K}{z, \chunk{a}{1}{K}}{\chunk{u}{1}{K}} \eqsp,
  % \end{align}
  % where
  % \begin{align}
  %   \margindensmu{\phi}{K}{z, \chunk{a}{1}{K}}{\chunk{u}{1}{K}} = \margindensm{\phi}{0}(\bigcirc_{j=1}^K\fwdtransfoparam{\phi}{u_j}^{-a_j}(z))J_{\bigcirc_{j=1}^K\fwdtransfoparam{\phi}{u_j}^{-a_j}}(z)
  %   \prod_{i=1}^K \alpha^{a_i}_{\phi,u_i}(\bigcirc_{j=i}^K\fwdtransfoparam{\phi}{u_j}^{-a_j}(z))\eqsp.
  % \end{align}
  % \Cref{prop:density-one-iteration} gives this result for $K=1$. Let us now suppose it is true for some $K-1 \in \nsets$.

  % We have, as observed above
  % \begin{equation}
  %   \margindensmu{\phi}{K}{z}{\chunk{u}{1}{K}} = \sum_{a_K \in \{0,1\}} \alpha^{a_K}_{\phi, u_K}(\fwdtransfoparam{\phi}{u_K}^{-a_K}(z))J_{\fwdtransfoparam{\phi}{u_K}^{-a_K}}(z)\margindensmu{\phi}{K-1}{\fwdtransfoparam{\phi}{u_K}^{-a_K}(z)}{\chunk{u}{1}{K-1}}
  % \end{equation}
  % and we have an expression for $\margindensm{\phi}{K-1}(\fwdtransfoparam{\phi}{u_K}^{-a_K}(z))$ by the induction hypothesis.
  % We can thus derive:

\subsection{Proof of \Cref{propo:2}}
  Let $(u,\phi)\in \msu \times \msphi$.  We want to find a condition such that the kernel $R_{\phi, \tu}$ defined by~\eqref{eq:def_markov_Q} is reversible \wrt~$\pi\otimes \nu$ where $\nu$ is a distribution on $\msv$.
  This means that for any $\msa_1, \msa_2 \in \mcb{\rset^D}$, $\msb_1, \msb_2 \subset \msv$,
  \begin{equation}
  \label{eq:reversibility_Rphi}
    \int_{\msa_1 \times \msa_2}\sum_{(v,v') \in \msb_1 \times \msb_2} \pi(z) \nu(v) R_{\phi, \tu}((z,v), \rmd z'\times \{v'\}) \rmd z =  \int_{\msa_2 \times \msa_1}\sum_{(v,v') \in \msb_2 \times \msb_1} \pi(z) \nu(v) R_{\phi, \tu}((z,v), \rmd z'\times \{v'\}) \rmd z \eqsp. 
  \end{equation}
  By definition of $R_{\phi, \tu}$ \eqref{eq:def_markov_Q}, the left-hand side simplifies to
  \begin{align}
    \nonumber
    &\int_{\msa_1 \times \msa_2}\sum_{(v,v') \in \msb_1 \times \msb_2} \pi(z) \nu(v) R_{\phi, \tu}((z,v), \rmd z'\times \{v'\}) \rmd z =\int_{\msa_1 }\sum_{v \in \msb_1 } \pi(z) \nu(v) R_{\phi, \tu}((z,v), \msa_2\times\msb_2) \rmd z\\
      \nonumber
    & \qquad \qquad = \int_{\msa_1 }\sum_{v \in \msb_1 }\pi(z) \nu(v) \left\{ \dmhratio(z,v) \updelta_{\invtransfo_{\phi, \tu}(z,v)}(\msa_2\times\msb_2) + \{1-\dmhratio(z,v)\} \updelta_{(z,v)} (\msa_2\times\msb_2) \right\}\rmd z  \\
    \label{eq:rev_proof_first_term_1}
    & \qquad \qquad  = \mathfrak{I}+\int_{\rset^D} \sum_{v \in \msv }\pi(z) \nu(v)   \{1-\dmhratio(z,v)\} \Ind_{\msa_2\times\msb_2}(z,v)  \Ind_{\msa_1\times\msb_1}(z,v)\rmd z \eqsp,
    % \int_{\rset^D} \sum_{v \in \msv }\pi(z) \nu(v) \left\{ \dmhratio(z,v) \Ind_{\msa_2\times\msb_2}(\invtransfo_{\phi,\tu}(z,v)) + \{1-\dmhratio(z,v)\} \Ind_{\msa_2\times\msb_2}(z,v) \right\} \Ind_{\msa_1\times\msb_1}(z,v)\rmd z.
  \end{align}
  where using that $\invtransfo_{\phi,\tu}$ is an involution, and   for $v \in \msv$, the change of variable $\tilde{z} = \fwdtransfo^{-v}_{\phi,\tu}(z)$,
  %  Then we can separate this integrals in two terms. Let us consider the first one:
  \begin{align}
    \nonumber
    \mathfrak{I}& = \int_{\rset^D} \sum_{v \in \msv }\pi(z) \nu(v) \dmhratio(z,v) \Ind_{\msa_2\times\msb_2}(\invtransfo_{\phi,\tu}(z,v))\Ind_{\msa_1\times\msb_1}(z,v)\rmd z  \\
    \nonumber
    &= \int_{\rset^D}\sum_{v \in \msv }\dmhratio(\fwdtransfo^{v}_{\phi,\tu}(\tilde{z}),v) J_{\fwdtransfo^{-v}_{\phi,\tu}}(\tilde{z})\pi(\fwdtransfo^{v}_{\phi,\tu}(\tilde{z}))\nu(v) \Ind_{\msa_2\times\msb_2}(\tilde{z},-v)\Ind_{\msa_1\times\msb_1}(\fwdtransfo^{v}_{\phi,\tu}(\tilde{z}),v)\rmd \tilde{z} \\
    \label{eq:rev_proof_first_term_2}
    &= \int_{\rset^D}\sum_{\tilde{v} \in \msv }\dmhratio(\invtransfo_{\phi,\tu}(\tilde{z},\tilde{v})) J_{\fwdtransfo^{\tilde{v}}_{\phi,\tu}}(\tilde{z})\pi(\fwdtransfo^{\tilde{v}}_{\phi,\tu}(\tilde{z}))\nu(-\tilde{v}) \Ind_{\msa_2\times\msb_2}(\tilde{z},\tilde{v})\Ind_{\msa_1\times\msb_1}(\invtransfo_{\phi,\tu}(\tilde{z},\tilde{v}))\rmd \tilde{z}  \eqsp,
  \end{align}
  where in the last step, we have the change of variable $\tilde{v} = -v$. As for the  right-hand side of~\eqref{eq:reversibility_Rphi}, we have by definition~\eqref{eq:def_markov_Q},  
  \begin{align}
    \nonumber
    &\int_{\msa_2 \times \msa_1}\sum_{(v,v') \in \msb_2 \times \msb_1} \pi(z) \nu(v) R_{\phi, \tu}((z,v), \rmd z'\times \{v'\}) \rmd z = \int_{\msa_2 }\sum_{v \in \msb_2 } \pi(z) \nu(v) R_{\phi, \tu}((z,v), \msa_1\times\msb_1) \rmd z\\
    \nonumber
    &= \int_{\rset^D}\sum_{v \in \msv }\pi(z) \nu(v) \left\{ \dmhratio(z,v) \Ind_{\msa_1\times\msb_1}(\invtransfo_{\phi,\tu}(z,v)) + \{1-\dmhratio(z,v)\} \Ind_{\msa_1\times\msb_1}(z,v) \right\} \Ind_{\msa_2\times\msb_2}(z,v)\rmd z \eqsp.
  \end{align}
  Therefore combining this result with~\eqref{eq:rev_proof_first_term_1}-\eqref{eq:rev_proof_first_term_2}, we get that if \eqref{eq:condition_alpha} holds then $R_{\phi,u}$ is reversible \wrt~$\pi\otimes \nu$. 
  % and we see the condition to satisfy reversibility:
  % \begin{equation}
  %   \dmhratio(z,v) \pi(z) \nu(v) = \dmhratio(\invtransfo_\phi(z,v)) \pi(\fwdtransfoparam{\phi}{\tu}^v(z)) \nu(-v) J_{\fwdtransfoparam{\phi}{\tu}^v}(z).
  % \end{equation}
  % %
  % Hence \eqref{eq:condition_alpha} follows.

  Moreover, let us suppose that $\varphi\colon \brset_+ \to \brset_+$ is a function satisfying $\varphi(\plusinfty)=1$ and for any $t \in \brset_+$, $t\varphi(1/t) = \varphi(t)$.
  If for any $z \in \rset^D$, $\tu \in \tmsu$, $v \in\msv$, $$\dmhratio(z,v) = \varphi\parentheseDeux{\frac{\pi(\fwdtransfoparam{\phi}{\tu}^v(z))\nu(-v) J_{\fwdtransfoparam{\phi}{\tu}^{v}}(z)}{\pi(z)\nu(v)}}\eqsp,$$ then
  \begin{align}
    \dmhratio(\invtransfo_\phi(z,v)) \pi(\fwdtransfoparam{\phi}{\tu}^v(z)) \nu(-v) J_{\fwdtransfoparam{\phi}{\tu}^v}(z) &= \varphi\parentheseDeux{\frac{\pi(z)\nu(v) J_{\fwdtransfoparam{\phi}{\tu}^{-v}}(\fwdtransfoparam{\phi}{\tu}^v(z))}{\pi(\fwdtransfoparam{\phi}{\tu}^v(z)) \nu(-v)}}\pi(\fwdtransfoparam{\phi}{\tu}^v(z)) \nu(-v)  J_{\fwdtransfoparam{\phi}{\tu}^v}(z)
    \nonumber
    \\
    & = \varphi\parentheseDeux{\frac{\pi(z)\nu(v) }{\pi(\fwdtransfoparam{\phi}{\tu}^v(z))\nu(-v)J_{\fwdtransfoparam{\phi}{\tu}^v}(z)}} \frac{\pi(\fwdtransfoparam{\phi}{\tu}^v(z)) \nu(-v)J_{\fwdtransfoparam{\phi}{\tu}^v}(z)}{\pi(z)\nu(v)}  \pi(z)\nu(v)
    \nonumber
    \\
    \nonumber
    & = \varphi\parentheseDeux{\frac{\pi(\fwdtransfoparam{\phi}{\tu}^v(z))\nu(-v)J_{\fwdtransfoparam{\phi}{\tu}^v}(z)}{\pi(z)\nu(v)}}\pi(z)\nu(v)
    =\dmhratio(z,v)\pi(z)\nu(v)\eqsp,
  \end{align}
  which concludes the proof for~\Cref{propo:2}

\subsection{Proof of \Cref{coro:Reversibility_Mphi_Rphi}}
  Let us suppose that for any $\tu \in \tmsu$, $\dmhratio$ are chosen such that $R_{\phi, u}$ defined by~\eqref{eq:def_markov_Q} is $\pi \otimes \nu$ invariant, where $\nu$ is any distribution on $\msv$.
  Then, by definition, for any $\msa_1, \msa_2 \in \mcb{\rset^D}$, $\msb_1, \msb_2 \subset {\msv}$, we have:
  \begin{align*}
    &\int_{\msa_1 \times \msa_2}\sum_{(v,v') \in \msb_1 \times \msb_2} \pi(z) \nu(v) R_{\phi, \tu}((z,v), \rmd z'\times \{v'\}) \rmd z =  \int_{\msa_2 \times \msa_1}\sum_{(v,v') \in \msb_2 \times \msb_1} \pi(z) \nu(v) R_{\phi, \tu}((z,v), \rmd z'\times \{v'\}) \rmd z
    \\
    &\int_{\msa_1 }\sum_{v \in \msb_1 } \pi(z) \nu(v) R_{\phi, \tu}((z,v), \msa_2\times\msb_2) \rmd z = \int_{\msa_2 }\sum_{v \in \msb_2 } \pi(z) \nu(v) R_{\phi, \tu}((z,v), \msa_1\times\msb_1) \rmd z\eqsp.
  \end{align*}
  This is true for any $\msb_1, \msb_2 \subset {\msv}$, hence in particular if $\msb_1= \msb_2=\msv$. Then
  \begin{align*}
    &\int_{\msa_1 }\sum_{v \in \msv }\pi(z) \nu(v) R_{\phi, \tu}((z,v), \msa_2 \times \msv) \rmd z   = \int_{\msa_2 }\sum_{v \in \msv } \pi(z) \nu(v) R_{\phi, \tu}((z,v), \msa_1 \times \msv) \rmd z 
    \\
    &\int_{\msa_1 }\pi(z)\sum_{v \in \msv } \nu(v)Q_{\phi,(\tu, v)}(z, \msa_2)\rmd z  = \int_{\msa_2}\pi(z) \sum_{v \in \msv }\nu(v)Q_{\phi,(\tu, v)}(z, \msa_1)\rmd z.
  \end{align*}
  By definition~\eqref{eq:link_Rphi_Qphi}.
  In particular, we obtain exactly, for any $\tu \in\tmsu$, 
  \begin{equation*}
    \int_{\msa_1}\pi(z) M_{\phi,\tu, \nu}(z, \msa_2)\rmd z  = \int_{\msa_2}\pi(z)M_{\phi,\tu, \nu}(z, \msa_1)\rmd z.
  \end{equation*}
  Thus concluding that $M_{\phi,\tu, \nu}$ is reversible \wrt~$\pi$, for any distribution $\nu$ and any $\tu\in\tmsu$.

\subsection{Checking the Assumption of \Cref{prop:density-one-iteration} for RWM and MALA algorithms}
\label{SPsec:checking-RMW-MALA}
  \paragraph{RWM:}
  For any $u \in \rset^D$ and $\phi$,
  \[
    \fwdtransfo^{\RWM}_{\phi,u}(z) = z + \Sigma_\phi^{1/2} u \eqsp,
  \]
  which clearly is a $\rmC^1(\rset^D,\rset^D)$ diffeomorphism with inverse
  \[
    \{\fwdtransfo^{\RWM}_{\phi,u}\}^{-1}(y)= y - \Sigma_\phi^{1/2} u \eqsp.
  \]
  In the simple case where
  $J_{\fwdtransfo^{\RWM}_{\phi,u}}(z)= 1$,
  using \Cref{coro:induction-argument}, we get
  \begin{equation}
  \label{eq:margin_densK_RWM}
    \margindensmu{\phi}{K}{z, \chunk{a}{1}{K}}{\chunk{u}{1}{K}} =\margindensm{\phi}{0}\left(z-\sum_{j=1}^{k} a_j u_j\right) \prod_{i=1}^K \alpha^{a_i}_{\phi,u_i}\left(z-\sum_{j=i}^{k} a_j u_j\right)
  \end{equation}

  \paragraph{MALA:} % Assume that the potential $\log(\pi)$ is gradient Lipschitz with constant $L$, and $\gamma \leq 1/{2L}$, then for any $u$ in $\rset^D$, $\fwdtransfoparam{\gamma}{u}$ is a $\rmC^1(\rset^D,\rset^D)$ diffeomorphism; see \Cref{prop:MALA_diff}. 

  We prove here that under appropriate conditions, the transformations defined by the Metropolis Adjusted Langevin Algorithm (MALA) are $\rmC^1$ diffeomorphisms. We consider only the case where $\Sigma_{\phi} = \gamma \Id$. The general case can be easily deduced  by a simple adaptation.
  Remember, for MALA,  $\mathcal{T} = \{\fwdtransfo_{\gamma, u}\colon z \mapsto z+ \gamma \nabla U(z) + \sqrt{2\gamma} u \, :\, u \in \rset^D , \gamma >0\}$, where $U$ is defined as $U(z) = \log(\tpi(z))$,
  \begin{proposition}
  \label{prop:MALA_diff}
    Assume the potential $U$ is gradient Lipschitz, that is there exists $L$ in $\rset^+$ such that for any $z_1,z_2 \in \rset^D$, $\|\nabla U(z_1) - \nabla U(z_2)\| \leq L\|z_1-z_2\|$, and that $\gamma \leq  1/(2L)$.
    Then for any $u$ in $\rset^D$, $\fwdtransfo_{\gamma, u}\colon z\mapsto z+\gamma \nabla U(z) + \sqrt{2\gamma}u$ is a $\rmC^1$ diffeomorphism.
  \end{proposition}
  \begin{proof}
    Let $\gamma \leq 1/(2L)$ and  $u \in \rset^D$.
    First we show that $\fwdtransfo_{\gamma,u}$ is invertible. 
    Consider, for each $y$ in $\rset^D$, the mapping $H_{y,u}(z) = y - \sqrt{2\gamma} - \gamma \nabla U(z)$.
    We have, for $z_1, z_2 \in \rset^D$,
    \begin{EQA}[c]
    \label{eq:H_contracting}
      \|H_{y,u}(z_1) - H_{y,u}(z_2)\| \leq \|\nabla U(z_1) - \nabla U(z_2)\| \leq \gamma L \|z_1-z_2\|
    \end{EQA}
    and $ \gamma L \leq 1/2$. Hence $H_{y,u}$ is a contraction mapping and thus has a unique fixed point $z_{y,u}$ and we have:
    \begin{EQA}[c]
    \label{eq:Invertibility_MALA}
      H_{y_u}(z_{y,u}) = z_{y,u} \Rightarrow y = z_{y,u} + \nabla U(z_{y,u}) + \sqrt{2\gamma}u
    \end{EQA}
    and existence and uniqueness of the fixed point $z_{y,u}$ thus complete the proof for invertibility of $\fwdtransfo_{\gamma,u}$. The fact that the inverse of $\fwdtransfo_{\gamma,u}$ is $\rmC^1$ follows from a simple application of the local inverse function theorem. 
  \end{proof}
  Therefore, \Cref{coro:induction-argument} can be applied again. Although there is no explicit expression available for $\margindensm{\gamma}{K}$ because of the intractability of the inverse of $\fwdtransfoparam{\gamma}{u}$, numerical approximations can be used.

\section{Reparameterization trick and estimator of the gradient}
\subsection{Expression for the reparameterization trick}
\label{SPsec:reparam_trick}
  The goal of the reparameterization trick is to rewrite a distribution
  depending on some parameters as a simple transformation of a fixed
  one. The implementation of this procedure is a bit more involved in
  our case, as the integration is now done on a mixture of the
  components
  $\margindensmu{\phi}{K}{z,\chunk{a}{1}{K}}{\chunk{u}{1}{K}}$, for
  $\chunk{a}{1}{K} \in \{0,1\}^K$.  To develop an understanding of the
  methodology we suggest, we consider first the case $K=1$. Recall that 
  $\varphibf$ stands for the density of the standard Gaussian distribution
  over $\mathbb{R}^D$, and suppose here that there exists $V_{\phi}\colon \rset^D \to \rset^D$ a $\rmC^1$ diffeomorphism such that for any $z \in \rset^D$, 
  $m_\phi^0(z) = \varphibf(V_{\phi}^{-1}(z)) J_{V_\phi}(V_{\phi}^{-1}(z))$ which is the basic assumption of the
  reparameterization trick. With the two changes of variables, $\tilde{z} = \fwdtransfoparam{\phi}{u_1}^{-a_1}(z)$ and $y=V^{-1}_{\phi}(\tilde{z})$,
  we get
  \begin{align}
    &\lowerboundaux(\phi) = \int h_1(u_1)  \margindensmu{\phi}{1}{z, a_1}{u_1} \log \left(\frac{\tilde{\pi}(z) r(a_1|z, u_1) }{\margindensmu{\phi}{1}{z, a_1}{u_1}}\right)\rmd z \rmd a_1 \rmd \mu_{\msu}(u_1)
    \nonumber
    \\
    &= \hspace{-10pt}\sum_{a_1 \in \{0,1\}} \int  h_1(u_1) \margindensm{\phi}{0}(\fwdtransfoparam{\phi}{u_1}^{-a_1}(z))\alpha^{a_1}_{\phi,u_1}(\fwdtransfoparam{\phi}{u_1}^{-a_1}(z)) J_{\fwdtransfoparam{\phi}{u_1}^{-a_1}}(z)\log \left(\frac{\tilde{\pi}(z) r(a_1| z, u_1) }{\margindensmu{\phi}{1}{z, a_1}{u_1}}\right)   \rmd z\rmd \mu_{\msu}(u_1)
    \nonumber
    \\
    &= \hspace{-10pt}\sum_{a_1 \in \{0,1\}} \int h_1(u_1)\margindensm{\phi}{0}(\tilde{z})\alpha^{a_1}_{\phi,u_1}(\tilde{z})   \log \left(\frac{\tilde{\pi}(\fwdtransfoparam{\phi}{u_1}^{a_1}(\tilde{z})) r(a_1| \fwdtransfoparam{u_1}{\phi}^{a_1}(\tilde{z}), u_1) }{\margindensmu{\phi}{1}{\fwdtransfoparam{\phi}{u_1}^{a_1}(\tilde{z}), a_1}{u_1}}\right) \rmd \tilde{z}\rmd \mu_{\msu}(u_1)
    \\
    \nonumber
    &= \hspace{-10pt}\sum_{a_1 \in \{0,1\}} \int h_1(u_1)\varphibf(y) \alpha^{a_1}_{\phi,u_1}(V_\phi(y))   \log \left(\frac{\tilde{\pi}(\fwdtransfoparam{\phi}{u_1}^{a_1}(V_\phi(y)) r(a_1| \fwdtransfoparam{\phi}{u_1}^{a_1}(V_\phi(y)), u_1) }{\margindensmu{\phi}{1}{\fwdtransfoparam{\phi}{u_1}^{a_1}(V_\phi(y)), a_1}{u_1}}\right) \rmd y\rmd \mu_{\msu}(u_1) \eqsp.
  \end{align}
  This result implies that we can integrate out everything with respect to $\varphibf$.
  % where we introduced two successive change of variables, first to integrate directly \wrt~$m^0_\phi$ our initial distribution, and then the reparameterization trick on $m^0_\phi$ to integrate out \wrt $\varphibf$, standard Gaussian in $\rset^D$.

  The intuition is the same after $K$ steps, and we can write:
  \begin{align}
    \lowerboundaux(\phi) &= \int\chunk{h}{1}{K}(\chunk{u}{1}{K})\sum_{\chunk{a}{1}{K} \in \{0,1\}^K} \margindensmu{\phi}{K}{z, \chunk{a}{1}{K}}{\chunk{u}{1}{K}} \log \left(\frac{\tilde{\pi}(z) r(\chunk{a}{1}{K}|z,\chunk{u}{1}{K}) }{\margindensmu{\phi}{K}{z, \chunk{a}{1}{K}}{\chunk{u}{1}{K}}}\right)\rmd z \rmd \mu_{\msu}^{\otimes K}(\chunk{u}{1}{K})
    \nonumber
    \\
        \nonumber
    &= \hspace{-10pt}\sum_{\chunk{a}{1}{K} \in \{0,1\}^K} \int\chunk{h}{1}{K}(\chunk{u}{1}{K}) \margindensm{\phi}{0}(\bigcirc_{j=1}^K\fwdtransfoparam{\phi}{u_j}^{-a_j}(z))J_{\bigcirc_{j=1}^K\fwdtransfoparam{\phi}{u_j}^{-a_j}}(z)
    \nonumber
    \\
    \nonumber
    &\qquad  \times\prod_{i=1}^K \alpha^{a_i}_{\phi,u_i}(\bigcirc_{j=i}^K\fwdtransfoparam{\phi}{u_j}^{-a_j}(z)) \log \left(\frac{\tilde{\pi}(z) r(\chunk{a}{1}{K}|z,\chunk{u}{1}{K}) }{\margindensmu{\phi}{K}{z, \chunk{a}{1}{K}}{\chunk{u}{1}{K}}}\right)\rmd z \rmd \mu_{\msu}^{\otimes K}(\chunk{u}{1}{K})
    \\
        \nonumber
    &= \hspace{-10pt}\sum_{\chunk{a}{1}{K} \in \{0,1\}^K} \int\chunk{h}{1}{K}(\chunk{u}{1}{K}) \margindensm{\phi}{0}(\tilde{z})
    \prod_{i=1}^K \alpha^{a_i}_{\phi,u_i}(\bigcirc_{j=i-1}^1\fwdtransfoparam{\phi}{u_j}^{a_j}(\tilde{z})) \\
    \nonumber
    & \qquad \times\log \left(\frac{\tilde{\pi}(\bigcirc_{j=K}^1\fwdtransfoparam{\phi}{u_j}^{a_j}(\tilde{z})) r(\chunk{a}{1}{K}|\bigcirc_{j=K}^1\fwdtransfoparam{\phi}{u_j}^{a_j}(\tilde{z}),\chunk{u}{1}{K}) }{\margindensmu{\phi}{K}{\bigcirc_{j=K}^1\fwdtransfoparam{\phi}{u_j}^{a_j}(\tilde{z}), \chunk{a}{1}{K}}{\chunk{u}{1}{K}}}\right)\rmd \tilde{z} \rmd \mu_{\msu}^{\otimes K}(\chunk{u}{1}{K})
    \nonumber
    \\
    &= \hspace{-10pt}\sum_{\chunk{a}{1}{K} \in \{0,1\}^K} \int\chunk{h}{1}{K}(\chunk{u}{1}{K}) \varphibf(y)
    \prod_{i=1}^K \alpha^{a_i}_{\phi,u_i}(\bigcirc_{j=i-1}^1\fwdtransfoparam{\phi}{u_j}^{a_j}(V_\phi(y))) 
    \nonumber
    \\
        \label{eq:repres_reparametrisation_trick}
    &\qquad  \times \log \left(\frac{\tilde{\pi}(\bigcirc_{j=K}^1\fwdtransfoparam{\phi}{u_j}^{a_j}(V_\phi(y))) r(\chunk{a}{1}{K}|\bigcirc_{j=K}^1\fwdtransfoparam{\phi}{u_j}^{a_j}(V_\phi(y)),\chunk{u}{1}{K}) }{\margindensmu{\phi}{K}{\bigcirc_{j=K}^1\fwdtransfoparam{\phi}{u_j}^{a_j}(V_\phi(y)), \chunk{a}{1}{K}}{\chunk{u}{1}{K}}}\right)\rmd y \rmd \mu_{\msu}^{\otimes K}(\chunk{u}{1}{K}) \eqsp.
  \end{align}

\subsection{Unbiased estimator for the gradient of the objective}
\label{SPsec:gradient_ELBO}
  %Using the reparameterization trick, we show that we could differentiate freely our lower bound.
  We now show in the following that we can also estimate without bias the gradient of the objective starting from the expression~\eqref{eq:repres_reparametrisation_trick} of the ELBO assuming that we can perform the associated reparameterization trick.

  Again, we start by the simple case $K=1$ to give the reader the main idea of the proof.
  Using for any function $f\colon \rset^D \to \rset_+^*$, $\nabla f = f \nabla \log(f)$, the gradient of our ELBO is given for any $\phi \in \msphi$ by
  \begin{align}
    \nabla \lowerboundaux(\phi)&= \sum_{a_{1}\in \{0,1\}} \int  h_1(u_1) \varphibf(y)  \alpha^{a_1}_{\phi,u_1}(V_\phi(y))\left[ \nabla\log \left(\frac{\tilde{\pi}(\fwdtransfoparam{\phi}{u_1}^{a_1}(V_\phi(y)) r(a_1| \fwdtransfoparam{\phi}{u_1}^{a_1}(V_\phi(y)), u_1) }{\margindensmu{\phi}{1}{\fwdtransfoparam{\phi}{u_1}^{a_1}(V_\phi(y)), a_1}{u_1}}\right)\right.
    \nonumber
    \\
    \nonumber
    &+ \nabla \log[\alpha^{a_1}_{\phi,u_1}(V_\phi(y))] \left. \log \left(\frac{\tilde{\pi}(\fwdtransfoparam{\phi}{u_1}^{a_1}(V_\phi(y)) r(a_1| \fwdtransfoparam{\phi}{u_1}^{a_1}(V_\phi(y)), u_1) }{\margindensmu{\phi}{1}{\fwdtransfoparam{\phi}{u_1}^{a_1}(V_\phi(y)), a_1}{u_1}}\right)\right]  \rmd y \rmd \mu_{\msu}(u_1) \eqsp.
  \end{align}
  %  where we differentiated the product of the logarithm and the acceptance function.
  Now, this form is particularly interesting, because we can access an unbiased estimator of this sum by sampling $u_1 \sim h_1$ and $y \sim \varphibf$, then $a_1 \sim \Ber\{\alpha^{1}_{\phi,u_1}(V_\phi(y))\}$ and computing the expression between brackets.

  The method goes the same way for the $K$-th step. Indeed for any $\phi \in \msphi$, we have using for any function $f\colon \rset^D \to \rset_+^*$, $\nabla f = f \nabla \log(f)$,
  \begin{align}
  \label{eq:grad_L_K}
    \nabla \lowerboundaux(\phi)&= \hspace{-10pt}\sum_{\chunk{a}{1}{K} \in \{0,1\}^K} \int\chunk{h}{1}{K}(\chunk{u}{1}{K}) \varphibf(y)
    \prod_{i=1}^K \alpha^{a_i}_{\phi,u_i}\left(\bigcirc_{j=i-1}^1\fwdtransfoparam{\phi}{u_j}^{a_j}(V_\phi(y))\right)
    \nonumber
    \\
    & \qquad \times \left[ \nabla \log \left(\frac{\tilde{\pi}(\bigcirc_{j=K}^1\fwdtransfoparam{\phi}{u_j}^{a_j}(V_\phi(y))) r(\chunk{a}{1}{K}|\bigcirc_{j=K}^1\fwdtransfoparam{\phi}{u_j}^{a_j}(V_\phi(y)),\chunk{u}{1}{K}) }{\margindensmu{\phi}{K}{\bigcirc_{j=K}^1\fwdtransfoparam{\phi}{u_j}^{a_j}(V_\phi(y)), \chunk{a}{1}{K}}{\chunk{u}{1}{K}}}\right)\right.
    \nonumber
    \\
    & \qquad + \log \left(\frac{\tilde{\pi}(\bigcirc_{j=K}^1\fwdtransfoparam{\phi}{u_j}^{a_j}(V_\phi(y))) r(\chunk{a}{1}{K}|\bigcirc_{j=K}^1\fwdtransfoparam{\phi}{u_j}^{a_j}(V_\phi(y)),\chunk{u}{1}{K}) }{\margindensmu{\phi}{K}{\bigcirc_{j=K}^1\fwdtransfoparam{\phi}{u_j}^{a_j}(V_\phi(y)), \chunk{a}{1}{K}}{\chunk{u}{1}{K}}}\right)
    \nonumber
    \\
    &\qquad  \times \left.\sum_{i=1}^K \nabla \log[\alpha^{a_i}_{\phi,u_i}\left(\bigcirc_{j=i-1}^1\fwdtransfoparam{\phi}{u_j}^{a_j}(V_\phi(y))\right)]\right]\rmd z \rmd \mu^{\otimes K}_{\msu}(\chunk{u}{1}{K}) \eqsp.
  \end{align}
  And again, this sum can be estimated by sampling $\chunk{u}{1}{K} \sim \chunk{h}{1}{K}$, $y \sim \varphibf$, then recursively $a_1 \sim \Ber(\alpha^{1}_{\phi,u_1}(V_\phi(y))$ and for $i>1$, $$a_i \sim \Ber\left(\alpha^{a_i}_{\phi,u_i}\left(\bigcirc_{j=i-1}^1\fwdtransfoparam{\phi}{u_j}^{a_j}(V_\phi(y))\right)\right) \eqsp.$$ Those variables sampled, the expression between brackets provides an unbiased estimator of the gradient of our objective. %We wrote $\chunk{u}{1}{K} \sim \mu^{\otimes K}$ by simplicity, but it is very important to our method (especially in the next part) that we can also write $\chunk{u}{1}{K} \sim \bigotimes_{i=1}^K \mu^i_\msu$.

\subsection{Extension to Hamiltonian Monte-Carlo}
\label{SPsec:HMC}
\subsubsection{Inversibility of kernel~\eqref{eq:indexed_kernel} when the transformations are involutions}
  The following proposition show a key result for applicability of HMC to our method.
  \begin{proposition}
  \label{prop:T_inv_reversibility}
    Assume that for any $(u, \phi) \in \msu\times\msphi$, $\fwdtransfoparam{\phi}{u}$ defines an involution, \ie~$\fwdtransfoparam{\phi}{u}\circ\fwdtransfoparam{\phi}{u}=\Id$. If for any $z\in\rset^D$, $$\alpha_{\phi,u}(z) \pi(z)  = \alpha_{\phi,u}(\fwdtransfoparam{\phi}{u}(z)) \pi(\fwdtransfoparam{\phi}{\tu}(z)) J_{\fwdtransfoparam{\phi}{\tu}}(z)\eqsp,$$ then the kernel defined by~\eqref{eq:indexed_kernel} and acceptance functions $\alpha$ and transformation $\fwdtransfo_\phi$ is reversible \wrt~$\pi$.
  \end{proposition}

  This result is direct consequence of \Cref{propo:2}.
  % in the case where $\msv=\varnothing $. 
  In particular, the functions $\varphi$ identified in \Cref{propo:2} are still applicable here.

\subsubsection{Application to HMC}
  An important special example which falls into the setting of~\Cref{prop:T_inv_reversibility} is the Hamiltonian Monte-Carlo
  algorithm (HMC)~\cite{duane:1987,neal:2011}. %, which has often been used in combination with VI.
  In such a case, the state variable is $z = (q, p) \in
  \rset^{2D}$, where $q$ stands for the position and $p$ the momentum. The unnormalized target distribution is defined as $\tilde{\pi}(q,p)= \tilde{\pi}(q)\varphibf(p)$ where $\varphibf$ is the density of the $D$-dimensional standard Gaussian distribution. Define the potential $U(q)= - \log(\tilde{\pi}(q))$.
  Hamiltonian dynamics propagates a particle in this extended space according to Hamilton's equation, for any $t\geq 0$,
  \begin{equation*}
  \frac{\rmd}{\rmd t}  \begin{bmatrix}q(t)\\p(t)\end{bmatrix} = \begin{bmatrix}p(t)\\-\nabla U(t)\end{bmatrix} \eqsp.
  \end{equation*}
  This dynamics preserves the extended target distribution as the flow described above is reversible, symplectic and preserves the Hamiltonian $H(q,p)$, defined as the sum of the potential energy $U(q)= -\log(\tilde{\pi}(q))$ and the kinetic energy $(1/2) p^{T}p$ (note that we can write $\tilde{\pi}(q,p)\propto \exp(-H(q,p))$), see~\cite{bou:sanz:2018}. It is not however usually possible to compute exactly the solutions of the continuous dynamics described above. However, reversibility and symplectiness can be preserved exactly by discretization using a particular scheme called the leap-frog integrator.
  Given a stepsize $\stephmc$, the elementary leap-frog $\LF_\stephmc(q_0,p_0)$ for any $(q_0,p_0)\in \rset^{2D}$ is given by $\LF_\stephmc(q_0,p_0) = (q_1,p_1)$ where 
  \begin{align*}
        p_{1/2}&=p_0 - \stephmc/2 \nabla U(q_0)\\q_1&= q_0 + \stephmc p_{1/2} \\ p_1&= p_{1/2} -\stephmc/2 \nabla U(q_1) \eqsp.
  \end{align*}
  The $N$-steps leap-frog integrator with stepsize $\stephmc>0$ is defined by $\LF_{\stephmc,N} = \bigcirc_{i=1}^N\LF_\stephmc$ is the $N$-times composition of $\LF_\stephmc$.
  % $\exp(-H(q,p))$, where $H(q,p)$ is the Hamiltonian defined as the sum of the potential energy $U(q)= -\log(\tilde{\pi}(q))$ and the kinetic energy $1/2 p^{T}M^{-1}p$ where $M$ is a positive definite mass matrix. Note that $\tilde{\pi}(q,p) = \tilde{\pi}(q)\varphibf(p)$ where $\Gamma$ is a normal distribution withe zero mean and covariance matrix $M$.
  For some parameters $a \in \ooint{0,1}$ and $\stephmc >0$, consider now the two following transformations:
  \begin{align*}
    \fwdtransfoparam{\phi}{u}^{\Refresh}(q,p)& = (q, ap+\sqrt{1-a^2}u)\,,\, u \in \msu=\rset^D \eqsp,
    \\
    \fwdtransfo_{\phi}^{\LF}(q,p) &= \LF_{\stephmc,N}(q,-p) \eqsp.
  \end{align*}
  Here, the parameter $\phi$ stands for the stepsize $\stephmc$ and the auto-regressive coefficient $a$ in the momentum refreshment transform. Other parameters could be included as well; see for example~\cite{levy2017generalizing}.
  % the leap-frog integrator parameters.
  For any $a \in \ooint{0,1}$, $\fwdtransfoparam{\phi}{u}^{\Refresh}$ is a continuously differentiable diffeomorphism. Then, taking $h=\varphibf$ and setting the acceptance ratios to $\alpha^{\Refresh}_{\phi,u} \equiv 1$, it is easily showed that $M^{\Refresh}_{\phi,h}$ defined by~\eqref{eq:def_M} -- with $T_{\phi,u}\leftarrow T^{\Refresh}_{\phi,u}$ -- is reversible \wrt~$\tilde{\pi}$.
  On the other hand, by composition and a straightforward induction, for any $\phi \in \msphi$, $\fwdtransfo_{\phi}^{\LF}$ is continuously differentiable if $\log(\pi)$ is twice continuously differentiable on $\rset^D$ and the determinant of its Jacobian is equal to $1$ on $\rset^D$. In addition, $\fwdtransfo_{\phi}^{\LF}$ is an involution by reversibility to the momentum flip operator of the Hamiltonian dynamics~\citep[Section 6]{bou:sanz:2018}. 
  Indeed, $\fwdtransfo_{\phi}^{\LF}$ is here written as the composition of $\LF_{\stephmc,N}$ and the momentum flip operator $S(q,p)= (q,-p)$, $\fwdtransfo_{\phi}^{\LF} = \LF_{\stephmc,N}\circ S$. \cite{bou:sanz:2018} show indeed that $\LF_{\stephmc,N}$ satisfies the following property $\LF_{\stephmc,N} \circ S = S \circ \LF_{\stephmc,N}^{-1}$. As $S$ is also an involution, we can write $S \circ \LF_{\stephmc,N} \circ S = \LF_{\stephmc,N}^{-1}$ and thus $\LF_{\stephmc,N} \circ S \circ \LF_{\stephmc,N} \circ S = \Id$, hence $\fwdtransfo_{\phi}^{\LF}$ is an involution.
  Thus, \Cref{prop:T_inv_reversibility} applies and the expression given for $\alpha_{\phi,u}$ is the classical acceptance ratio for HMC when $\varphi(t) = \min(1,t)$.
  HMC algorithm is obtained by alternating repeatedly these two kernels $M^{\Refresh}_{\phi,h}$ and $M^{\LF}_\phi$. To fall exactly into the framework outlined above, one might consider the extension $\msu \times \{0,1\}$, for $(q,p)\in (\rset^D)^2$, $(u,v) \in \msu \times \{0,1\}$, the transformation
  \begin{EQA}[c]
    T_\phi\bigl((p, q), (u, v)\bigr) = \left\{\begin{matrix} \fwdtransfoparam{\phi}{u}^{\Refresh}(q, p) \text{ if }v=1,
    \\
    \fwdtransfo_{\phi}^{\LF}(q,p)\text{ if }v=0,
    \end{matrix}\right.
  \end{EQA}
  and the two densities $h_\Refresh(u, v) = h(u) \Ind_{\{1\}}(v)$, $h_\LF(u, v) = f(u)\Ind_{\{0\}}(v)$ where $f$ is an arbitrary density since $T_\phi^\LF$ does not depend on $u$. The kernel defined by~\eqref{eq:def_M} associated with this transformation and density $h_\Refresh$ is $M_{\phi, h_\Refresh}= M^{\Refresh}_{\phi,h}$, and similarly  $M_{\phi, h_\LF}= M^{\LF}_{\phi,g}$. Then, the density after $K$ HMC steps with parameters $\phi$ can be written as $\marginal[\phi]{K} = \marginal[\phi]{0} (M_{\phi,h_\LF}M_{\phi, h_\Refresh})^K$.

\section{Optimization Procedure}
\subsection{Optimization in the general case}
  We saw in the previous section a way to compute an unbiased estimator of our objective. We will perform gradient ascent in the following, using the estimator provided before.

  At each step of optimization, we will sample $\chunk{u}{1}{K} \sim \chunk{h}{1}{K}$, $y \sim \varphibf$ and then sequentially $a_1 \sim \Ber(\alpha^{1}_{\phi,u_1}(V_\phi(y))$ and for $i>1$, $a_i \sim \Ber(\alpha^{a_i}_{\phi,u_i}(\bigcirc_{j=i-1}^1\fwdtransfoparam{\phi}{u_j}^{a_j}(V_\phi(y)))$. With those variables, we can compute the expression between brackets in the formula above~\eqref{eq:grad_L_K}. We then perform a stochastic gradient scheme,
  % averaging the $n$ different expressions obtained per optimization step.
  % We can repeat
  repeating the process  until convergence of our parameters. A detailed algorithm highlighting the simplicity of our method is presented in~\ref{alg:optimization}. Note that our method ``follows'' a trajectory, conditioned on noise $\chunk{u}{1}{K}$ and accept/reject Booleans $\chunk{a}{1}{K}$, which is conceptually equivalent to a flow pushforward.
  \alaini{ma version}
  \begin{algorithm}[h!]
  \label{alg:optimization}
    \caption{Optimization procedure}
    \begin{algorithmic}
      \STATE {\bfseries Input:} Transformation $T_{\phi}$, Acceptance function $\alpha_{\phi,u}$, Unnormalized target $\tilde{\pi}(.)$, Variational prior $m^0_\phi(.)$ and reparameterization trick $V_\phi$, densities on $\msu$ $\{h_i, i\in \{1,\ldots,K\}\}$ \wrt~$\mu_\msu$
      \STATE {\bfseries Input:} $T$ optimization steps, schedule $\gamma(t)$
      \STATE Initialize parameters $\phi$;
      \FOR{$t=1$ {\bfseries to} $T$}
        \STATE Sample $u_1,\ldots, u_K \text{ from } h_1, \ldots, h_K$ innovation noise;
        \STATE Sample $y \sim \mathcal{N}(0,\mathrm{I})$ starting point;
        \STATE Define current point $z_{aux} \leftarrow V_\phi(y)$;
%        \STATE Initialize chain rule $DT_{aux} \leftarrow \nabla_\phi V_\phi(y)$;
        \STATE $\overline{a} \leftarrow 1,\,S_a \leftarrow 0$ product of the $\alpha$ and sum of the log gradients respectively;
        \FOR{$k=1$ {\bfseries to} $K$}
%          \STATE Compute $\alpha_{aux} =   \alpha_{\phi,u_k}\left(z_{aux} \right)$;
          \STATE Sample $a_k \text{ from }  \Ber(\alpha_{\phi,u_k}^{1}(z_{aux}))$;
          \STATE $\alpha_{aux} =\alpha_{\phi,u_k}^{a_k}(z_{aux})$;
          \STATE Compute $da_k = \nabla_{\phi}\alpha_{aux}$;
          \STATE Update auxiliary variables:
          \STATE $\quad\overline{a} \leftarrow \overline{a} \times \alpha_{aux} \, ,\,S_a\leftarrow ,\, S_a+ da_k/\alpha_{aux}$;
 %         \STATE $\quad DT_{aux} \leftarrow DT_{aux} \times \nabla_\phi(T^{a_{k}}_{\phi,\tu_k})\left(z_{aux} \right)$;
          \STATE $\quad z_{aux} \leftarrow T^{a_{k}}_{\phi,u_k}(z_{aux})$;
        \ENDFOR
        \STATE Compute $dp = \nabla_\phi (\log(\tilde{\pi}(z_{aux}))$;
        \STATE Compute $dr = \nabla_\phi (\log(r(a_{1:K} \mid z_{aux}, \chunk{u}{1}{K})) $;
        \STATE Compute $dm =  \nabla_\phi( |V_\phi |)/|V_\phi| + \nabla_{\phi}\log(J_{\bigcirc_{j=1}^K\fwdtransfoparam{\phi}{u_j}^{-a_j}}(z_{aux})) + S_a $;
        \STATE Compute $p = \log(\tilde{\pi} (z_{aux}))$;
        \STATE Compute $m =  \log\left(\gamma(y)|V_\phi|J_{\bigcirc_{j=1}^K\fwdtransfoparam{\phi}{u_j}^{-a_j}}(z_{aux})\overline{a}\right)$;
        \STATE Compute $r = \log(r(a_{1:K} \mid z_{aux}, \chunk{u}{1}{K}))$
        \STATE Apply gradient update with $\Gamma_\phi = dp+dr-dq + (p+r-q)\times S_a $, schedule $\gamma(t)$;
      \ENDFOR
    \end{algorithmic}
    
  \end{algorithm}

\subsection{Optimization for \textit{MetFlow}}
  We give in the following a detailed algorithm for the optimization procedure for \textit{MetFlow}, highlighting the low computational complexity of our method compared to the previous attempts. Note that the increased complexity is only linear in $K$ the number of steps in our Markov Chain. The functions denoted as acceptance function are for $\tu \in \tmsu$, $v \in \{-1,1\}$: $\alpha_{\tu, v}^{1}(z) = \dmhratio(z,v)$ and $\alpha_{\tu, v}^{0}(z)=1-\dmhratio(z,v)$.
  
  Define the Rademacher distribution $\Rad(p)$ on $\{-1,1\}$ with parameter $p \in \ccint{-1,1}$ by $\Rad(p) = p\updelta_{1} + (1-p)\updelta_{-1}$. The noise $v$ for MetFlow will be sampled using a Rademacher distribution, whose parameter $p$ is let to depend on some parameters $\phi$ which will be optimized.
  \alaini{ma version}
  \begin{algorithm}[h!]
  \label{alg:optimization_MetFlow}
    \caption{Optimization procedure for MetFlow}
    \begin{algorithmic}
      \STATE {\bfseries Input:} $T$ optimization steps, schedule $\gamma(t)$
      \STATE {\bfseries Input:} Transformation $T_{\phi}$, Acceptance function $\alpha_{\tu, v}^{a}$, Unnormalized target $\tilde{\pi}(.)$, Variational prior $m^0_\phi(.)$ and reparameterization trick $V_\phi$, densities on $\tmsu$ $\{h_i, i\in \{1,\ldots,K\}\}$ \wrt~$\mu_\tmsu$, probabilities  $\{p_{\phi,i}, i\in \{1,\ldots,K\}\}$ for Rademacher distributions
      \STATE Initialize parameters $\phi$;
      \FOR{$t=1$ {\bfseries to} $T$}
        \STATE Sample $\tu_1,\ldots, \tu_K \text{ from } h_1, \ldots, h_K$;
        \STATE Sample $v_1,\ldots, v_K \text{ from } \Rad(p_{\phi,1}), \ldots, \Rad(p_{\phi,K})$;
        \STATE Sample $y\sim \mathcal{N}(0,\mathrm{I})$;
        \STATE $S_a,S_p \leftarrow 0$;
        \STATE $\overline{a} \leftarrow 1$;
        \STATE $z_{aux} \leftarrow V_\phi(y)$;
        % \STATE $DT_{aux} \leftarrow \nabla_\phi V_\phi(y)$;
        \FOR{$k=1$ {\bfseries to} $K$}
%          \STATE Compute $\alpha_{aux} =   \alpha^1_{\tu_k,v_k}\left(z_{aux} \right)$;
          \STATE Sample $a_k  \sim  \Ber(\alpha^1_{\tu_k,v_k}\left(z_{aux} \right))$;
          \STATE $\alpha_{aux} = \alpha^{a_k}_{\tu_k,v_k}\left(z_{aux} \right)$;
          \STATE Compute $da_k = \nabla_{\phi}\alpha_{aux}$;
          \STATE $S_a \leftarrow S_a+ da_k/\alpha_{aux}$;
          \STATE $S_p \leftarrow S_p+ 2^{-1}(1+v_k) \nabla_{\phi} p_{\phi,k}/p_{\phi,k} - 2^{-1}(1-v_k)\nabla_{\phi} p_{\phi,k}/(1-p_{\phi,k})$;
          \STATE $\overline{a} \leftarrow \overline{a} \times \alpha_{aux}$
%          \STATE $DT_{aux} \leftarrow DT_{aux} \times \nabla_\phi(T^{v_k a_{k}}_{\phi,\tu_k})\left(z_{aux} \right)$;
          \STATE $z_{aux} \leftarrow T^{v_k a_{k}}_{\phi,\tu_k}(z_{aux})$;
        \ENDFOR
        \STATE Compute $dp = \nabla_\phi (\log(\tilde{\pi}(z_{aux})) $;
        \STATE Compute $dm =  \nabla_\phi( |V_\phi |)/|V_\phi| + \nabla_{\phi}\log(J_{\bigcirc_{j=1}^K\fwdtransfoparam{\phi}{\tu_j}^{-v_j a_j}}(z_{aux})) + S_a $;
        \STATE Compute $dr = \nabla_{\phi}  \log(r(a_{1:K} | z_{aux}, \chunk{\tu}{1}{K},\chunk{v}{1}{K}))$
        \STATE Compute $p = \log(\tilde{\pi} (z_{aux}))$;
        \STATE Compute $m =  \log\left(\gamma(y)|V_\phi|J_{\bigcirc_{j=1}^K\fwdtransfoparam{\phi}{\tu_j}^{-v_j a_j}}(z_{aux})\overline{a}\right)$;
        \STATE Compute $r = \log(r(a_{1:K} | z_{aux}, \chunk{\tu}{1}{K},\chunk{v}{1}{K}))$;
        \STATE Apply gradient update using $\Gamma_\phi = dp+dr-dq + (p+r-q)\times (S_a+S_p) $, schedule $\gamma(t)$;
      \ENDFOR
    \end{algorithmic}
  \end{algorithm}

\section{Experiments}
  In all the sampling experiments presented in the main paper (mixture of Gaussians, \cite{pmlr-v37-rezende15} distributions) as well as for the additional experiments presented here (Neal's funnel distribution, mixture of Gaussians in higher dimensions), the flows used are real non volume preserving (R-NVP)~\cite{dinh2016density} flows.
  For $\phi \in \msphi$, one R-NVP transform $f_{\phi}$ on $\rset^D$ is defined as follows, % in the deterministic case,
  for any $z \in \rset^D$, $\tz =   f_\phi(z)$, with $\tz_{\msi} = z_{\msi}$ and $\tz_{\msi^{\complementary}} = z_{\msi^{\complementary}}\odot \exp( s_\phi (z_{\msi})) +t_\phi (z_{\msi})$,
  % \begin{equation*}
  %   f_\phi(z) = (z_{\msi}, )\eqsp,
  % \end{equation*}
  where $\odot$ is the element-wise product,  $\msi \subset \{1,\ldots,D\}$ with cardinal $\card{\msi}$, called an auto regressive mask, and $\msi^{\complementary} = \{1,\ldots,D\} \setminus \msi $,  $s_{\phi},t_\phi\colon \rset^{\card{\msi}} \to \rset^{D-\card{\msi}}$. Typically, the two functions $s_{\phi},t_\phi$ are parametrized by a neural network and that case $\phi$ represents the corresponding weights. 
  The use of this kind of functions is justified by the fact that  inverses for these flows can be easily computed and have a tractable Jacobian. Indeed, a straightforward calculation leads to for any $z \in \rset^D$, $\tz =     f_\phi^{-1}(z)$ with $\tz_{\msi} = z_{\msi}$ and $\tz_{\msi^{\complementary}} = (z_{\msi^{\complementary}}-t_\phi (z_{\msi}))\odot \exp(-s_\phi (z_{\msi}))$ and $J_{f_\phi}(z)=  \exp{ \{\sum_{j= 1}^{D-\card{\msi}}  (s_\phi(z_{\msi}))_j \}}$.

  In our experiments, we consider a generalization of this setting which we refer to as latent noisy NVP (LN-NVP) flows defined for $(\phi,u) \in \msphi \times \msu$, $f_{\phi,u}\colon \rset^D \to \rset^D$, of the form for any $z\in \rset^D$, $\tz = f_{\phi,u}(z)$ with $\tz_{\msi} = z_{\msi}$, $\tz_{\msi^{\complementary}} = z_{\msi^{\complementary}}\odot \exp( s_\phi (z_{\msi},u)) +t_\phi (z_{\msi},u)$, where $\msi$ is an auto-regressive mask and $s_\phi,t_\phi\colon \rset^{\card{\msi}} \times \msu \to \rset^{D-\card{\msi}}$. All the results on R-NVP apply to our LN-NVP, in particular for any $z \in \rset^D$, $\tz =     f_{\phi,u}^{-1}(z)$ with $\tz_{\msi} = z_{\msi}$ and $\tz_{\msi^{\complementary}} = (z_{\msi^{\complementary}}-t_\phi (z_{\msi},u))\odot \exp(-s_\phi (z_{\msi},u))$ and $J_{f_{\phi,u}}(z) = \exp{ \{\sum_{j= 1}^{D-\card{\msi}}  (s_\phi(z_{\msi},u))_j \}}$.

% , in the deterministic case for example:
% % \begin{equation*}
% %     f_\phi^{-1}(z_{1:D}) = (z_{\msi}, )\eqsp.
% % \end{equation*}
% and $J_{f_\phi}(z_{1:D})=  \exp{ \sum_j (s_\phi(z_{\msi}))_j}$ the exponential of the sum of the elements in $s_\phi(z_{\msi})$. This justifies our use of R-NVP transforms.

% As emphasized in~\cite{dinh2016density}, for any $\phi \in \msphi$ and $z \in \rset^D$, $J_{f_{\phi}}(z) = 1$.  
% When we are sampling in the pseudo deterministic case, a R-NVP transform is defined, for $z_{1:D} \in \rset^D$, $u \in \msu$, as 
% \begin{equation*}
%     f_\phi(z_{1:D}, u) = (z_{\msi}, z_{\msi^{\complementary}}\odot \exp( s_\phi (z_{\msi},u)) +t_\phi (z_{\msi},u)) \eqsp.
% \end{equation*}
% Here, $\msi$ represents an auto regressive mask, and $\msi^{\complementary}$ its complementary in $\{1,\dots,D\}$, as presented in~\cite{dinh2016density}, and $\odot$ the element-wise product.

\subsection{Mixture of Gaussians: Additional Results}
  We use MetFlow with the pseudo random setting. Recall that by this, we mean that we define a unique function $(z,u) \to T_\phi(z, u)$ and before training, we sample innovation noise $u_1, \dots, u_K$ that will be ``fixed'' during all the training process.  We then perform optimization on the ``fixed'' flows $\trueflow_{\phi,i} \leftarrow \fwdtransfo_\phi(\cdot, u_i)\,,\, i \in \{1,\dots, K\}$. The specific setting we consider first is as follows. We compare the sampling based on variational inference with R-NVP Flow and our methodology MetFlow with LN-NVP. In both case, each elementary transformation we consider is the composition of 6 NVP transforms. Here $\msu = \rset^D$ with $D=2$ and the target distribution is the one described in \Cref{sec:mixture-gaussians_0}. 
  Each function $t$ and $s$ in the NVP flows is a neural network with one fully connected hidden layers of size $4$ for R-NVP and LN-NVP, % twice the input dimension (here, we call "input dimension" the latent dimension when there is no noise input, and twice the latent dimension when there is noise input - noise has the size of the latent space)
  and LeakyRelu (0.01) activation functions, and final layers with activation functions $\tanh$ for $s$ and identity for $t$. Each method is trained for 25000 iterations, with ADAM~\cite{kingma2014adam}, with learning rate of $10^{-3}$ and betas = $(0.9,0.999)$, and an early stopping criterion of 250 iterations (If within 250 iterations, the objective did not improve, we stop training). The prior distribution is a standard normal Gaussian. 
  \Cref{fig:5_flows_ours_png} and \Cref{fig:5_flows_rnvp_png} show the results and the effect of each trained flow. The first pictures are gradient-coloured along the $x$-axis, with the colour of each point corresponding to its position in the previous image. This helps us to understand how well the processes mixes and transforms an original distribution.
  \begin{figure*}
    \includegraphics[width=\linewidth]{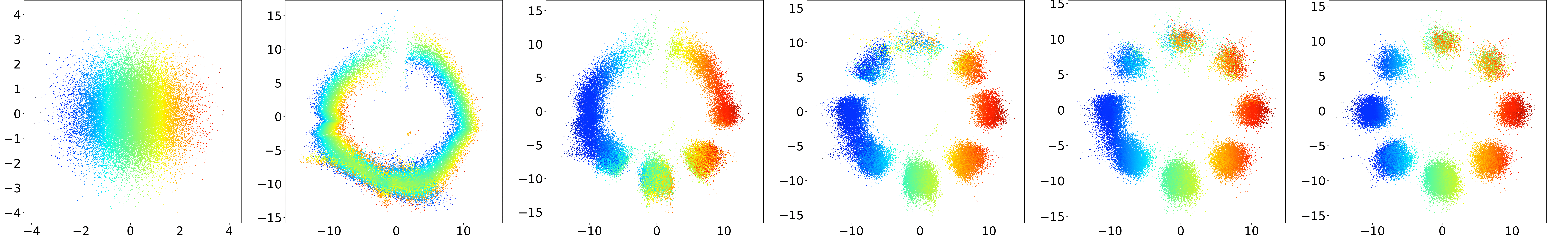}
    \caption{Consecutive outputs of each MetFlow kernel. Left: prior normal distribution, then successive effect of the 5 trained MetFlow kernels.}
  \label{fig:5_flows_ours_png}
  \end{figure*}
  \begin{figure*}
    \includegraphics[width=\linewidth]{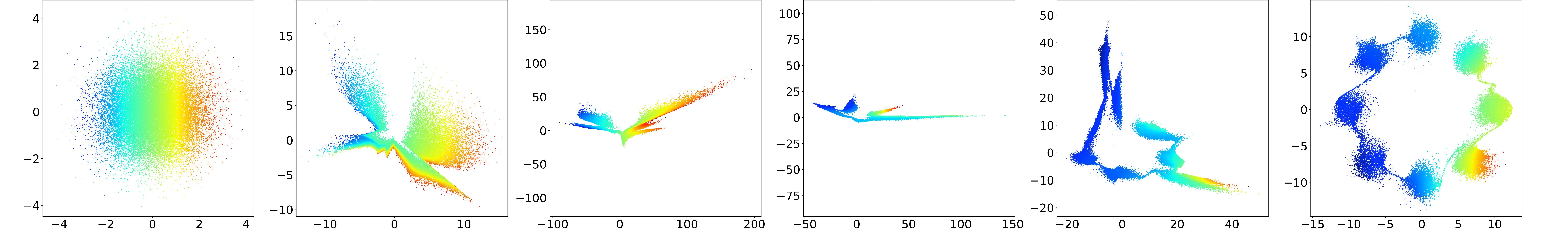}
    \caption{Consecutive outputs of each block of R-NVP. Left: prior normal distribution, then successive effect of the 5 trained R-NVP blocks - 6 transforms each.}
  \label{fig:5_flows_rnvp_png}
  \end{figure*}
  It is interesting to note that our method produces flows that are all relevant and all help the distribution to improve at each step. On the contrary, classical R-NVP and flows in general only take interest in what happens at the very last ($K$-th) step: the pushforward distribution can take any form in between. However, this representation is only interesting when each of the transformations have different sets of parameters, increasing a lot the total numbers of parameters to tune. The evolution of our method can be interpreted as well with the presence of acceptance ratios, that at each step ``filter'' or ``tutor'' the distribution, helping it not to get too far from the target. This allows us as well to introduce the setting where we can iterate our flows to refine more and more the approximation we produce in the end.
  We also show that our method is robust to a change in the prior (even a violent one). In the following figure, we change the standard normal prior to a mixture of two well-separated Gaussian distributions, with a standard deviation of 1 and means $(-50,0)$ and $(50,0)$. \Cref{fig:prior_subs} shows that MetFlow still remains efficient after iterating only a few times the learnt kernel, without retraining it at any point. It obviously does not work for R-NVP. Recall that by "iterating a kernel", we mean (as is explained in the main paper) that additional noise $\{u_i\}_{i\geq K+1}$ is sampled to define other transformations $\trueflow_{\phi,i} \leftarrow\fwdtransfo_\phi(\cdot, u_i)$ and thus additional MetFlow kernels. This particular property could find applications with time-changing data domains. Indeed, it does not require retraining existing models, while remaining an efficient way to sample from a given distribution at low computational cost.
 % \begin{figure*}
 %   \includegraphics[width=\linewidth]{pics/old/repeated_100_5_unique_ours_target_subst-1.png}
 %   \caption{Prior substitution for MetFlow}
 %   \label{fig:our_prior_subst}
 % \end{figure*}
 % \begin{figure*}
 %   \includegraphics[width=\linewidth]{pics/old/repeated_10_5_unique_rnvp_target_subst-1.png}
 %   \caption{Prior substitution for R-NVP}
 %   \label{fig:rnvp_prior_subst}
 % \end{figure*}
  \begin{figure}[h!]
    \begin{tabular}{ccccc}
      \includegraphics[width=0.23\linewidth]{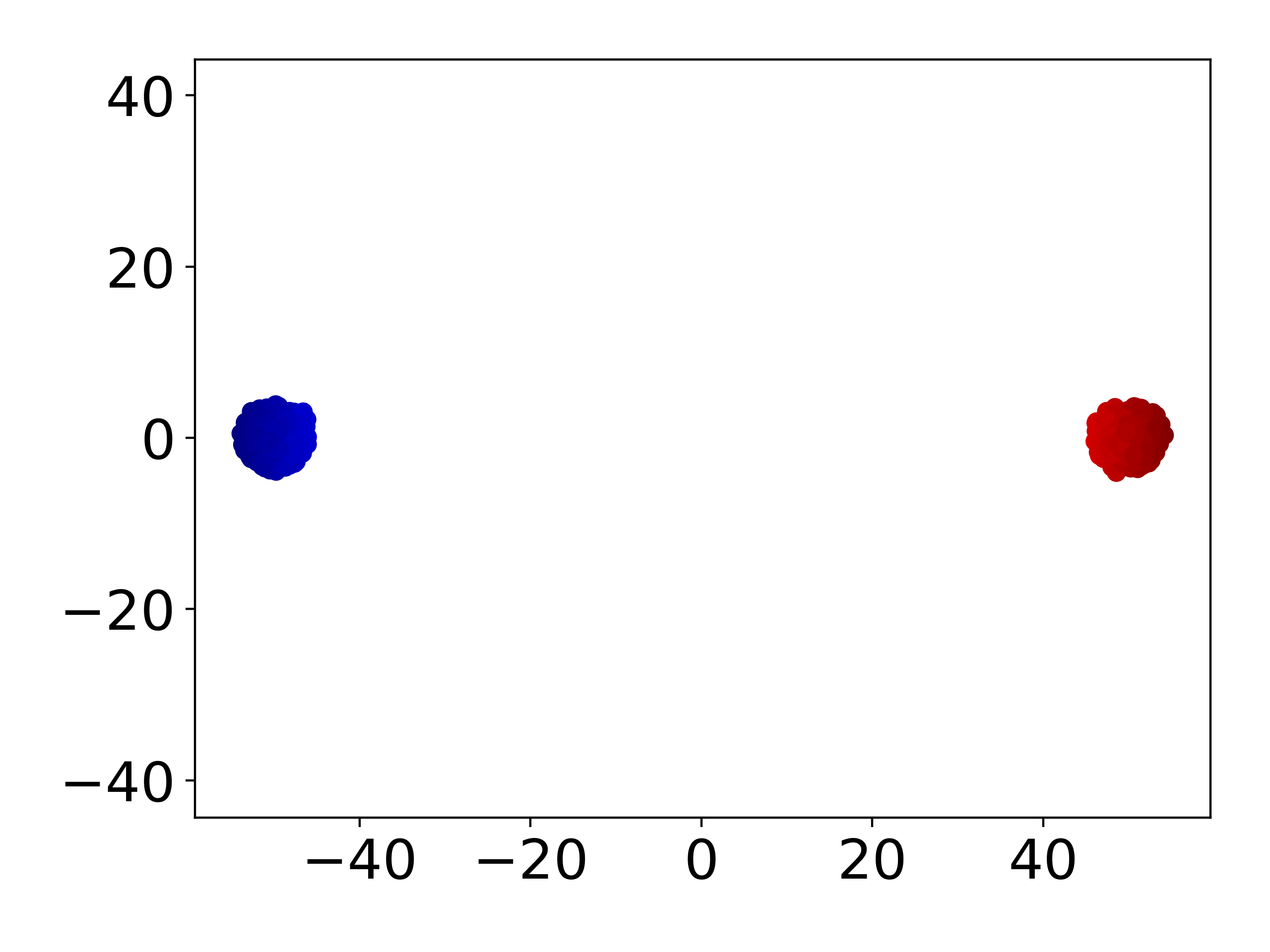} & \includegraphics[width=0.23\linewidth]{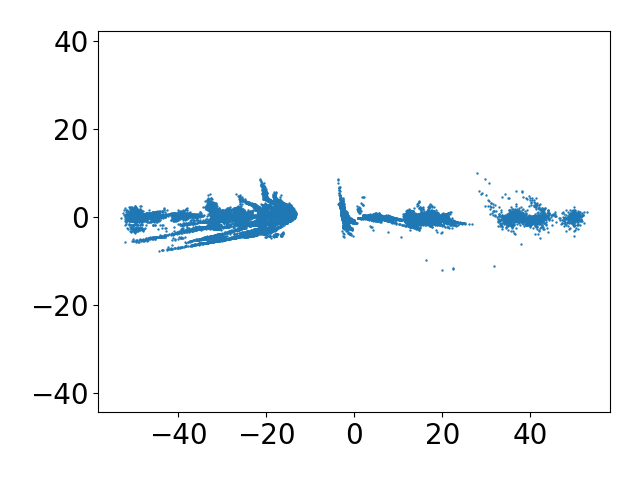} & \includegraphics[width=0.23\linewidth]{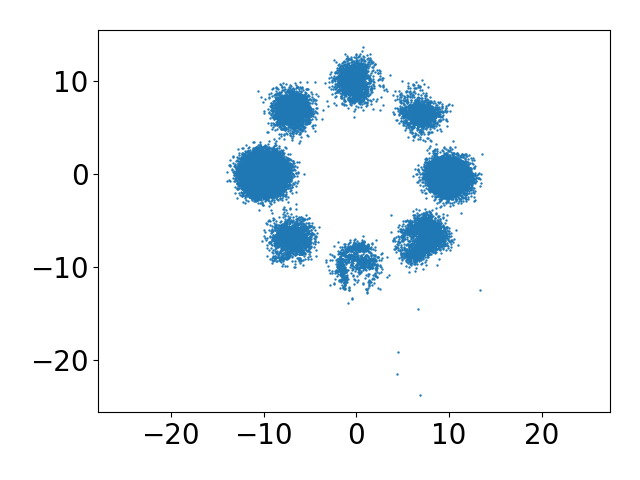}& \includegraphics[width=0.23\linewidth]{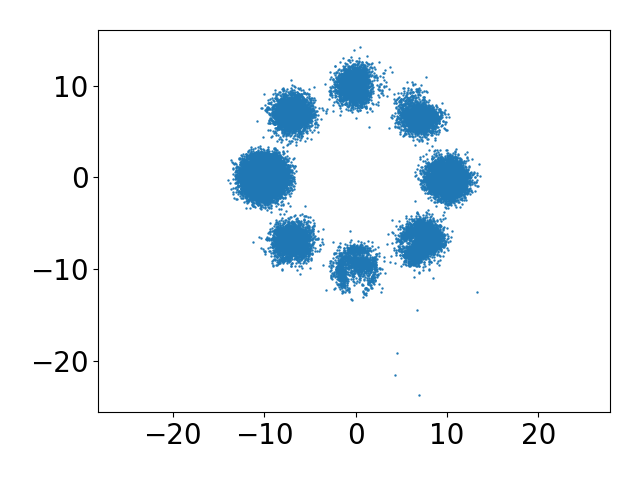}\\
      \includegraphics[width=0.23\linewidth]{pics/supp_gaussian_prior_subst/prior_subst.png} & \includegraphics[width=0.23\linewidth]{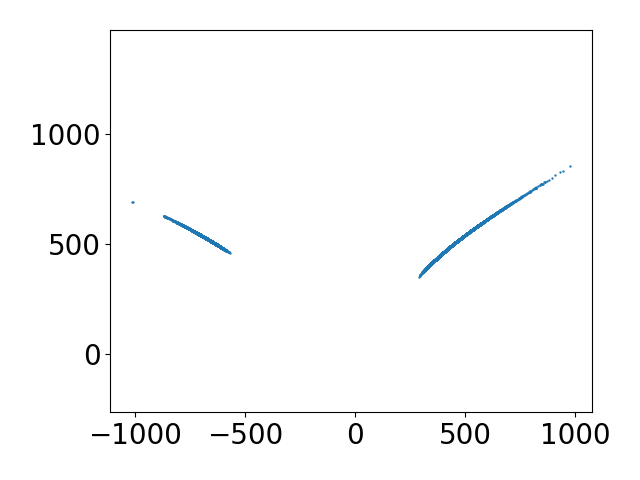} 
    \end{tabular}
    \caption{Changing the prior to a mixture of two separated Gaussians, having trained the method on a standard normal prior. Top row, from left to right: Subsituted prior, 5 trained MetFlow kernels, re-iteration of 100 MetFlow kernels, 200 MetFlow kernels. Bottom row: Substituted prior, R-NVP flow.}
  \label{fig:prior_subs}
  \end{figure}  

  In addition to the 2-dimensional results presented in the paper, we show here the results of our method for mixture of Gaussians of higher dimensions. Figure~\ref{fig:mode_capture} shows the number of Gaussians retrieved by our model when the target distribution is a mixture of 8 Gaussians with variance $1$, located at the corners of the $d$-dimensional hypercube, for different values of $d$. We see that our method significantly outperforms others, be they state-of-the-art Normalizing Flows (NAF) or MCMC methods (NUTS). Furthermore, our method scales more efficiently with dimension, as it is still able to retrieve two of the modes in dim 100 (for some runs), when other methods collapse to one mode for $d \geq 20$. MetFlow here is used in the pseudo random setting, with 7 LN-NVP blocks of two elementary transforms, $t$ and $s$ being as previously two one-layer fully connected neural networks, input dimension of twice the dimension considered $2D$ (again, we take $\msu=\rset^D$) and hidden dimension of $2D$. The activations functions stay the same. NUTS is ran with 1500 warm-up steps and a length of 10 000 samples. 
  \begin{figure}
    \includegraphics[width=\linewidth]{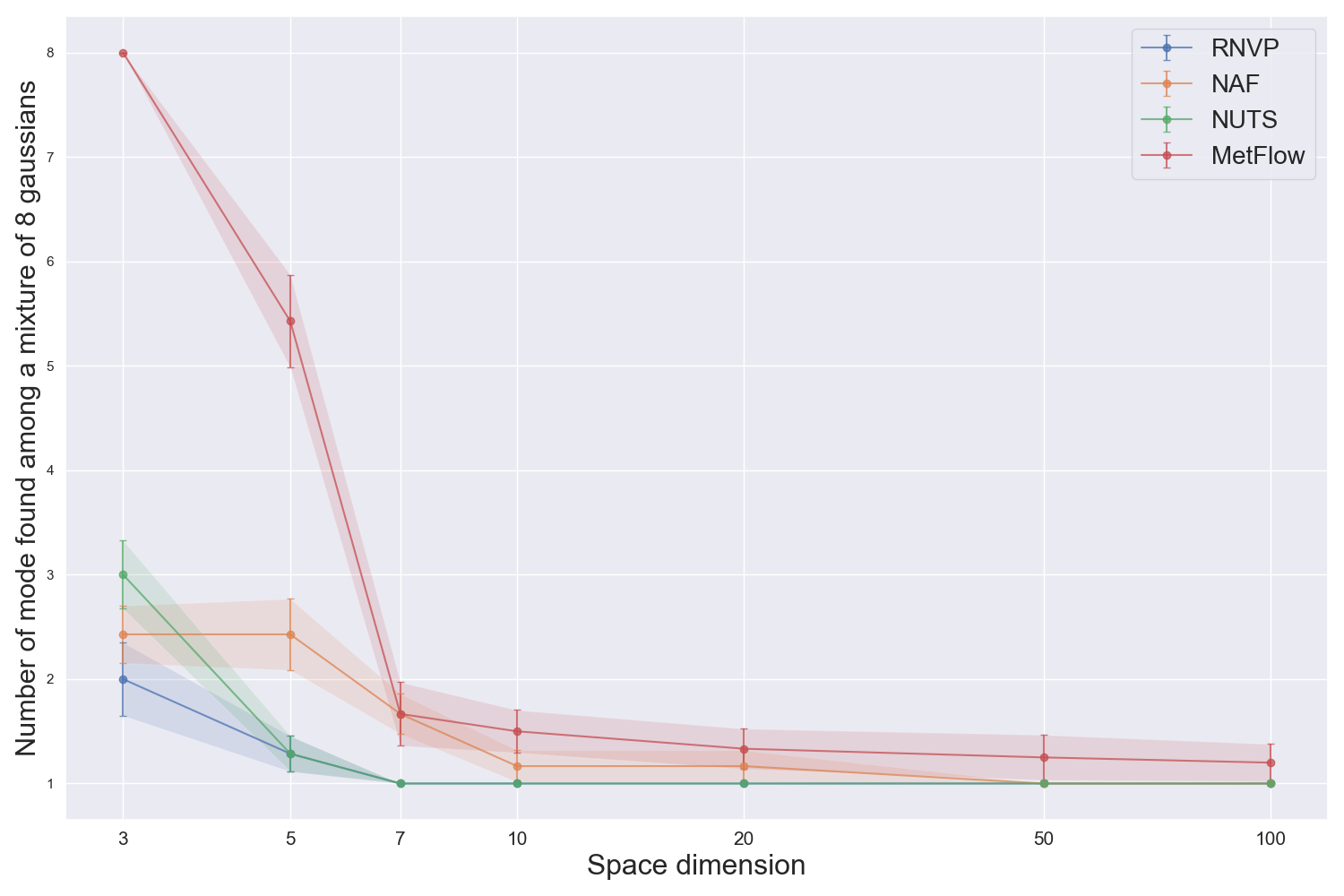}
    \caption{Number of modes retrieved by different methods. The target distribution is a mixture of 8 isotropic Gaussian distributions of variance 1 located at the corners of a $d$-dimensional hypercube. The methods are trained to convergence or retrieval of all modes, and mode retrieval is computed by counting the number of samples in a ball of radius $4d$ around the center of a mode. Error bars represent the standard deviation of the mean number of modes retrieved for different runs of the method (different initialization and random seed).}
  \label{fig:mode_capture}
  \end{figure}

\subsection{Funnel distribution}
  We now test our approach on an other hard target distribution for MCMC proposed in~\cite{neal2003}.
  The target distribution has density Lebesgue density \wrt~the Lebesgue measure given for any $z \in \rset^2$ by
  \begin{equation}
    \label{eq:def_funnel_distribution}
    \pi(z) \propto \exp(z_1^2/2) \exp(z_2\rme^{z_1}/2) \eqsp. 
  \end{equation}
  We are using exactly the same setting for MetFlow with LN-NVP and R-NVP as for the mixture of Gaussian. We train the flows with the same optimizer, for 25000 iterations again. Figure~\ref{fig:funnel_all} show the results of MetFlow and R-NVP flows. Again, as we are in the pseudo random setting, we  sample 95 additional innovation noise $(u_i)_{i \in \{6,\ldots,100\}}$ and show the distribution produced by 100 MetFlow kernels, having only trained the first five transformations $\{\trueflow_{\phi,i}\,,\, i\in\{1,\dots,5\}\}$, as described in \Cref{sec:experiments} of the main document.
  %The results plotted on Figure~\ref{fig:funnel_all} show the marginal distribution after the composition of the five trained kernels again in the pseudo random setting, and after iterating those five trained kernels 50, 100, 150, 200 and 250 times, by resampling 50, 100, 150, 200, 250 more innovation noise, as explained in the previous section. The figures are colored using the same setting.
  We can observe that after only five steps, the distribution has been pushed toward the end of the funnel. However, the amplitude is not recovered fully. It can be interpreted in light of the ``tutoring'' analogy we used previously. As the Accept/Reject control at each step the evolution of the points, if the number of steps is too small, the proposals do not got to the far end of the distribution. However, the plots show that the proposal given by MetFlow are still learnt relevantly, as Figure~\ref{fig:funnel_all} illustrates that iterating only a few more MetFlow kernels matches the target distribution in all its amplitude. %Moreover, as shown by the shuffling of the colors, MetFlow mixes as well in this situation.

  \begin{figure*}
    \includegraphics[width=\linewidth]{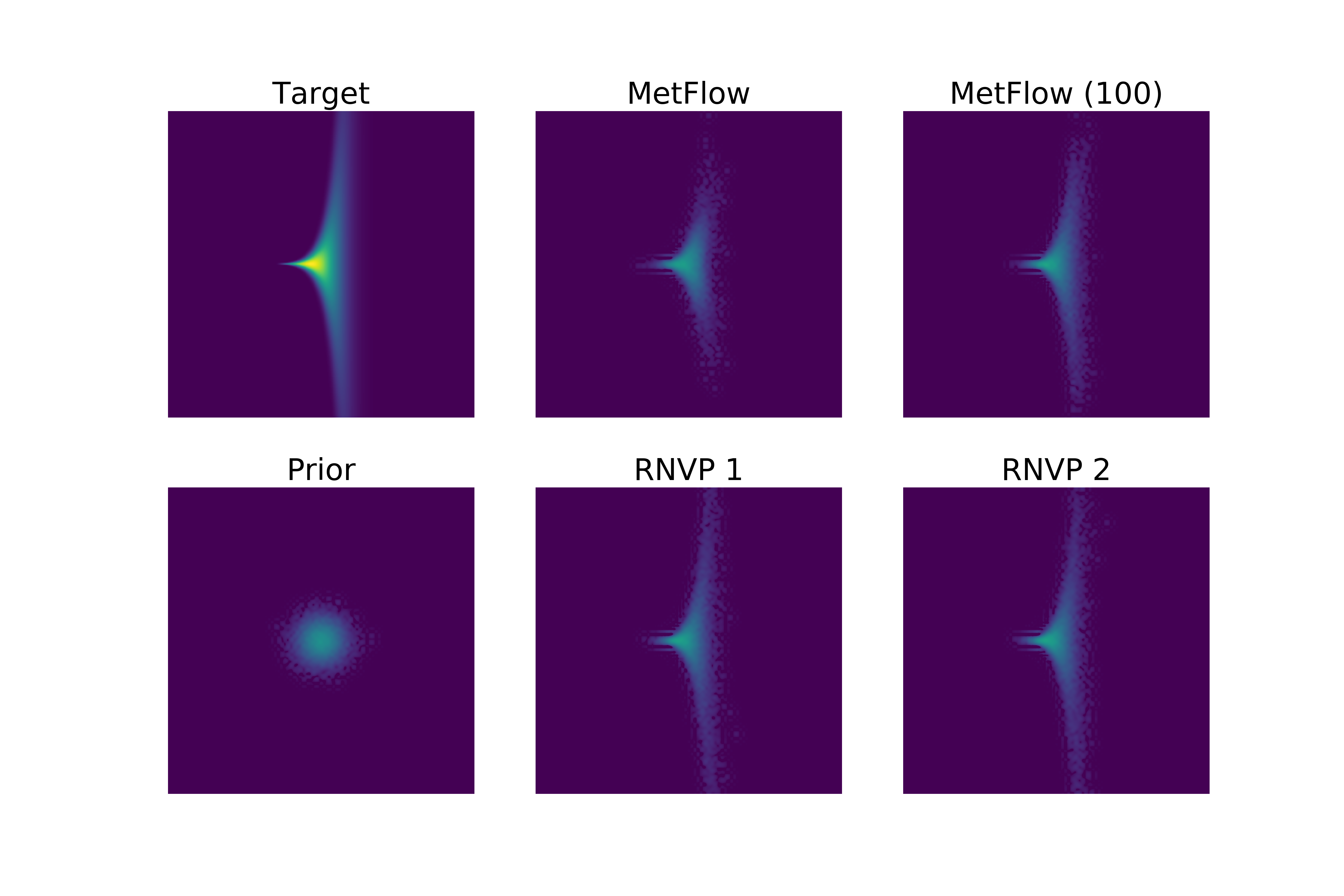}
    \caption{Density matching for funnel. Top row: Target distribution, MetFlow with 5 trained kernels, MetFlow with 5 trained kernel iterated 100 times. Bottom row: Prior distribution, First run of 5 R-NVP, second run of 5 R-NVP}
  \label{fig:funnel_all}
  \end{figure*}
%   \begin{figure*}[h!]
%     \includegraphics[width=\linewidth]{pics/repeated_251_5_unique_ours_funnel-1.png}
%     \caption{Funnel distribution, Our method}
%   \label{fig:repeated_251_5_unique_ours_funnel.pdf}
  % \end{figure*}

\subsection{Real-world inference - MNIST}
  In the experiments ran on MNIST described in the main paper, we fix a decoder $p_\theta$. To obtain our model, we use the method described in~\cite{kingma2016improved}. The mean-field approximation output by the encoder is here pushforward by a flow. In~\cite{kingma2016improved}, the flows used were Inverse Autoregressive Flows, introduced in the same work. We choose here to use a more flexible class of flows neural auto-regressive flows (NAF), introduced in~\cite{huang2018neural}, which show better results in terms of log-likelihood. 
  In practice, we use convolutional networks for our encoder and decoder, matching the architecture described in~\cite{salimans2015markov}. The inference network (encoder) consists of three convolutional layers, each with filters of size $5\times5$ and a stride of 2, and output 16, 32, and 32 feature maps, respectively. The output of the third layer feeds a fully connected layer with hidden dimension 450, which then is fully connected to the outputs, means and standard deviations, of the size of our latent dimension, here 64. Softplus activation functions are used everywhere except before outputting the means. For the decoder, a similar but reversed architecture is used, using upsampling instead of stride, again as described in~\cite{salimans2015markov}. The Neural Autoregressive Flows are given by pyro library. NAF have a hidden layer of 64 units, with an AutoRegressiveNN which is a deep sigmoidal flow~\cite{huang2018neural}, with input dimension 64 and hidden dimension 128. 
  Our data is the classical stochastic binarization of MNIST~\cite{salakhutdinov:murray:2008}. We train our model using Adam optimizer for 2000 epoches, using early stopping if there is no improvement after 100 epoches. The learning rate used is $10^{-3}$, and betas $(0.9,0.999)$. This produces a complex and expressive model for both our decoder and variational approximation.

  We show first using mixture experiments that MetFlow can overcome state-of-the-art sampling methods. The mixture experiment described in the main paper goes as follows. We fix $L$ different samples, and wish to approximate the complex posterior $p_\theta(\cdot|(x_i)_{i=1}^L))\propto p(z)\prod_{i=1}^L p_\theta(x_i|z)$. We give two approximations of this distribution, given by a state-of-the-art method, and MetFlow. The state-of-the-art method is a NAF a hidden layer of 16 units, with an AutoRegressiveNN which is a deep sigmoidal flow~\cite{huang2018neural}, with input dimension 64 and hidden dimension 128. MetFlow is trained here in the deterministic setting, with 5 blocks of 2 R-NVPs, where again each function $t$ and $s$ is a neural network with one fully connected hidden layers of size 128 and LeakyRelu (0.01) activation functions, and final layers with activation functions $\tanh$ for $s$ and identity for $t$. MetFlow and NAF are optimized using 10000 batches of size 250 and early stopping tolerance of 250, with ADAM, with learning rate of $10^{-3}$ and betas = $(0.9,0.999)$. We use Barker ratios as well here. The prior in both cases is a standard $64$-dimensional Gaussian. 
  \begin{figure}[h!]
    \centering
    \includegraphics[width=0.4\linewidth]{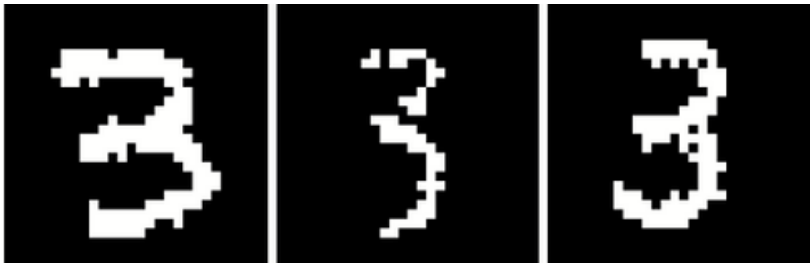}
    \label{fig:fixed_digits}
    \caption{Fixed digits for mixture experiment.}
  \end{figure}
  \begin{figure}[h!]
    \centering
    \begin{minipage}{0.45\textwidth}
      \centering
      \includegraphics[width=\textwidth]{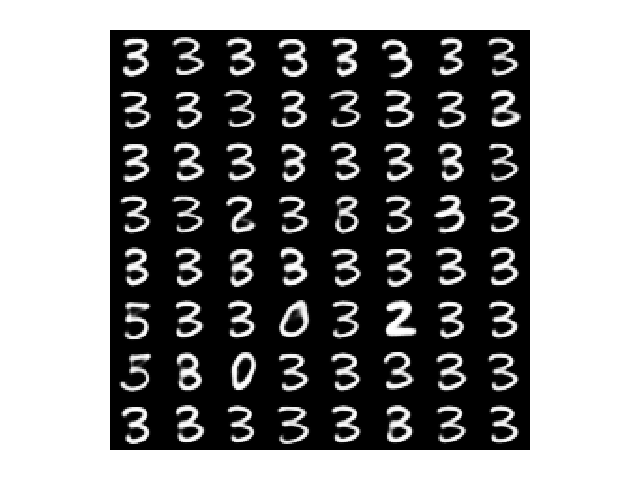}
      \caption{Mixture of 3, MetFlow approximation.}
    \label{fig:mix_metFlow}
    \end{minipage}
    \begin{minipage}{0.45\textwidth}
      \centering
      \includegraphics[width=\textwidth]{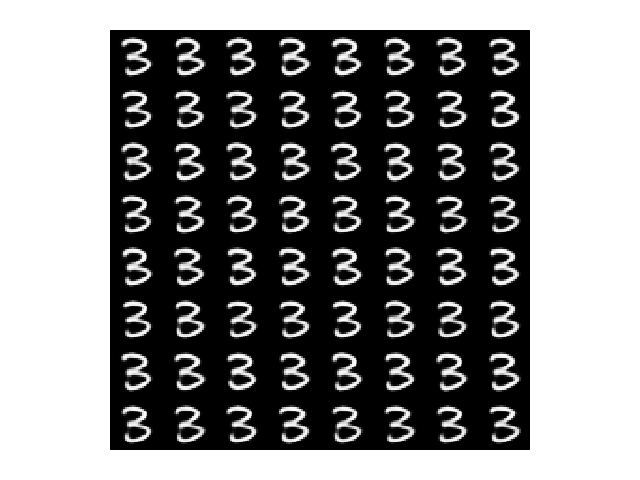}
      \caption{Mixture of 3, NAF approximation.}
    \label{fig:mix_naf}
    \end{minipage}
  \end{figure}

  We see on \Cref{fig:mix_naf} that if NAF ``collapses'' to one fixed digit (thus one specific mode of the posterior), MetFlow \Cref{fig:mix_metFlow} is able to find diversity and multimodality in the posterior, leading even to ``wrong'' digits sometimes, showing that it truly explores the complicated latent space.

  Moreover, we can compare the variational approximations computed for our VAE (encoder - mean field approximation - and encoder and NAF) to MetFlow approximation. MetFlow used here are the composition of 5 blocks of 2 R-NVPs (deterministic setting), where each function $t$ and $s$ is a neural network with one fully connected hidden layers of size 128 and LeakyRelu (0.01) activation functions, and final layers with activation functions $\tanh$ for $s$ and identity for $t$. MetFlow is optimized using $150$ epoches of $192$ batches of size $250$ over the dataset MNIST, with ADAM, with learning rate of $10^{-4}$ and betas = $(0.9,0.999)$ and early stopping with tolerance of 25 (if within 25 epoches the objective did not improve, we stop training). The optimization goes as follows. We fix encoder (mean-field approximation) and decoder (target distribution). For each sample $x$ in a minibatch of the dataset, we optimize MetFlow starting from the prior given by the encoder (mean-field approximation) and targetting posterior $p_\theta(\cdot|x)$. Note that during  all this optimization procedure, the initial distribution of MetFlow is fixed to be the encoder, which corresponds to freeze the corresponding parameters.  This is a simple generalization of our method to amortized inference. In the following, we use Barker ratios. The additional results are given by \Cref{fig:Gibbs1,fig:Gibbs3,fig:Gibbs4,fig:Gibbs5,fig:Gibbs6}.
  \begin{figure*}
    \includegraphics[width=\linewidth]{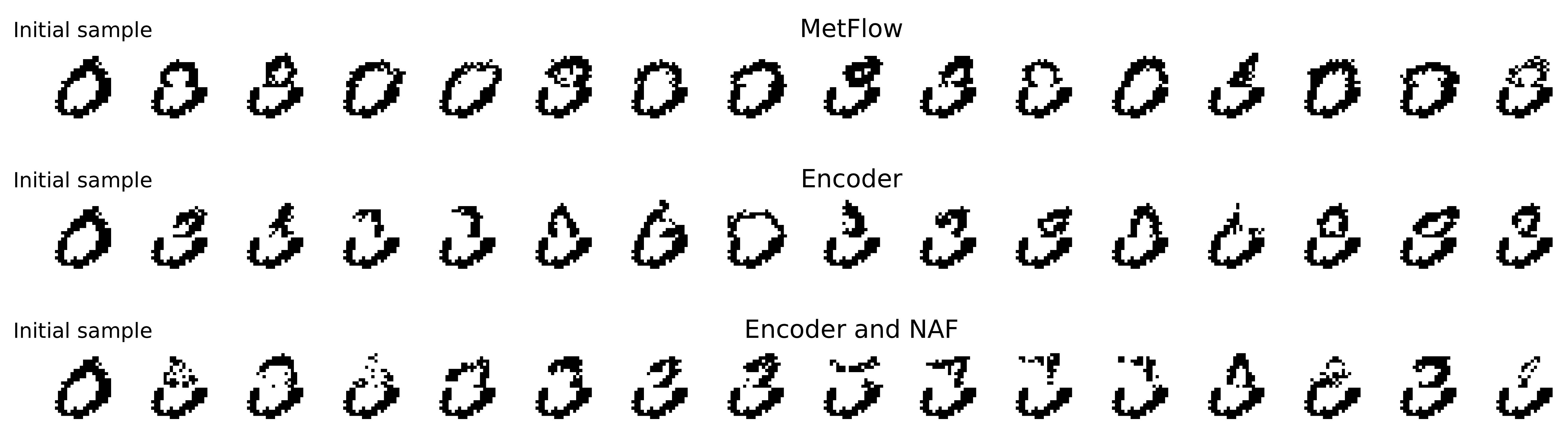}
    \caption{Gibbs inpainting experiments starting from digit 0.}
    \label{fig:Gibbs1}
  \end{figure*}
   \begin{figure*}
    \includegraphics[width=\linewidth]{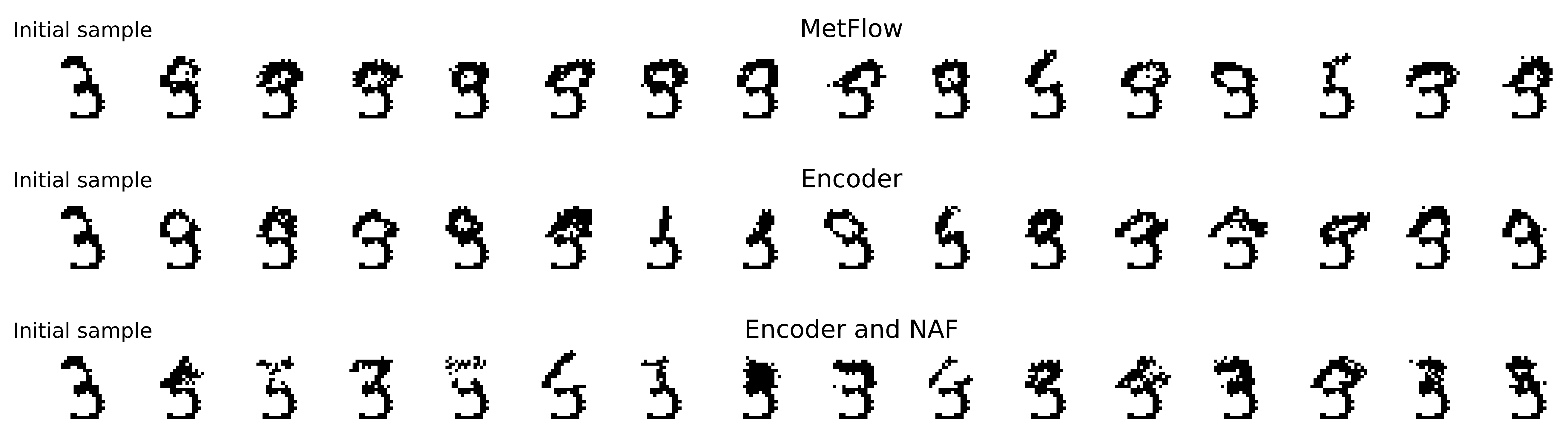}
    \caption{Gibbs inpainting experiments starting from digit 3.}
    \label{fig:Gibbs3}
  \end{figure*}
  \begin{figure*}
    \includegraphics[width=\linewidth]{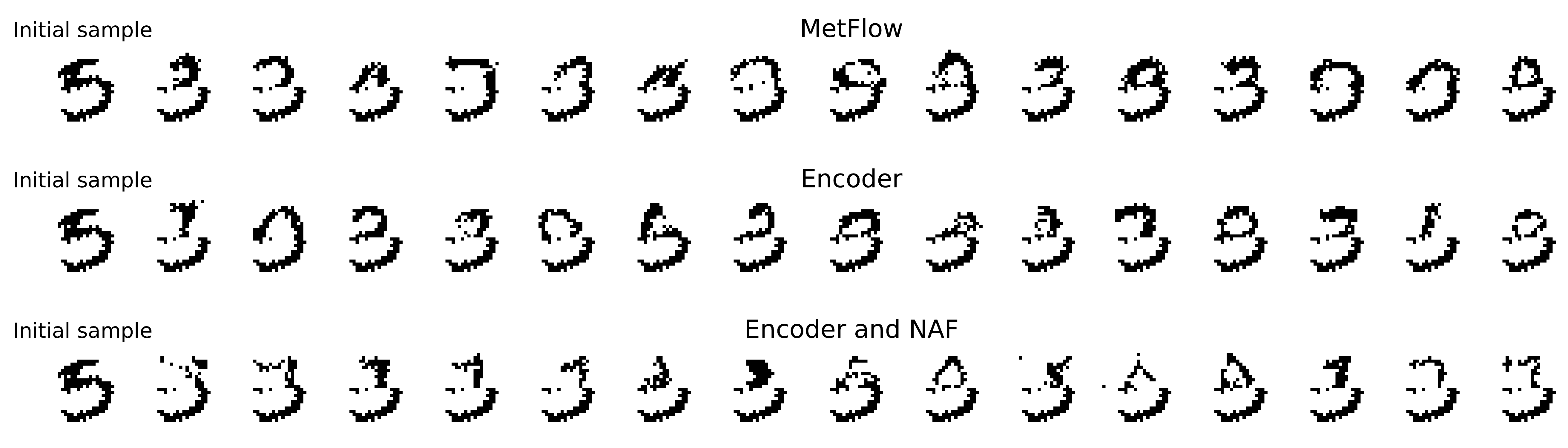}
    \caption{Gibbs inpainting experiments starting from digit 9.}
    \label{fig:Gibbs4}
  \end{figure*}
  \begin{figure*}
    \includegraphics[width=\linewidth]{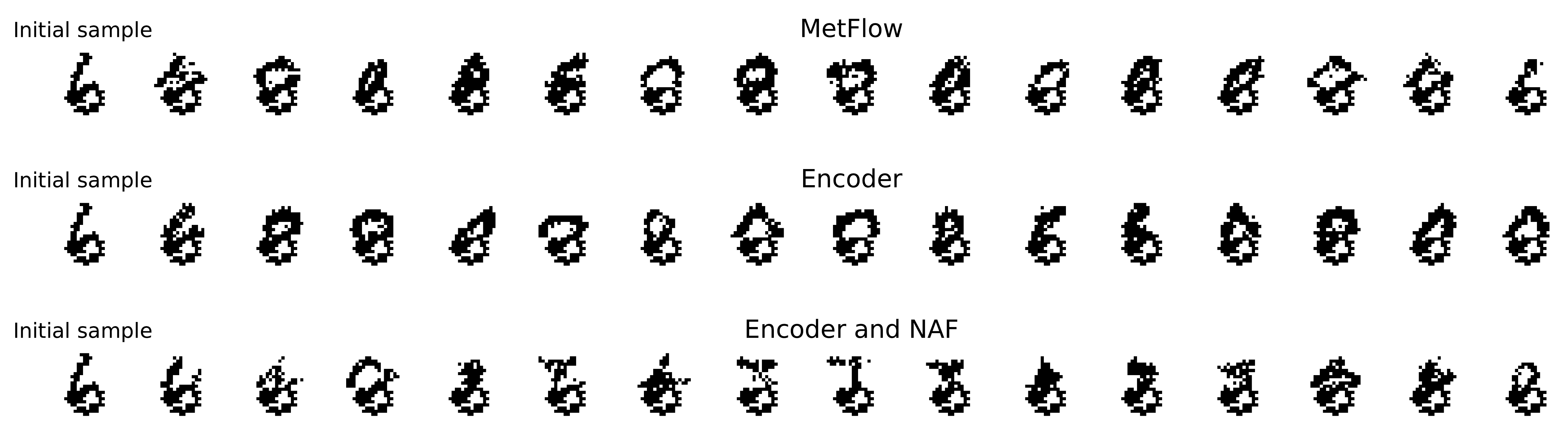}
    \caption{Gibbs inpainting experiments starting from digit 6.}
    \label{fig:Gibbs5}
  \end{figure*}
   \begin{figure*}
    \includegraphics[width=\linewidth]{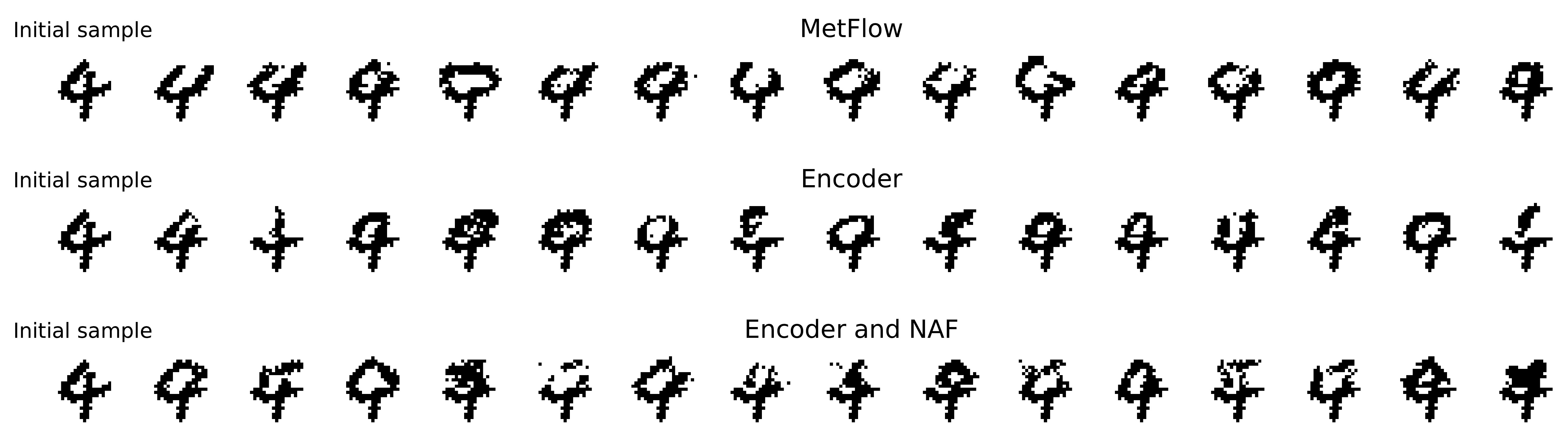}
    \caption{Gibbs inpainting experiments starting from digit 4.}
    \label{fig:Gibbs6}
  \end{figure*}

  The encoder represents just the mean-field approximation here, while encoder and NAF represent the total variational distribution learnt by the VAE described above.

\subsection{Additional setting of experiments}
  So far, we have described two settings. The first one, that we have called deterministic, in which the transformations take no input ``innovation noise'', but all have different sets of parameters.  The second one, pseudo random, defines one global transformation $\fwdtransfo_\phi$ with a unique set of parameters $\phi$, considers $K$ initially sampled at random ``innovation noise'' $\tu_1,\dots, \tu_K$, and considers densities $\chunk{h}{1}{K}$ such that the noise resampled at every step of the optimization is ``fixed''. It then learns the parameters for the global transformation using the considered transformations $\trueflow_{\phi,i} \leftarrow\fwdtransfo_\phi(\cdot, u_i)\,,\,i\in \{1,\dots,K\}$.
  We can consider here a last setting, \emph{fully random}, on which a global transformation $\fwdtransfo_\phi$ with a unique set of parameters $\phi$ is considered. However, now, we consider a unique density $h$ (typically Gaussian) to sample the ``innovation noise'' at every step of the optimization - note that this is still covered by Algorithm~\ref{alg:optimization_MetFlow}. This allow us to consider properly random transformations, and looks more like the classical framework of MCMC. 
  Even if this introduces more noise in our stochastic gradient descent, it encourages MetFlow kernels trained to incorporate a lot more diversity. This can be seen with the two experiments described before, for mixture of different digits in MNIST, or the inpainting experiments. Even though experiments can show more diversity, it is important to note that they typically require a longer number of epoches to reach convergence.
  We give in \Cref{fig:comparison_setting}~a comparison of the three settings for a mixture of digits problem. As we can see, the diversity introduced by our method is highest in the fully random setting. We can see a wider variety of 3, even though other digits tend to appear more when we decode them as well. 
  \begin{figure}[h!]
    \centering
    \begin{minipage}{0.33\textwidth}
      \centering
      \includegraphics[width=\textwidth]{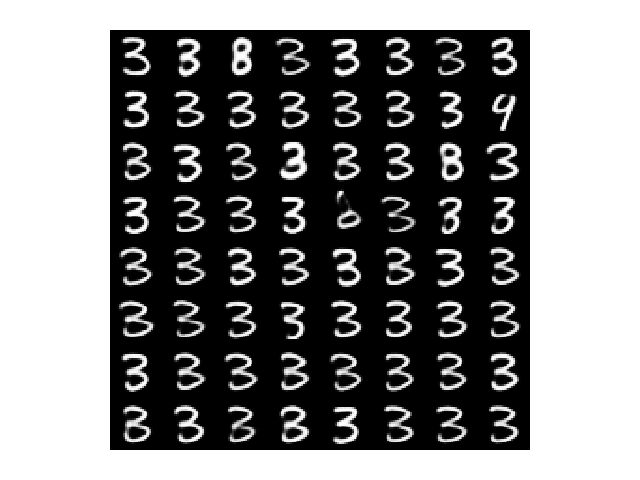}
    \end{minipage}
    \begin{minipage}{0.33\textwidth}
      \centering
      \includegraphics[width=\textwidth]{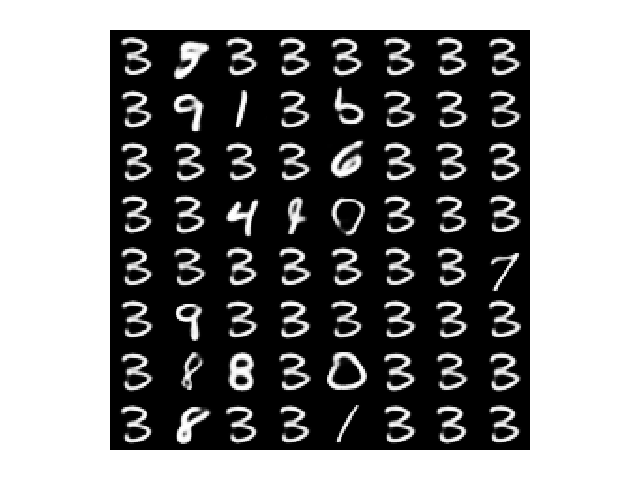}
    \end{minipage}
    \begin{minipage}{0.33\textwidth}
      \centering
      \includegraphics[width=\textwidth]{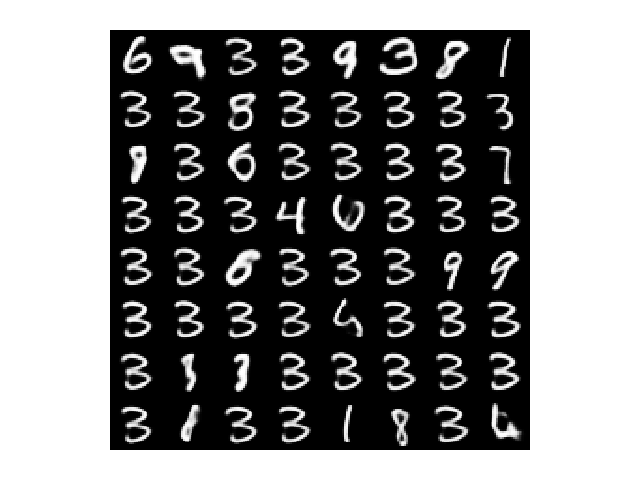}
    \end{minipage}
    \label{fig:comparison_setting}
    \caption{Comparison of the different settings described for a mixture of digits experiment. From left to right, deterministic setting, pseudo-random setting, fully random setting.}
  \end{figure}